%% file: SE_Sync.tex
\documentclass{article}

\usepackage[margin = 1.35in]{geometry}  

\usepackage{amsmath, amssymb, amsthm}
\usepackage{bbm}  
\usepackage{accents}  
\newcommand{\ubar}[1]{\underaccent{\bar}{#1}}

\usepackage[all]{xy}

\usepackage[pdftex]{graphicx}

\usepackage{epstopdf}
\usepackage{epsfig}

\usepackage{tikz}

\usepackage{subfigure}

\usepackage[numbers,sort&compress]{natbib} 

\usepackage{url}

\usepackage{authblk}
\usepackage{xspace}

\usepackage{hhline}  

\newcommand{\sync}{SE-Sync}
\newcommand{\syncSpace}{\sync\xspace}

\newcommand{\titlestring}{\sync:  A Certifiably Correct Algorithm for Synchronization over the Special Euclidean Group}
\newcommand{\authorstring}{David M.\ Rosen, Luca Carlone,  Afonso S.\ Bandeira, and John J.\ Leonard}

\usepackage[bookmarks=true, bookmarksnumbered=true, pdfpagemode=UseOutlines, pdfstartview=FitH, 
colorlinks, linkcolor=black, citecolor=black, urlcolor=black, pdftitle={\titlestring}, pdfauthor={\authorstring}]{hyperref}

\usepackage{algorithm}
\usepackage{algpseudocode}
\makeatletter
\let\OldStatex\Statex
\renewcommand{\Statex}[1][0]{%
  \setlength\@tempdima{\algorithmicindent}%
  \OldStatex\hskip\dimexpr#1\@tempdima\relax}
\makeatother

\algnewcommand\algorithmicinput{\textbf{Input:}}
\algnewcommand\Input{\item[\algorithmicinput]}
\algnewcommand\algorithmicoutput{\textbf{Output:}}
\algnewcommand\Output{\item[\algorithmicoutput]}

\newcommand{\R}{\mathbb{R}}

\newcommand{\N}{\mathbb{N}}

\DeclareMathOperator{\Sym}{Sym}

\DeclareMathOperator{\tr}{tr}
\DeclareMathOperator{\Span}{span}
\DeclareMathOperator{\rank}{rank}
\DeclareMathOperator{\Diag}{Diag}
\DeclareMathOperator \vect{vec}
\def \ones{\mathbbm{1}}
\def \transpose{^\mathsf{T}}
\def \inv{^{-1}}
\def \tinv{^{-\mathsf{T}}}
\def \pinv {^\dagger}

\DeclareMathOperator{\image}{image}

\DeclareMathOperator{\Stiefel}{St}

\DeclareMathOperator{\DirectionalDerivative}{D}
\DeclareMathOperator{\grad}{grad}
\DeclareMathOperator{\Hess}{Hess}
\DeclareMathOperator{\proj}{Proj}

\DeclareMathOperator{\Orthogonal}{O}
\DeclareMathOperator{\SO}{SO}

\DeclareMathOperator{\SE}{SE}

\DeclareMathOperator*{\argmax}{argmax}
\DeclareMathOperator*{\argmin}{argmin}

\newcommand{\st}{\textnormal{s.t.\; }}
\def \Lag {\mathcal{L}}

\DeclareMathOperator{\SD}{SD}
\DeclareMathOperator{\Uniform}{\mathcal{U}}
\DeclareMathOperator{\Gaussian}{\mathcal{N}}  
\DeclareMathOperator{\Langevin}{Langevin}  
\DeclareMathOperator{\vonMises}{vM}

\def \Nodes {\mathcal{V}}  
\def \Edges {\mathcal{E}}  
\newcommand{\directed}[1]{\vec{#1}}  
\def \dEdges{\directed{\Edges}}  

\def \edge{\lbrace i,j \rbrace}  
\def \dedge{(i,j)}  

\def \incEdges{\delta}
\def \inEdges{\delta^{+}}
\def \outEdges{\delta^{-}}

\def\Lap{L}  

\def\incMat{A}  
\def\redIncMat{\ubar{A}}  

\def \OrthoProj{\pi}
\def \OrthoProjMatrix{\Pi}

\newtheorem{thm}{Theorem}

\newtheorem{lem}[thm]{Lemma}
\newtheorem{prop}[thm]{Proposition}

\theoremstyle{definition}
\newtheorem{problem}{Problem}

\def \CSAIL{Computer Science and Artificial Intelligence Laboratory}
\def \LIDS{Laboratory for Information and Decision Systems}
\def \MathDep{Department of Mathematics}
\def \MIT{Massachusetts Institute of Technology}
\def \MITaddr{Cambridge, MA 02139, USA}

\def \CourInst{Courant Institute of Mathematical Sciences}
\def \CDS{Center for Data Science}
\def \NYU{New York University}
\def \NYUaddr{New York, NY 10012, USA}

\title{\titlestring\footnote{This report is an extended version of a paper presented at the 12\textsuperscript{th} International Workshop on the Algorithmic Foundations of Robotics \cite{Rosen2016Certifiably}. }}

\date{}

\author{David M.\ Rosen\thanks{Corresponding author.  Email: \href{mailto:dmrosen@mit.edu}{dmrosen@mit.edu}}}
\affil{\CSAIL, \MIT, \MITaddr}

\author{Luca Carlone}
\affil{\LIDS, \MIT, \MITaddr}

\author{Afonso S.\ Bandeira}
\affil{\MathDep\ and \CDS, \CourInst, \NYU, \NYUaddr}

\author[1]{John J.\ Leonard}

\DeclareMathOperator{\BDiag}{BlockDiag}
\DeclareMathOperator{\SymBlockDiag}{SymBlockDiag}
\DeclareMathOperator{\SBD}{SBD}  

\def \pose{x}
\def \tran{t}
\def \rot{R}

\newcommand{\true}[1]{\ubar{#1}}
\newcommand{\noisy}[1]{\tilde{#1}}
\newcommand{\est}[1]{\hat{#1}}

\def \tpose{\true{\pose}}
\def \ttran{\true{\tran}}
\def \trot{\true{\rot}}

\def \optsym{*}
\def \topt{\tran^\optsym}
\def \Ropt{\rot^\optsym}

\def \poseopt{\pose^\optsym}
\def \Lopt{\Lambda^\optsym}
\def \Yopt{Y^\optsym}
\def \Zopt{Z^\optsym}

\def \npose{\noisy{\pose}}
\def \ntran{\noisy{\tran}}
\def \nrot{\noisy{\rot}}

\def \tranNoise{\tran^\epsilon}
\def \rotNoise{\rot^\epsilon}

\def \RotEst{\est{\rot}}
\def \TranEst{\est{\tran}}
\def \PoseEst{\est{\pose}}

\def \MLEtext{\textnormal{MLE}}
\def \PosesMLE{\PoseEst_{\MLEtext}}

\def \MLEval{p_{\textnormal{MLE}}^*}  
\def \OMLEval{p_\textnormal{O}^*}  
\def \SDPval{p_{\textnormal{SDP}}^*}  
\def \SDPLRval{p_{\textnormal{SDPLR}}^*}

\def \rotsym{\rho}  
\def \transym{\tau}  
\def \rotdeg{d^{\rotsym}} 

\def \TranW{W^{\transym}}  
\def \LapTranW{\Lap(\TranW)}  
\def \RotW{W^{\rotsym}}  
\def \LapRotW{\Lap(\RotW)}  

\def \MeasRotGraph{\noisy{G}^\rotsym}  
\def \TrueRotGraph{\true{G}^\rotsym}   
\def \MeasRotConLap{\Lap(\MeasRotGraph)}  
\def \TrueRotConLap{\Lap(\TrueRotGraph)}  

\def \nCrossTerms{\noisy{V}}  
\def \nOuterProducts{\noisy{\Sigma}}  

\def \QPFormMatrix{M}

\def \tQ{\true{Q}}
\def \nQ{\noisy{Q}}  
\def \dQ{\Delta Q}
\def \dR{\Delta \rot}
\def \nQtran{\noisy{Q}^\transym}  

\def \tQtran{\true{Q}^\transym}  

\def \tranPrecisions{\Omega}  
\def \nT{\noisy{T}}  
\def \tT{\true{T}}

\def \manifold{\mathcal{M}}

\DeclareMathOperator{\Sorbit}{\mathcal{S}}
\DeclareMathOperator{\Oorbit}{\mathcal{O}}
\def \Oorbdist{d_{\mathcal{O}}}
\def \Sorbdist{d_{\mathcal{S}}}

\def \exactnessBound{\beta}
\def \certMat{C}

\begin{document}

\maketitle

\begin{abstract}
Many important geometric estimation problems naturally take the form of \emph{synchronization over the special Euclidean group}: estimate the values of a set of unknown group elements $\pose_1, \dotsc, \pose_n \in \SE(d)$ given noisy measurements of a subset of their pairwise relative transforms $\pose_i\inv \pose_j$.  Examples of this class include the foundational problems of pose-graph simultaneous localization and mapping (SLAM) (in robotics), camera motion estimation (in computer vision), and sensor network localization (in distributed sensing), among others.  This inference problem is typically formulated as a nonconvex maximum-likelihood estimation that is computationally hard to solve in general.  Nevertheless, in this paper we present an algorithm that is able to efficiently recover \emph{certifiably globally optimal} solutions of the special Euclidean synchronization problem in a non-adversarial noise regime.  The crux of our approach is the development of a semidefinite relaxation of the maximum-likelihood estimation whose minimizer provides an \emph{exact} MLE so long as the magnitude of the noise corrupting the available measurements falls below a certain critical threshold; furthermore, whenever exactness obtains, it is possible to \emph{verify} this fact \emph{a posteriori}, thereby \emph{certifying} the optimality of the recovered estimate.  We develop a specialized optimization scheme for solving large-scale instances of this semidefinite relaxation by exploiting its low-rank, geometric, and graph-theoretic structure to reduce it to an equivalent optimization problem defined on a low-dimensional Riemannian manifold, and  then design a Riemannian truncated-Newton trust-region method to solve this reduction efficiently.  Finally, we combine this fast optimization approach with a simple rounding procedure to produce our algorithm, \emph{\sync}.  Experimental evaluation on a variety of simulated and real-world pose-graph SLAM datasets shows that \syncSpace is capable of recovering certifiably globally optimal solutions when the available measurements are corrupted by noise up to an order of magnitude greater than that typically encountered in robotics and computer vision applications, and  does so more than an order of magnitude faster than the Gauss-Newton-based approach that forms the basis of current state-of-the-art techniques.

\end{abstract}

\newpage
\tableofcontents
 \newpage

\input{Introduction.tex}
\input{Problem_Formulation.tex}
\input{Forming_the_Semidefinite_Relaxation.tex}
\input{The_SE_Sync_Algorithm.tex}
\input{Experimental_Results.tex}

\input{Conclusion.tex}

\appendix\input{The_Isotropic_Langevin_distribution.tex}

\input{Reformulating_the_MLE.tex}

\input{A_Sufficient_Condition_for_Exactness.tex}

\bibliographystyle{plainnat}
\bibliography{references}

\end{document}

%% file: Introduction.tex
\section{Introduction}

\subsection{Motivation}

Many important geometric estimation problems naturally take the form of \emph{synchronization over the special Euclidean group}: estimate the values of a collection of unknown group elements $\pose_1, \dotsc, \pose_n \in \SE(d)$ given noisy observations of a subset of their pairwise relative transforms $\pose_i \inv \pose_j$.\footnote{More generally, \emph{synchronization over a group $G$} is the problem of estimating a set of unknown group elements $g_1, \dotsc, g_n \in G$ given noisy observations of a subset of their pairwise relative transforms $g_i \inv g_j$ \cite{Singer2011Angular}.  The nomenclature originates with the   prototypical example of this class of problems: synchronization of clocks over a communications network \cite{Karp2003Optimal,Giridhar2006Distributed} (corresponding to synchronization over the additive group $\R$). }  For example, the foundational problems of pose-graph simultaneous localization and mapping (SLAM) \cite{Lu1997Globally,Grisetti2010Tutorial} (in robotics), camera motion and/or orientation estimation  \cite{Arrigoni2015Spectral,Hartley2013Rotation,Tron2016Survey} (in computer vision), and sensor network localization \cite{Peters2015Sensor} (in distributed sensing) all belong to this class, and closely-related formulations arise in applications as diverse as artifact reconstruction \cite{Brown2008System} (in archaeology) and molecular structure determination \cite{Singer2011Three,Bandeira2015Nonunique} (in chemistry).  Consequently, the development of \emph{fast} and \emph{reliable} methods to solve the special Euclidean synchronization problem is of considerable practical import.

\subsection{Related work}

In general, group synchronization problems are formulated as instances of maximum-likelihood estimation under an assumed probability distribution for the measurement noise.  This formulation is attractive from a theoretical standpoint due to the powerful analytical framework and strong performance guarantees that maximum-likelihood estimation affords \cite{Ferguson1996Course}.  However, this formal rigor often comes at the expense of computational tractability, as it is frequently the case that the optimization problem underlying an instance of maximum-likelihood estimation is nonconvex, and therefore computationally hard to solve in general.  Unfortunately, this turns out to be the case for the special Euclidean synchronization problem in particular, due to the nonconvexity of $\SE(d)$ itself.

 Given the fundamental computational hardness of nonconvex optimization, prior work on special Euclidean synchronization (and related estimation problems) has predominantly focused on the development of \emph{approximate} algorithms that can efficiently compute high-quality estimates in practice.   These approximate algorithms can be broadly categorized into two classes.  
 
 The first class consists of algorithms that are based upon the (heuristic) application of fast \emph{local} search techniques to identify promising estimates.  This approach is particularly attractive for robotics applications, as the high computational speed of first- and second-order smooth nonlinear programming methods \cite{Nocedal2006Numerical}, together with their ability to exploit the measurement sparsity that typically occurs in naturalistic problem instances \cite{Dellaert2006Square}, enables these techniques to scale effectively to large problems while maintaining real-time operation with limited computational resources. Indeed, there now exist a variety of mature algorithms and software libraries implementing this approach that are able to process special Euclidean synchronization problems involving tens to hundreds of thousands of latent states in real time using only a single thread on a commodity processor \cite{Dellaert2010Subgraph,Olson2006Fast,Grisetti2009Nonlinear,Kuemmerle2011g20,Kaess2012iSAM2ijrr,Rosen2014RISE,Lourakis2009SBA,Konolige2010Sparse}.  However, the restriction to \emph{local} search leaves these methods vulnerable to convergence to significantly suboptimal local minima, even for relatively low levels of measurement noise \cite{Carlone2015Lagrangian}.  Furthermore, even when these techniques \emph{do} succeed in recovering a globally optimal solution, they provide no means of \emph{verifying} that this is so.
 
As an alternative to local search, the second class of algorithms employs \emph{convex relaxation}: in this approach, one modifies the original estimation problem so as to obtain a convex approximation that can be (\emph{globally}) solved efficiently.  Recent work has proposed a wide variety of convex relaxations for special Euclidean synchronization and related estimation problems, including linear \cite{Carlone2015Initialization,Martinec2007Robust}, spectral \cite{Singer2011Angular,Arrigoni2015Spectral,Cucuringu2012Sensor,Bandeira2013Cheeger} and semidefinite \cite{Singer2011Angular,Rosen2015Approximate,Wang2013Exact,Ozyesil2015Stable} formulations, among others.  The advantage of these techniques is that the convex surrogates they employ generally capture the global structure of the original problem well enough that their solutions lie near high-quality regions of the search space for the original estimation problem.  However, as these surrogates are typically obtained from the original problem by relaxing constraints, their minimizers are generally infeasible for the original estimation problem, and therefore must be (potentially suboptimally) reprojected onto the original problem's feasible set.\footnote{This reprojection operation is often referred to as \emph{rounding}.}

Motivated by these considerations, in our previous work \cite{Carlone2015Lagrangian,Carlone2016Planar} we considered the following simpler but related \emph{verification problem}:  given a candidate solution $\PoseEst = (\PoseEst_1, \dotsc, \PoseEst_n) \in \SE(d)^n$ for a special Euclidean synchronization problem (obtained, for example, using a fast approximate algorithm), provide an upper bound on $\PoseEst$'s suboptimality.  In the course of that investigation, we employed Lagrangian duality to develop a convex relaxation of the special Euclidean synchronization problem that we observed is frequently  \emph{tight}.\footnote{A relaxation is said to be \emph{tight} if it attains the same optimal value as the problem from which it is derived.}  Furthermore, we showed that whenever tightness obtains, there is a simple (linear) relation between primal-dual pairs $(\pose^*, \lambda^*)$ of optimal solutions for the synchronization problem and its relaxation that enables the recovery of one from the other.  Together, these observations enabled us to develop an algorithm for \emph{computationally certifying} the correctness (global optimality) of a candidate solution $\PoseEst$ by constructing from it the corresponding \emph{dual certificate} $\hat{\lambda}$.  Moreover, we observed that any candidate solution $\PoseEst$ whose optimality could be \emph{certified} using this verification procedure could also be \emph{directly computed} by solving the Lagrangian dual problem; this established (at least in principle), the existence of an algorithm that is capable of recovering \emph{certifiably globally optimal} solutions of special Euclidean synchronization problems by means of convex programming.

However, the Lagrangian relaxation developed in our verification procedure turned out to be a \emph{semidefinite program}  \cite{Vandenberghe1996Semidefinite}, and while it is always possible to solve such problems in polynomial time using interior-point methods, in practice the high computational cost of these techniques prevents them from scaling effectively to problems in which the dimension of the decision variable is greater than a few thousand \cite{Todd2001Semidefinite}.  Unfortunately, the semidefinite relaxations corresponding to real-world instances of $\SE(d)$ synchronization arising in (for example) robotics and computer vision applications are typically one to two orders of magnitude larger than this maximum effective problem size, and are therefore well beyond the reach of these general-purpose techniques.\footnote{This includes the most commonly-used semidefinite programming software libraries, including SDPA \cite{Yamashita2003Implementation}, SeDuMi \cite{Sturm1999SeDuMi}, SDPT3 \cite{Toh1999SDPT3}, CSDP \cite{Borchers1999CSDP}, and DSDP \cite{Benson2000Solving}.}  It was thus not clear on the basis of our prior work \cite{Carlone2015Lagrangian,Carlone2016Planar} alone whether this convex relaxation approach could be implemented as a practically-effective algorithm for $\SE(d)$ synchronization.

\subsection{Contribution}

The main contribution of the present paper is to show that the semidefinite relaxation approach proposed in our prior work  \cite{Carlone2015Lagrangian,Carlone2016Planar} can indeed be realized as a practical algorithm that is capable of efficiently recovering \emph{certifiably globally optimal} solutions of large-scale special Euclidean synchronization problems in a non-adversarial (but operationally relevant) noise regime.  More specifically, our contribution consists of the following elements:

\begin{itemize}
 \item We develop improved formulations of the special Euclidean synchronization problem (Problem \ref{Simplified_maximum_likelihood_estimation_for_SE3_synchronization}) and its semidefinite relaxation (Problem \ref{dual_semidefinite_relaxation_for_SE3_synchronization_problem}) that both simplify and generalize the versions originally presented in \cite{Carlone2015Lagrangian,Carlone2016Planar}.
 
 
 \item With the aid of these improved formulations, we prove the existence of a non-adversarial noise regime within which minimizers of the semidefinite relaxation Problem \ref{dual_semidefinite_relaxation_for_SE3_synchronization_problem} provide \emph{exact}, \emph{globally optimal} solutions of the $\SE(d)$ synchronization problem (Proposition \ref{A_sufficient_condition_for_exact_recovery_prop}).
 
 \item Finally, we develop a specialized  optimization scheme that is capable of efficiently solving large-scale instances of the semidefinite relaxation in practice (Section \ref{optimization_approach_subsection}).  Combining this fast optimization approach with a simple rounding method (Algorithm \ref{Rounding_algorithm}) produces \emph{\syncSpace} (Algorithm \ref{SE_sync_algorithm}), our proposed algorithm for synchronization over the special Euclidean group.
\end{itemize}

\syncSpace is thus a \emph{certifiably correct} algorithm \cite{Bandeira2016Probably}, meaning that it is capable of efficiently solving a generally-intractable problem within a restricted operational regime, and of \emph{computationally certifying} the correctness of the solutions that it recovers.  Intuitively, algorithms belonging to this class give up the ability to solve \emph{every} instance of a problem in order to achieve efficient computation on a subset that contains a large fraction of those instances \emph{actually encountered in practice}.  In the case of our algorithm, experimental evaluation on a  variety of simulated and real-world datasets drawn from the motivating application of pose-graph SLAM (Section \ref{Experimental_results_section}) shows that \syncSpace is capable of recovering certifiably globally optimal solutions when the available measurements are corrupted by noise up to an order of magnitude greater than that typically encountered in robotics and computer vision applications, and does so more than an order of magnitude faster than the Gauss-Newton-based approach that forms the basis of current state-of-the-art techniques \cite{Kuemmerle2011g20,Kaess2012iSAM2ijrr,Rosen2014RISE}.

%

%% file: Problem_Formulation.tex
\section{Problem formulation}
\label{Problem_formulation_and_main_results_section}

\subsection{Notation and mathematical preliminaries}
\label{Notation_subsection}

\textbf{Miscellaneous sets:}  The symbols $\N$ and $\R$ denote the nonnegative integers and the real numbers, respectively, and we write $[n] \triangleq \lbrace 1, \dotsc, n \rbrace$ for $n > 0$ as a convenient shorthand notation for sets of indexing integers.  We use $\lvert S \rvert$ to denote the cardinality of a set $S$.

\textbf{Differential geometry and Lie groups:}  We will encounter several smooth manifolds and Lie groups over the course of this paper, and will often make use of both the \emph{intrinsic} and \emph{extrinsic} formulations of the same manifold as convenience dictates; our notation will generally be consistent with that of \cite{Warner1983Manifolds} in the former case and \cite{Guillemin1974Differential} in the latter.  When considering an extrinsic realization $\manifold \subseteq \R^d$ of a manifold $\manifold$ as an embedded submanifold of a Euclidean space and a function $f \colon \R^d \to \R$, it will occasionally be important for us to distinguish the notions of $f$ considered as a function on $\R^d$ and $f$ considered as a function on the submanifold $\manifold$; in these cases, to avoid confusion we will always reserve $\nabla f$ and $\nabla^2 f$ for the gradient and Hessian of $f$ with respect to the usual metric on $\R^d$, and write $\grad f$ and $\Hess f$ to refer to the Riemannian gradient and Hessian of $f$ considered as a function on $\manifold$ (equipped with the metric inherited from its embedding) \cite{Boothby2003Riemannian,Kobayashi1996Foundations}.

We let $\Orthogonal(d)$, $\SO(d)$, and $\SE(d)$ denote the orthogonal, special orthogonal, and special Euclidean groups in dimension $d$, respectively.  For computational convenience we will often identify the (abstract) Lie groups $\Orthogonal(d)$ and $\SO(d)$ with their realizations as the matrix groups:
\begin{subequations}
 \label{orthogonal_and_special_orthogonal_matrix_groups}
 \begin{equation}
 \label{orthogonal_matrix_group_definition}
 \Orthogonal(d) \cong \lbrace R \in \R^{d \times d} \mid R\transpose R = R R\transpose = I_d \rbrace
 \end{equation}
\begin{equation}
\label{special_orthogonal_matrix_group_definition}
 \SO(d) \cong \lbrace R \in \R^{d \times d} \mid R\transpose R = R R\transpose = I_d, \; \det(R) = +1 \rbrace,
\end{equation}
\end{subequations}
and $\SE(d)$ with the semidirect product $\R^d \rtimes \SO(d)$ with group operations:
\pagebreak
\begin{subequations}
\label{SE3_group_operations}
\begin{equation}
\label{SE3_multiplication_rule}
(\tran_1, \rot_1) \cdot (\tran_2, \rot_2) = (\tran_1 +\rot_1 \tran_2, \rot_1 \rot_2).
\end{equation}
\begin{equation}
 \label{SE3_inversion_rule}
 (\tran, \rot)^{-1} = (-\rot^{-1} \tran, \rot^{-1}).
\end{equation}
\end{subequations}
The set of orthonormal $k$-frames in $\R^n$ ($k \le n$):
\begin{equation}
 \label{Stiefel_manifold_definition}
 \Stiefel(k, n) \triangleq \left \lbrace Y \in \R^{n \times k} \mid Y\transpose Y = I_k \right \rbrace
\end{equation}
is also a smooth compact matrix manifold, called the \emph{Stiefel manifold}, and we equip $\Stiefel(k,n)$ with the Riemannian metric induced by its embedding into $\R^{n \times k}$ \cite[Sec.\ 3.3.2]{Absil2009Optimization}.

\textbf{Linear algebra:}  In addition to the matrix groups defined above, we write $\Sym(d)$ for the set of real $d \times d$ symmetric matrices; $A \succeq 0$ and $A \succ 0$ indicate that $A \in \Sym(d)$ is positive semidefinite and positive definite, respectively.  For general matrices $A$ and $B$, $A \otimes B$ indicates the Kronecker (matrix tensor) product, $A\pinv$ the Moore-Penrose pseudoinverse, and $\vect(A)$ the vectorization operator that concatenates the columns of $A$ \cite[Sec.\ 4.2]{Horn1991Topics}.  We write $e_i \in \R^d$ and $E_{ij} \in \R^{m \times n}$ for the $i$th unit coordinate vector and $(i,j)$-th unit coordinate matrix, respectively, and $\ones_d \in \R^d$ for the all-$1$'s vector.  Finally, $\lVert \cdot \rVert_2$, $\lVert \cdot \rVert_F$, and $\lVert \cdot \rVert_*$ denote the spectral, Frobenius, and nuclear matrix norms, respectively.  

We will also frequently consider various $(d \times d)$-block-structured matrices, and it will be useful to have specialized operators for them.  To that end, given square matrices $A_i \in \R^{d \times d}$ for $i \in [n]$, we let $\Diag(A_1, \dotsc, A_n)$ denote their matrix direct sum. Conversely, given a $(d \times d)$-block-structured matrix $M\in \R^{dn \times dn}$ with $ij$-block $M_{ij} \in \R^{d \times d}$, we let $\BDiag_d(M)$ denote the linear operator that extracts $M$'s $(d\times d)$-block diagonal:
\begin{equation}
 \label{BDiag_operator_definition}
 \BDiag_d(M) \triangleq \Diag(M_{11}, \dotsc, M_{nn}) = 
 \begin{pmatrix}
  M_{11}  \\
  
  & \ddots \\
  & & M_{nn}
 \end{pmatrix}
\end{equation}
and $\SymBlockDiag_d$ its corresponding symmetrized form:
\begin{equation}
 \label{SymBlockDiag_operator_definition}
 \SymBlockDiag_d(M) \triangleq \frac{1}{2} \BDiag_d\left(M + M\transpose\right).
\end{equation}
Finally, we let $\SBD(d, n)$ denote the set of symmetric $(d \times d)$-block-diagonal matrices in $\R^{dn \times dn}$:
\begin{equation}
 \SBD(d, n) \triangleq \lbrace \Diag(S_1, \dotsc, S_n) \mid S_1, \dotsc, S_n \in \Sym(d) \rbrace.
\end{equation}

\textbf{Graph theory:} An \emph{undirected} \emph{graph} is a pair $G = (\Nodes, \Edges)$, where $\Nodes$ is a finite set and $\Edges$ is a set of unordered pairs $\edge$ with $i, j \in \Nodes$ and $i \ne j$.  Elements of $\Nodes$ are called  \emph{vertices} or \emph{nodes}, and elements of $\Edges$ are called \emph{edges}.  An edge $e = \edge \in \Edges$ is said to be \emph{incident} to the vertices $i$ and $j$;  $i$ and $j$ are called the \emph{endpoints} of $e$.  We write $\incEdges(v)$ for the set of edges incident to a vertex $v$.

A \emph{directed} \emph{graph} is a pair $\directed{G} = (\Nodes, \dEdges)$, where $\Nodes$ is a finite set and $\dEdges \subset \Nodes \times \Nodes$ is a set of \emph{ordered} pairs $\dedge$ with $i \ne j$.\footnote{Note that our definitions of directed and undirected graphs exclude loops and parallel edges.  While all of our results can be straightforwardly generalized to admit parallel edges (and indeed our experimental implementation of \syncSpace supports them), we have adopted this restriction in order to simply the following  presentation.} As before, elements of $\Nodes$ are called \emph{vertices} or \emph{nodes}, and elements of $\dEdges$ are called (\emph{directed}) \emph{edges} or \emph{arcs}.  Vertices $i$ and $j$ are called the \emph{tail} and \emph{head} of the directed edge $e = \dedge$, respectively; $e$ is said to \emph{leave} $i$ and \emph{enter} $j$ (we also say that $e$ is \emph{incident} to $i$ and $j$ and that $i$ and $j$ are $e$'s \emph{endpoints}, as in the case of undirected graphs).  We let $t, h \colon \dEdges \to \Nodes$ denote the functions mapping each edge to its tail and head, respectively, so that $t(e) = i$ and $h(e) = j$ for all $e = \dedge \in \dEdges$.   Finally, we again let $\incEdges(v)$ denote the set of directed edges incident to $v$, and $\outEdges(v)$ and $\inEdges(v)$ denote the sets of edges leaving and entering vertex $v$, respectively.

Given an undirected graph $G = (\Nodes, \Edges)$, we can construct a directed graph $\directed{G} = (\Nodes, \dEdges)$ from it by arbitrarily ordering the elements of each pair $\edge \in \Edges$; the graph $\directed{G}$ so obtained is called an \emph{orientation} of $G$.


A \emph{weighted graph} is a triplet $G = (\Nodes, \Edges, w)$ comprised of a graph $(\Nodes, \Edges)$ and a weight function $w \colon \Edges \to \R$ defined on the edges of this graph; since $\Edges$ is finite, we can alternatively specify the weight function $w$ by simply listing its values $\lbrace w_{e} \rbrace_{e \in \Edges}$ on each edge.  Any unweighted graph can be interpreted as a weighted graph equipped with the \emph{uniform weight function} $w \equiv 1$.

We can associate to a directed graph $\directed{G} = (\Nodes, \dEdges)$ with $n = \lvert \Nodes \rvert$ and $m = \lvert \dEdges \vert$ the \emph{incidence matrix} $\incMat(\directed{G}) \in \R^{n \times m}$ whose rows and columns are indexed by $i\in \Nodes$ and $e \in \dEdges$, respectively, and whose elements are determined by:
\begin{equation}
 \label{incidence_matrix_definition}
 \incMat(\directed{G})_{ie} = 
 \begin{cases}
+1, & e\in \inEdges(i) \\
-1, & e \in \outEdges(i), \\
0, &\textnormal{otherwise}.
 \end{cases}
\end{equation} Similarly, we can associate to an undirected graph $G$ an \emph{oriented incidence matrix} $\incMat(G)$ obtained as the incidence matrix of any of its orientations $\directed{G}$.  We obtain a \emph{reduced} (\emph{oriented}) \emph{incidence matrix} $\redIncMat(G)$ by removing the final row from the (oriented) incidence matrix $\incMat(G)$. 

Finally, we can associate to a weighted undirected graph $G = (\Nodes, \Edges, w)$ with $n = \lvert \Nodes \rvert$ the \emph{Laplacian matrix} $\Lap(G) \in \Sym(n)$ whose rows and columns are indexed by $i \in \Nodes$, and whose elements are determined by:
\begin{equation}
\label{Laplacian_matrix_definition}
\Lap(G)_{ij} = 
\begin{cases}
\sum_{e \in \delta(i)} w(e), & i = j, \\
-w(e), & i \ne j \textnormal{ and } e = \edge \in \Edges, \\
0, & \textnormal{otherwise}.
\end{cases}
\end{equation}
A straightforward computation shows that the Laplacian of a weighted graph $G$ and the incidence matrix of one of its orientations $\directed{G}$ are related by:
\begin{equation}
\label{Laplacian_in_terms_of_incidence_matrix}
 \Lap(G) = \incMat(\directed{G}) W \incMat(\directed{G})\transpose,
\end{equation}
where $W \triangleq \Diag(w(e_1), \dotsc, w(e_m))$ is the diagonal matrix containing the weights of $G$'s edges.

\textbf{Probability and statistics:}  We write $\Gaussian(\mu, \Sigma)$ for the multivariate Gaussian distribution with mean $\mu \in \R^d$ and covariance matrix $0 \preceq \Sigma \in \Sym(d)$, and $\Langevin(M, \kappa)$ for the isotropic Langevin distribution on $\SO(d)$ with mode $M \in \SO(d)$ and concentration parameter $\kappa \ge 0$ (cf.\ Appendix \ref{Isotropic_Langevin_distribution_appendix}).  With reference to a hidden parameter $X$ whose value we wish to infer, we will write $\true{X}$ for its true (latent) value, $\noisy{X}$ to denote a noisy observation of $\true{X}$, and $\est{X}$ to denote an estimate of $\true{X}$.

\subsection{The special Euclidean synchronization problem}
\label{Problem_formulation_subsection}
The $\SE(d)$ synchronization problem consists of estimating the values of a set of $n$ unknown group elements $\pose_1, \dotsc, \pose_n  \in\SE(d)$ given noisy measurements of $m$ of their pairwise relative transforms $\pose_{ij} \triangleq  \pose_i^{-1} \pose_j$ ($i \ne j$).  We model the set of available measurements using an undirected graph $G = (\Nodes, \Edges)$ in which the nodes $i \in \Nodes$ are in one-to-one correspondence with the unknown states $x_i$ and the edges $\edge \in \Edges$ are in one-to-one correspondence with the set of available measurements, and we assume without loss of generality that $G$ is connected.\footnote{If $G$ is not connected, then the problem of estimating the unknown states $x_1, \dotsc, x_n$ decomposes into a set of independent estimation problems that are in one-to-one correspondence with the connected components of $G$; thus, the general case is always reducible to the case of connected graphs.}   We let $\directed{G} =(\Nodes, \dEdges)$ be the directed graph obtained from $G$ by fixing an orientation, and assume that a noisy measurement $\npose_{ij}$ of the relative transform $\pose_{ij}$  is obtained 
by sampling from the following probabilistic generative model:
\begin{equation}
\label{probabilistic_generative_model_for_noisy_observations}
\begin{aligned}
 \ntran_{ij} &= \ttran_{ij} + \tranNoise_{ij},  &  \tranNoise_{ij} &\sim \Gaussian\left(0, \tau_{ij}^{-1} I_d\right), \\
 \nrot_{ij} &= \trot_{ij} \rotNoise_{ij}, & \quad \rotNoise_{ij} &\sim \Langevin\left(I_d, \kappa_{ij}\right),
\end{aligned} \quad \quad \forall \dedge \in \dEdges
\end{equation}
where $\tpose_{ij} = (\ttran_{ij}, \trot_{ij})$ is the true (latent) value of $\pose_{ij}$.\footnote{We use a directed graph to model the measurements $\npose_{ij}$ sampled from \eqref{probabilistic_generative_model_for_noisy_observations} because the distribution of the noise corrupting the latent values $\tpose_{ij}$ is not invariant under $\SE(d)$'s group inverse operation, as can be seen by composing \eqref{probabilistic_generative_model_for_noisy_observations} with \eqref{SE3_inversion_rule}.  Consequently, we must keep track of which state $x_i$ was the ``base frame'' for each measurement. }  Finally, we define $\npose_{ji} \triangleq \npose_{ij}^{-1}$, $\kappa_{ji} \triangleq \kappa_{ij}$, $\tau_{ji} \triangleq \tau_{ij}$, and $\nrot_{ji} \triangleq \nrot_{ij}\transpose$ for all $\dedge \in \dEdges$.

Given a set of noisy measurements $\npose_{ij}$ sampled from the generative model \eqref{probabilistic_generative_model_for_noisy_observations}, a straightforward computation shows that a maximum-likelihood estimate $\PosesMLE \in \SE(d)^n$ for the states $\pose_1, \dotsc, \pose_n$ is obtained as a minimizer of:\footnote{\label{MLEs_are_nonunique_footnote}Note that a minimizer of problem \eqref{SE3_Synchronization_MLE_eq} is \emph{a} maximum-likelihood estimate (rather than \emph{the} maximum-likelihood estimate) because problem \eqref{SE3_Synchronization_MLE_eq} always has multiple (in fact, infinitely many) solutions: since the objective function in \eqref{SE3_Synchronization_MLE_eq} is constructed from \emph{relative} measurements of the form $\pose_i^{-1} \pose_j$, if $\pose^* = (\pose_1^*, \dotsc, \pose_n^*) \in \SE(d)^n$ minimizes \eqref{SE3_Synchronization_MLE_eq}, then $g \bullet \pose^* \triangleq (g \cdot \pose_1^*, \dotsc, g \cdot \pose_n^*)$ also minimizes \eqref{SE3_Synchronization_MLE_eq} for all $g \in \SE(d)$. Consequently, the solution set of \eqref{SE3_Synchronization_MLE_eq} is organized into orbits of the diagonal action $\bullet$ of $\SE(d)$ on $\SE(d)^n$.  This gauge symmetry simply corresponds to the fact that \emph{relative} measurements $\pose_i^{-1} \pose_j$ provide no information about the \emph{absolute} values of the states $\pose_i$.}
\begin{equation}
 \label{SE3_Synchronization_MLE_eq}
 \min_{\stackrel{\tran_i \in \R^d}{\rot_i \in \SO(d)}} \sum_{\dedge \in \dEdges} -\kappa_{ij} \tr\left(\nrot_{ij} \rot_j^{-1} \rot_i  \right) + \frac{\tau_{ij}}{2} \left \lVert \ntran_{ij} - \rot_i^{-1}(\tran_j - \tran_i)  \right \Vert_2^2.
\end{equation}
Using the fact that $X^{-1} = X\transpose$ and $\lVert X - Y \rVert_F^2 = 2d - 2 \tr(X\transpose Y)$ for all $X, Y \in \Orthogonal(d)$ together with the orthogonal invariance of the Frobenius and $\ell_2$ norms, it is straightforward to verify that $\pose^* \in \SE(d)^n$ is a minimizer of \eqref{SE3_Synchronization_MLE_eq} if and only if it is also a minimizer of the following nonlinear least-squares problem:

\begin{problem}[Maximum-likelihood estimation for $\SE(d)$ synchronization]
 \label{SE3_synchronization_MLE_NLS_problem}
 \begin{equation}
 \label{SE3_synchronization_MLE_NLS_optimization}
\MLEval = \min_{\stackrel{\tran_i \in \R^d}{\rot_i \in \SO(d)}} \sum_{\dedge \in \dEdges} \kappa_{ij} \lVert \rot_j - \rot_i \nrot_{ij} \rVert_F^2 + \tau_{ij} \left \lVert \tran_j - \tran_i - \rot_i \ntran_{ij} \right \rVert_2^2
 \end{equation}
\end{problem}

Unfortunately, Problem \ref{SE3_synchronization_MLE_NLS_problem} is a high-dimensional nonconvex nonlinear program, and is therefore computationally hard to solve in general.  Consequently, in this paper our strategy will be to replace this problem with a (convex) \emph{semidefinite relaxation} \cite{Vandenberghe1996Semidefinite}, and then exploit this relaxation to search for solutions of Problem \ref{SE3_synchronization_MLE_NLS_problem}.


%% file: Forming_the_Semidefinite_Relaxation.tex
\section{Forming the semidefinite relaxation}
\label{Forming_the_semidefinite_relaxation_section}
In this section we develop the semidefinite relaxation that we will solve in place of the maximum-likelihood estimation Problem \ref{SE3_synchronization_MLE_NLS_problem}.  Our approach proceeds in two stages.  We begin in Section \ref{simplifying_the_maximum_likelihood_estimation_section} by developing a sequence of simplified but equivalent reformulations of Problem \ref{SE3_synchronization_MLE_NLS_problem} with the twofold goal of simplifying its analysis and elucidating some of the structural correspondences between the optimization  \eqref{SE3_synchronization_MLE_NLS_optimization} and several simple graph-theoretic objects that can be constructed from the set of available measurements $\npose_{ij}$ and the graphs $G$ and $\directed{G}$.  We then exploit the simplified versions of Problem \ref{SE3_synchronization_MLE_NLS_problem} so obtained to derive the semidefinite relaxation in Section \ref{relaxing_the_maximum_likelihood_estimation}.
%
%

\subsection{Simplifying the maximum-likelihood estimation}
\label{simplifying_the_maximum_likelihood_estimation_section}
Our first step is to rewrite Problem \ref{SE3_synchronization_MLE_NLS_problem} in a more standard form for quadratic programs.  First, define the \emph{translational} and \emph{rotational weight graphs} $\TranW \triangleq (\Nodes, \Edges, \lbrace \tau_{ij}\rbrace)$ and $\RotW \triangleq (\Nodes, \Edges, \lbrace \kappa_{ij} \rbrace)$  to be the weighted undirected graphs with node set $\Nodes$, edge set $\Edges$, and edge weights $\tau_{ij}$ and $\kappa_{ij}$ for $\edge \in \Edges$, respectively.  The Laplacians of $\TranW$ and $\RotW$ are then:
\begin{subequations}
\label{Weight_graph_Laplacians}
\begin{equation}
\label{Laplacian_of_translational_weight_graph}
\LapTranW_{ij} = 
\begin{cases}
\sum_{e \in \incEdges(i)} \tau_{e}, & i = j, \\
-\tau_{ij}, & \edge \in \Edges, \\
0, & \edge \notin \Edges,
\end{cases}
\end{equation}
\begin{equation}
 \label{Laplacian_of_rotational_weight_graph}
 \LapRotW_{ij} = 
\begin{cases}
\sum_{e \in \incEdges(i)} \kappa_{e}, & i = j, \\
-\kappa_{ij}, & \edge \in \Edges, \\
0, & \edge \notin \Edges.
\end{cases}
\end{equation}
\end{subequations}
Similarly, let $\MeasRotConLap$ denote the \emph{connection Laplacian} for the rotational synchronization problem determined by the measurements $\nrot_{ij}$ and measurement weights $\kappa_{ij}$ for $\dedge \in \dEdges$; this is the symmetric $(d \times d)$-block-structured matrix determined by  (cf.\ \cite{Singer2012Vector,Wang2013Exact}):
 \begin{subequations}
 \label{connection_Laplacian_definition}
\begin{equation}
 \begin{split}
\MeasRotConLap &\in \Sym(dn) \\
\MeasRotConLap_{ij} &\triangleq 
\begin{cases}
\rotdeg_i I_d, & i = j, \\
- \kappa_{ij} \nrot_{ij}, & \edge \in \Edges, \\
0_{d \times d}, & \edge \notin \Edges,
\end{cases}
 \end{split}
\end{equation}
\begin{equation}
 \rotdeg_i \triangleq \sum_{e \in \incEdges(i)} \kappa_{e}.
\end{equation}
 \end{subequations}
Finally, let $\nCrossTerms \in \R^{n \times dn}$ be the $(1 \times d)$-block-structured matrix with $(i,j)$-blocks:
\begin{equation}
\label{cross_term_matrix_definition}
\nCrossTerms_{ij} \triangleq 
\begin{cases}
\sum_{e \in \outEdges(j)} \tau_{e} \ntran_{e}\transpose, & i = j, \\
-\tau_{ji} \ntran_{ji}\transpose, & (j,i) \in \dEdges, \\
0_{1 \times d}, & \textnormal{otherwise},
\end{cases}
\end{equation}
and $\nOuterProducts$ the $(d \times d)$-block-structured block-diagonal matrix determined by:
\begin{equation}
\label{Sigma_matrix_definition}
\begin{split}
\nOuterProducts &\triangleq \Diag(\nOuterProducts_1, \dotsc, \nOuterProducts_n) \in \SBD(d, n) \\
\nOuterProducts_i &\triangleq \sum_{e \in \outEdges(i)} \tau_{e} \ntran_{e} \ntran_{e}\transpose.
\end{split}
\end{equation}
Aggregating the rotational and translational states into the block matrices:
\begin{subequations}
\label{block_matrix_state_definitions}
\begin{equation}
 \label{rotational_block_matrix_state}
 R \triangleq 
\begin{pmatrix}
 R_1 & \dotsb & R_n
\end{pmatrix}
 \in \SO(d)^n \subset \R^{d \times dn}
\end{equation}
\begin{equation}
 \label{translational_block_matrix_state_definition}
 \tran \triangleq 
\begin{pmatrix}
 t_1 \\
 \vdots \\
 t_n
\end{pmatrix}
 \in \R^{dn}
\end{equation}
\end{subequations}
and exploiting definitions \eqref{Laplacian_of_translational_weight_graph}--\eqref{Sigma_matrix_definition}, Problem \ref{SE3_synchronization_MLE_NLS_problem} can be rewritten more compactly in the following standard form:
\begin{problem}[Maximum-likelihood estimation, QP form]
\label{SE3_synchronization_MLE_QP_form_problem}
 \begin{subequations}
  \label{SE3_synchronization_MLE_QP_form_optimization}
 \begin{equation}
\MLEval = \min_{\stackrel{\tran \in \R^{dn}}{\rot \in \SO(d)^n}}
\begin{pmatrix}
\tran \\
\vect(\rot)
\end{pmatrix}\transpose
(\QPFormMatrix \otimes I_d)
\begin{pmatrix}
\tran \\
\vect(\rot)
\end{pmatrix},
 \end{equation}
 \begin{equation}
 \label{M_matrix_definition}
\QPFormMatrix\triangleq
\begin{pmatrix}
\LapTranW & \nCrossTerms \\
\nCrossTerms\transpose & \MeasRotConLap + \nOuterProducts
\end{pmatrix}.
 \end{equation}
 \end{subequations}
\end{problem}
Problem \ref{SE3_synchronization_MLE_QP_form_problem} is obtained from Problem \ref{SE3_synchronization_MLE_NLS_problem} through a straightforward (although somewhat tedious) manipulation of the objective function (Appendix \ref{deriving_the_quadratic_form_of_the_MLE}). 

Expanding the quadratic form in \eqref{SE3_synchronization_MLE_QP_form_optimization}, we obtain:
\begin{equation}
\label{block_partitioned_quadratic_objective_in_MLE_problem}
\MLEval = \min_{\stackrel{t \in \R^{dn}}{R \in \SO(d)^n}}
\left \lbrace 
\begin{aligned}
&\tran\transpose \left(\LapTranW \otimes I_d \right) \tran + 2 \tran\transpose \left(\nCrossTerms \otimes I_d \right) \vect(R) \\
&\quad+ \vect(R)\transpose \left(\left( \MeasRotConLap + \nOuterProducts \right) \otimes I_d\right) \vect(R)
\end{aligned}
\right \rbrace.
\end{equation}
Now observe that for a fixed value of $\rot$, \eqref{block_partitioned_quadratic_objective_in_MLE_problem} reduces to the \emph{unconstrained} minimization of a quadratic form in the translational variable $\tran$, for which we can find a closed-form solution.  This enables us to analytically eliminate $\tran$ from the optimization problem \eqref{block_partitioned_quadratic_objective_in_MLE_problem}, thereby obtaining:

\begin{problem}[Rotation-only maximum-likelihood estimation]
\label{Rotational_maximum_likelihood_estimation_for_SE3_synchronization}
\begin{subequations}
\label{Rotational_maximum_likelihood_estimation_for_SE3_synchronization_optimization_problem}
\begin{equation}
\MLEval = \min_{\rot \in \SO(d)^n} \tr(\nQ \rot\transpose\rot) 
\end{equation}
\begin{equation}
\label{initial_Q_quadratic_form_definition}
\nQ \triangleq \MeasRotConLap + \underbrace{\nOuterProducts - \nCrossTerms \transpose \LapTranW\pinv \nCrossTerms}_{\nQtran}
\end{equation}
\end{subequations}
\end{problem}
\noindent Furthermore, given any minimizer $\Ropt$ of \eqref{Rotational_maximum_likelihood_estimation_for_SE3_synchronization_optimization_problem}, we can recover a corresponding optimal value $\topt$ for $\tran$ according to:
\begin{equation}
 \label{minimizing_value_of_t_from_minimizing_value_of_R}
 \topt = - \vect\left( \Ropt \nCrossTerms\transpose \LapTranW\pinv \right).
\end{equation}
The derivation of \eqref{Rotational_maximum_likelihood_estimation_for_SE3_synchronization_optimization_problem} and \eqref{minimizing_value_of_t_from_minimizing_value_of_R} from \eqref{block_partitioned_quadratic_objective_in_MLE_problem} is given in Appendix \ref{eliminating_translational_states_subsection}. 

Finally, we derive a simplified expression for the translational data matrix $\nQtran$ appearing in \eqref{initial_Q_quadratic_form_definition}.  Let
\begin{equation}
\label{translational_precision_matrix}
 \tranPrecisions \triangleq \Diag(\tau_{e_1}, \dotsc, \tau_{e_m}) \in \Sym(m)
\end{equation}
denote the diagonal matrix whose rows and columns are indexed by the directed edges $e \in \dEdges$  and whose $e$th diagonal element gives the precision of the translational observation corresponding to that edge.  Similarly, let $\nT \in \R^{m \times dn}$ denote the $(1 \times d)$-block-structured matrix with rows and columns indexed by $e \in \dEdges$ and $k \in \Nodes$, respectively, and whose $(e,k)$-block is given by:
\begin{equation}
\label{nT_matrix_definition}
 \nT_{ek} \triangleq
 \begin{cases}
-\ntran_{kj}\transpose, & e = (k,j) \in \dEdges,\\
0_{1 \times d}, & \textnormal{otherwise}.
 \end{cases}
\end{equation}
Then Problem \ref{Rotational_maximum_likelihood_estimation_for_SE3_synchronization} can be rewritten as:

\begin{problem}[Simplified maximum-likelihood estimation]
\label{Simplified_maximum_likelihood_estimation_for_SE3_synchronization}
\begin{subequations}
\label{Simplified_maximum_likelihood_estimation_for_SE3_synchronization_optimization_problem}
\begin{equation}
\MLEval = \min_{\rot \in \SO(d)^n} \tr(\nQ \rot\transpose\rot) 
\end{equation}
\begin{equation}
\label{Q_quadratic_form_definition}
\nQ = \MeasRotConLap + \nQtran
\end{equation}
\begin{equation}
 \label{nQtran_alternative_form}
 \nQtran = \nT\transpose \tranPrecisions^{\frac{1}{2}} \OrthoProjMatrix \tranPrecisions^{\frac{1}{2}}  \nT
 \end{equation}
\end{subequations}
\end{problem}
\noindent where $\OrthoProjMatrix \in \R^{m \times m}$ is the matrix of the orthogonal projection operator $\OrthoProj \colon \R^m \to \ker ( \incMat(\directed{G}) \tranPrecisions^{\frac{1}{2}} )$ onto the kernel of the weighted incidence matrix $\incMat(\directed{G}) \tranPrecisions^{\frac{1}{2}}$  of $\directed{G}$.  The derivation of \eqref{nQtran_alternative_form} from \eqref{initial_Q_quadratic_form_definition} is presented in Appendix \ref{An_alternative_form_for_the_translational_data_matrix_subsection}.  

The advantage of expression \eqref{nQtran_alternative_form} for $\nQtran$ versus the original formulation given in  \eqref{initial_Q_quadratic_form_definition} is that the constituent matrices $\OrthoProjMatrix$, $\tranPrecisions$, and $\nT$ in \eqref{nQtran_alternative_form} each admit particularly simple interpretations in terms of the underlying directed graph $\directed{G}$ and the translational data $(\tau_{ij}, \ntran_{ij})$ attached to each edge $\dedge \in \dEdges$; our subsequent development will heavily exploit this structure.

\subsection{Relaxing the maximum-likelihood estimation}
\label{relaxing_the_maximum_likelihood_estimation}
In this subsection, we turn our attention to the development of a convex relaxation that will enable us to recover a global minimizer of Problem \ref{Simplified_maximum_likelihood_estimation_for_SE3_synchronization}  in practice.  We begin by relaxing the condition that $\rot \in \SO(d)^n$, obtaining the following:
\begin{problem}[Orthogonal relaxation of the maximum-likelihood estimation]
\label{Orthogonal_relaxation_of_the_MLE_problem}
\begin{equation}
 \label{orthogonally_relaxed_primal_problem}
 \OMLEval = \min_{\rot \in \Orthogonal(d)^n} \tr(\nQ \rot\transpose\rot).
\end{equation}
\end{problem}

We immediately have that $\OMLEval \le \MLEval$ since $\SO(d)^n \subset \Orthogonal(d)^n$.  However, we expect that this relaxation will often be \emph{exact} in practice: since $\Orthogonal(d)$ is a disjoint union of two components separated by a distance of $\sqrt{2}$ under the Frobenius norm, and the values $\trot_i$ that we wish to estimate all lie in $\SO(d)$, the elements $\Ropt_i$ of an estimate $\Ropt$ obtained as a minimizer of \eqref{orthogonally_relaxed_primal_problem} will still all lie in the $+1$ component of $\Orthogonal(d)$ so long as the elementwise estimation error in $\Ropt$ satisfies $\lVert \Ropt_i - \trot_i \rVert_{F} < \sqrt{2}$ for all $i \in [n]$. This latter condition will hold so long as the  noise perturbing the data matrix $\nQ$ is not too large (cf.\ Appendix \ref{Upper_bound_for_estimation_error_subsection}).\footnote{There is also some empirical evidence that the relaxation from $\SO(d)$ to $\Orthogonal(d)$ is not the limiting factor in the exactness of our approach.  In our prior work \cite{Tron2015Determinant}, we observed that in the specific case $d = 3$ it is possible to replace the (cubic) determinantal constraint in \eqref{special_orthogonal_matrix_group_definition} with an equivalent \emph{quadratic} constraint by using the cross-product operation on the columns of each $R_i$ to enforce the correct orientation; this leads to an equivalent formulation of Problem \ref{SE3_synchronization_MLE_NLS_problem} as a quadratically-constrained quadratic program that can be relaxed directly to a semidefinite program \cite{Luo2010Semidefinite} \emph{without} the intermediate relaxation through $\Orthogonal(d)$.  We found no significant difference between the sharpness of the relaxation incorporating the determinantal constraint and the relaxation without (Problem \ref{dual_semidefinite_relaxation_for_SE3_synchronization_problem}). }

Now we derive the Lagrangian dual of Problem \ref{Orthogonal_relaxation_of_the_MLE_problem}, using its \emph{extrinsic} formulation:
\begin{equation}
\label{orthogonally_constrained_relaxed_program}
\begin{split}
\OMLEval = \min_{\rot \in \R^{d \times dn}}  &\tr(\nQ \rot\transpose \rot)   \\
\st \; &\rot_i\transpose \rot_i = I_d \quad \forall i = 1, \dotsc, n.
\end{split}
\end{equation}
The Lagrangian corresponding to \eqref{orthogonally_constrained_relaxed_program} is:
\begin{equation}
 \label{initial_Lagrangian_definition}
 \begin{split}
 \Lag &\colon \R^{d \times dn} \times \Sym(d)^n \to \R \\
 \Lag(\rot, \Lambda_1, \dotsc, \Lambda_n) &= \tr(\nQ \rot\transpose \rot) - \sum_{i = 1}^n \tr\left( \Lambda_i (\rot_i\transpose \rot_i - I_d) \right) \\ 
 &= \tr(\nQ \rot\transpose\rot) + \sum_{i = 1}^n\tr(\Lambda_i) - \tr\left( \Lambda_i \rot_i\transpose \rot_i \right)
 \end{split}
\end{equation}
where the $\Lambda_i \in \Sym(d)$ are symmetric matrices of Lagrange multipliers for the (symmetric) matrix orthonormality constraints in \eqref{orthogonally_constrained_relaxed_program}.  We can simplify \eqref{initial_Lagrangian_definition} by aggregating the Lagrange multipliers $\Lambda_i$ into a single direct sum matrix $\Lambda \triangleq \Diag (\Lambda_1, \dotsc, \Lambda_n) \in \SBD(d,n)$ to yield:
\begin{equation}
 \label{Lagrangian_formulation}
 \begin{split}
 &\Lag \colon \R^{d \times dn} \times \SBD(d,n) \to \R \\
 \Lag&(\rot, \Lambda) = \tr\left((\nQ - \Lambda) \rot\transpose \rot \right) + \tr(\Lambda).
 \end{split}
\end{equation}
The Lagrangian dual problem for \eqref{orthogonally_constrained_relaxed_program} is thus:
\begin{equation}
\label{initial_formulation_of_Lagrangian_dual_problem}
\max_{\Lambda \in \SBD(d,n)} \left \lbrace \inf_{R \in \R^{d \times dn}} \tr\left( (\nQ - \Lambda)\rot\transpose \rot \right) + \tr(\Lambda) \right \rbrace,
\end{equation}
with corresponding dual function:
\begin{equation}
\label{Lagrangian_dual_function}
 d(\Lambda) \triangleq\inf_{\rot \in \R^{d \times dn}} \tr\left( (\nQ - \Lambda)\rot\transpose \rot\right) + \tr(\Lambda).
\end{equation}
However, we observe that since 
\begin{equation}
\label{proof_of_necessity_of_positive_semidefiniteness_of_Q_minus_Lambda}
 \tr\left( (\nQ - \Lambda) \rot\transpose \rot \right) = \vect(\rot)\transpose \left( (\nQ - \Lambda) \otimes I_d \right) \vect(\rot),
\end{equation}
then $d(\Lambda) = - \infty$ in \eqref{Lagrangian_dual_function} unless $(\nQ -\Lambda) \otimes I_d \succeq 0$, in which case the infimum is attained for $\rot = 0$.  Furthermore, we have $(\nQ - \Lambda) \otimes I_d \succeq 0$ if and only if $\nQ - \Lambda \succeq 0$.  Therefore, the dual problem \eqref{initial_formulation_of_Lagrangian_dual_problem} is equivalent to the following semidefinite program:

\begin{problem}[Primal semidefinite relaxation for $\SE(d)$ synchronization]
\label{primal_semidefinite_relaxation_for_SE3_synchronization_problem}
\begin{equation}
\label{primal_semidefinite_relaxation_for_SE3_synchronization_optimization}
\begin{split}
 \SDPval = \max_{\Lambda \in \SBD(d,n)} &\tr(\Lambda) \\
\st \nQ - &\Lambda \succeq 0.
\end{split}
\end{equation}
\end{problem}

Finally, a straightforward application of the duality theory for semidefinite programs (see Appendix \ref{Deriving_the_dual_semidefinite_relaxation_from_the_primal_subsection} for details) shows that the dual of Problem \ref{primal_semidefinite_relaxation_for_SE3_synchronization_problem} is:

\begin{problem}[Dual semidefinite relaxation for $\SE(d)$ synchronization]
\label{dual_semidefinite_relaxation_for_SE3_synchronization_problem}
\begin{equation}
\label{dual_semidefinite_relaxation_for_SE3_synchronization_optimization}
\begin{split}
&\SDPval = \min_{Z \in \Sym(dn)} \tr(\nQ Z) \\
\st &Z =
  \begin{pmatrix}
I_d & * & * & \dotsb & *\\
* & I_d & * & \dotsb & * \\
* & * & I_d & & * \\
\vdots & \vdots &  &  \ddots & \vdots \\
*& *  & *  & \dotsb & I_d
  \end{pmatrix} \succeq 0.
\end{split}
\end{equation}
\end{problem}

At this point, it is instructive to compare the dual semidefinite relaxation \eqref{dual_semidefinite_relaxation_for_SE3_synchronization_optimization} with the simplified maximum-likelihood estimation \eqref{Simplified_maximum_likelihood_estimation_for_SE3_synchronization_optimization_problem}.  For any $\rot \in \SO(d)^n$, the product $Z =\rot \transpose \rot$ is positive semidefinite and has identity matrices along its $(d \times d)$-block-diagonal, and so is a feasible point of \eqref{dual_semidefinite_relaxation_for_SE3_synchronization_optimization}; in other words, \eqref{dual_semidefinite_relaxation_for_SE3_synchronization_optimization} can be regarded as a relaxation of the maximum-likelihood estimation obtained by \emph{expanding \eqref{Simplified_maximum_likelihood_estimation_for_SE3_synchronization_optimization_problem}'s feasible set}.  Consequently, if it so happens that a minimizer $\Zopt$ of Problem \ref{dual_semidefinite_relaxation_for_SE3_synchronization_problem} admits a decomposition of the form $\Zopt = {\Ropt}\transpose \Ropt$ for some $\Ropt \in \SO(d)^n$, then it is straightforward to verify that this $\Ropt$ is also a minimizer of Problem \ref{Simplified_maximum_likelihood_estimation_for_SE3_synchronization}.  More precisely, we have the following:

\begin{thm}
\label{certifying_exactness_theorem}
Let $\Zopt$ be a minimizer of the semidefinite relaxation Problem \ref{dual_semidefinite_relaxation_for_SE3_synchronization_problem}.  If $\Zopt$ factors as:
\begin{equation}
\label{factorization_for_certifying_exactness_theorem}
\Zopt = {\Ropt} \transpose \Ropt, \quad \Ropt \in \Orthogonal(d)^n,
\end{equation}
then $\Ropt$ is a minimizer of Problem \ref{Orthogonal_relaxation_of_the_MLE_problem}.  If additionally $\Ropt \in \SO(d)^n$, then $\Ropt$ is also a  minimizer of Problem \ref{Simplified_maximum_likelihood_estimation_for_SE3_synchronization}, and $\poseopt = (\topt, \Ropt)$  \emph{(}with $\topt$ given by equation \emph{\eqref{minimizing_value_of_t_from_minimizing_value_of_R}}\emph{)} is an optimal solution of the maximum-likelihood estimation Problem \ref{SE3_synchronization_MLE_NLS_problem}.
\end{thm}

\begin{proof}
 Weak Lagrangian duality implies that the optimal values of Problems \ref{Orthogonal_relaxation_of_the_MLE_problem} and \ref{primal_semidefinite_relaxation_for_SE3_synchronization_problem} satisfy $\SDPval \le \OMLEval$.  But if $\Zopt$ admits the factorization \eqref{factorization_for_certifying_exactness_theorem}, then $\Ropt$ is also a feasible point of \eqref{orthogonally_relaxed_primal_problem}, and so we must have that $\OMLEval \le \tr(\nQ {\Ropt}\transpose \Ropt) = \SDPval$.  This shows that $\OMLEval = \SDPval$, and consequently that $\Ropt$ is a minimizer of Problem \ref{Orthogonal_relaxation_of_the_MLE_problem}, since it attains the optimal value.
 
 Similarly, we have already established that the optimal values of Problems \ref{Simplified_maximum_likelihood_estimation_for_SE3_synchronization} and \ref{Orthogonal_relaxation_of_the_MLE_problem} satisfy $\OMLEval \le \MLEval$.  But if additionally $\Ropt \in \SO(d)^n$, then $\Ropt$ is feasible for Problem \ref{Simplified_maximum_likelihood_estimation_for_SE3_synchronization}, and so by the same logic as before we have that $\OMLEval = \MLEval$ and $\Ropt$ is a minimizer of Problem \ref{Simplified_maximum_likelihood_estimation_for_SE3_synchronization}.  The final claim now follows from the optimality of $\Ropt$ for Problem \ref{Simplified_maximum_likelihood_estimation_for_SE3_synchronization} together with equation \eqref{minimizing_value_of_t_from_minimizing_value_of_R}.
\end{proof}

From a practical standpoint, Theorem \ref{certifying_exactness_theorem} serves to identify a class of solutions of the (convex) semidefinite relaxation Problem \ref{dual_semidefinite_relaxation_for_SE3_synchronization_problem} that correspond to \emph{global minimizers} of the nonconvex maximum-likelihood estimation Problem \ref{SE3_synchronization_MLE_NLS_problem}.  This naturally leads us to consider the following two questions:  Under what conditions does Problem \ref{dual_semidefinite_relaxation_for_SE3_synchronization_problem} admit a solution belonging to this class?  And if there does exist such a solution, can we guarantee that we will be able to \emph{recover} it by solving Problem \ref{dual_semidefinite_relaxation_for_SE3_synchronization_problem} using a numerical optimization method?\footnote{Note that Problem \ref{dual_semidefinite_relaxation_for_SE3_synchronization_problem} could conceivably have multiple solutions, only \emph{some} of which belong to the class specified in Theorem \ref{certifying_exactness_theorem}; in that case, it is possible that a numerical optimization method might converge to a minimizer of Problem \ref{dual_semidefinite_relaxation_for_SE3_synchronization_problem} that does \emph{not} correspond to a solution of Problem \ref{SE3_synchronization_MLE_NLS_problem}.}   These questions are addressed by the following:

\begin{prop}[Exact recovery via the semidefinite relaxation Problem \ref{dual_semidefinite_relaxation_for_SE3_synchronization_problem}]
\label{A_sufficient_condition_for_exact_recovery_prop}
Let $\tQ$ be the matrix of the form \eqref{Q_quadratic_form_definition} constructed using the true \emph{(}latent\emph{)} relative transforms $\tpose_{ij} = (\ttran_{ij}, \trot_{ij})$ in \eqref{probabilistic_generative_model_for_noisy_observations}.  There exists a constant $\exactnessBound \triangleq \exactnessBound(\tQ) > 0$ \emph{(}depending upon $\tQ$\emph{)} such that, if $\lVert \nQ - \tQ \rVert_2 < \exactnessBound$, then:
\begin{enumerate}
 \item [$(i)$]  The dual semidefinite relaxation Problem \ref{dual_semidefinite_relaxation_for_SE3_synchronization_problem} has a unique solution $\Zopt$, and
 \item [$(ii)$] $\Zopt = {\Ropt}\transpose \Ropt$, where $\Ropt \in \SO(d)^n$ is a minimizer of the simplified maximum-likelihood estimation Problem \ref{Simplified_maximum_likelihood_estimation_for_SE3_synchronization}.
\end{enumerate}
\end{prop}
\noindent This result is proved in Appendix \ref{A_sufficient_condition_for_exactness_appendix}, using an approach adapted from \citet{Bandeira2016Tightness}.

In short, Proposition \ref{A_sufficient_condition_for_exact_recovery_prop} guarantees that as long as the noise corrupting the available measurements $\npose_{ij}$ in \eqref{probabilistic_generative_model_for_noisy_observations} is not too large (as measured by the spectral norm of the deviation of the data matrix $\nQ$ from its exact latent value $\tQ$),\footnote{Ideally, one would like to have both (i) an explicit (i.e.\ closed-form) expression that lower-bounds the magnitude $\beta$ of the admissible deviation of the data matrix  $\nQ$ from its exact value $\tQ$ (as measured in some suitable norm) and (ii) a concentration inequality \cite{Tropp2015Introduction} (or several) that upper-bounds the probability $p(\lVert \nQ - \tQ \rVert > \delta)$ of large deviations; together, these would enable the derivation of a lower bound on the probability that a given realization of Problem \ref{Simplified_maximum_likelihood_estimation_for_SE3_synchronization} sampled from the generative model \eqref{probabilistic_generative_model_for_noisy_observations} admits an exact semidefinite relaxation \eqref{dual_semidefinite_relaxation_for_SE3_synchronization_optimization}.  While it is possible (with a bit more effort) to derive such lower bounds on $\beta$ using  straightforward (although somewhat tedious) quantitative refinements of the continuity argument given in Appendix \ref{A_sufficient_condition_for_exactness_appendix}, to date the sharpest concentration inequalities that we have been able to derive appear to be significantly suboptimal, and therefore lead to estimates for the probability of exactness that are grossly conservative versus what we observe empirically (cf.\ also the discussion in Remark 4.6 and Sec.\ 5 of \cite{Bandeira2016Tightness}).  Consequently, we have chosen to state Proposition \ref{A_sufficient_condition_for_exact_recovery_prop} as a simple existence result for $\beta$ in order to simplify its presentation and proof, while still providing some rigorous justification for our convex relaxation approach.  

We remark that as a practical matter, we have already shown in our previous work \cite{Carlone2015Lagrangian} (and do so again here in Section \ref{Experimental_results_section}) that Problem \ref{dual_semidefinite_relaxation_for_SE3_synchronization_problem} in fact remains exact with high probability when the measurements $\npose_{ij}$ in \eqref{probabilistic_generative_model_for_noisy_observations} are corrupted with noise up to an order of magnitude greater than what is encountered in typical robotics and computer vision applications; consequently, we leave the derivation of sharper concentration inequalities and explicit lower bounds on the probability of exactness to future research.} we can recover a global minimizer $\Ropt$ of Problem \ref{Simplified_maximum_likelihood_estimation_for_SE3_synchronization} (and hence also a global minimizer $\poseopt = (\topt, \Ropt)$ of the maximum-likelihood estimation Problem \ref{SE3_synchronization_MLE_NLS_problem} via \eqref{minimizing_value_of_t_from_minimizing_value_of_R}) by solving Problem \ref{dual_semidefinite_relaxation_for_SE3_synchronization_problem} using \emph{any} numerical optimization method.  

%% file: The_SE_Sync_Algorithm.tex
\section{The \syncSpace  algorithm}
\label{SE_Sync_algorithm_section}

In light of Proposition \ref{A_sufficient_condition_for_exact_recovery_prop}, our overall strategy in this paper will be to search for exact solutions of the (hard) maximum-likelihood estimation Problem \ref{SE3_synchronization_MLE_NLS_problem} by solving the (convex) semidefinite relaxation Problem \ref{dual_semidefinite_relaxation_for_SE3_synchronization_problem}.  In order to realize this strategy as a practical algorithm, we therefore require (i) a method that is able to solve Problem \ref{dual_semidefinite_relaxation_for_SE3_synchronization_problem} effectively in large-scale real-world problems, and (ii) a rounding procedure that recovers an \emph{optimal} solution of Problem \ref{SE3_synchronization_MLE_NLS_problem} from a solution of Problem \ref{dual_semidefinite_relaxation_for_SE3_synchronization_problem} when exactness obtains, and a \emph{feasible approximate solution} otherwise.  In this section, we develop a pair of algorithms that fulfill these requirements.  Together, these procedures comprise \emph{\sync}, our proposed algorithm for synchronization over the special Euclidean group.

\subsection{Solving the semidefinite relaxation}
\label{optimization_approach_subsection}

 As a semidefinite program, Problem \ref{dual_semidefinite_relaxation_for_SE3_synchronization_problem} can in principle be solved in polynomial time using interior-point methods \cite{Vandenberghe1996Semidefinite,Todd2001Semidefinite}.  In practice, however, the high computational cost of general-purpose semidefinite programming algorithms prevents these methods from scaling effectively to problems in which the dimension of the decision variable $Z$ is greater than a few thousand \cite{Todd2001Semidefinite}.  Unfortunately, typical instances of  Problem \ref{dual_semidefinite_relaxation_for_SE3_synchronization_problem} arising in (for example) robotics and computer vision applications are one to two orders of magnitude larger than this maximum effective problem size, and are therefore well beyond the reach of these general-purpose methods.  To overcome this limitation, in this subsection we develop a specialized optimization  procedure for solving large-scale instances of Problem \ref{dual_semidefinite_relaxation_for_SE3_synchronization_problem} efficiently. We first exploit this problem's low-rank, geometric, and graph-theoretic structure to reduce it to an equivalent optimization problem defined on a low-dimensional Riemannian manifold \cite{Boothby2003Riemannian,Kobayashi1996Foundations}, and then design a fast Riemannian optimization method to solve this reduction efficiently.

\subsubsection{Simplifying Problem \ref{dual_semidefinite_relaxation_for_SE3_synchronization_problem}}
\label{Simplifying_dual_semidefinite_relaxation_section}

\paragraph{Exploiting low-rank structure:}  The dominant computational cost when applying general-purpose semidefinite programming methods to solve Problem \ref{dual_semidefinite_relaxation_for_SE3_synchronization_problem} is the need to store and manipulate expressions involving the (large, dense) matrix variable $Z$.  In particular, the $O(n^3)$ computational cost of multiplying and factoring such expressions quickly becomes intractable as the problem size $n$ increases.  On the other hand, in the case that exactness holds, we know that the actual \emph{solution} $\Zopt$ of Problem \ref{dual_semidefinite_relaxation_for_SE3_synchronization_problem} that we seek has a very concise description in the factored form $\Zopt = {\Ropt}\transpose \Ropt$ for $\Ropt \in \SO(d)^n$.  More generally, even in those cases where exactness fails, minimizers $\Zopt$ of Problem \ref{dual_semidefinite_relaxation_for_SE3_synchronization_problem} typically have a rank $r$ not much greater than $d$, and therefore admit a symmetric rank decomposition $\Zopt = {\Yopt}\transpose {\Yopt}$ for $\Yopt \in \R^{r \times dn}$ with $r \ll dn$.

In a pair of papers, \citet{Burer2003Nonlinear,Burer2005Local} proposed an elegant general approach to exploit the fact that large-scale semidefinite programs often admit such low-rank solutions: simply replace every instance of the decision variable $Z$ with a rank-$r$ product of the form $Y\transpose Y$ to produce a \emph{rank-restricted} version of the original problem.  This substitution has the two-fold effect of (i) dramatically reducing the size of the search space and (ii) rendering the positive semidefiniteness constraint \emph{redundant}, since $Y\transpose Y \succeq 0$ for \emph{any} choice of $Y$.  The resulting rank-restricted form of the problem is thus a low-dimensional \emph{nonlinear} program, rather than a \emph{semidefinite} program.  In the specific case of Problem \ref{dual_semidefinite_relaxation_for_SE3_synchronization_problem}, this produces:
\begin{problem}[Rank-restricted semidefinite relaxation, NLP form]
 \label{rank_restricted_semidefinite_relaxation_NLP_form_problem}
 \begin{equation}
 \label{rank_restricted_semidefinite_relaxation_NLP_form}
 \begin{split}
 &\SDPLRval = \min_{Y \in \R^{r \times dn}} \tr(\nQ Y\transpose Y) \\
\st  &\BDiag_d(Y\transpose Y) = \Diag(I_d, \dotsc, I_d).
\end{split}
 \end{equation}
\end{problem}
Provided that Problem \ref{dual_semidefinite_relaxation_for_SE3_synchronization_problem} admits a solution $\Zopt$ with $\rank(\Zopt) \le r$, we can \emph{recover} such a solution from an optimal solution $\Yopt$ of Problem \ref{rank_restricted_semidefinite_relaxation_NLP_form_problem} according to $\Zopt = {\Yopt}\transpose \Yopt$.

\paragraph{Exploiting geometric structure:} In addition to exploiting Problem \ref{dual_semidefinite_relaxation_for_SE3_synchronization_problem}'s low-rank structure,  following \citet{Boumal2015Riemannian} we also observe that the specific form of the constraints appearing in Problems \ref{dual_semidefinite_relaxation_for_SE3_synchronization_problem} and \ref{rank_restricted_semidefinite_relaxation_NLP_form_problem}  (i.e., that the $d \times d$ block-diagonals of $Z$ and $Y\transpose Y$ must be $I_d$) admits a nice geometric interpretation that can be exploited to further simplify Problem \ref{rank_restricted_semidefinite_relaxation_NLP_form_problem}.  Introducing the block decomposition:
\begin{equation}
 Y \triangleq
 \begin{pmatrix}
Y_1 & \dotsb & Y_n
 \end{pmatrix} \in \R^{r \times dn},
\end{equation}
 the block-diagonal constraints appearing in \eqref{rank_restricted_semidefinite_relaxation_NLP_form} are equivalent to:
\begin{equation}
\label{block_diagonal_constraints_in_rank_restricted_semidefinite_relaxation}
 Y_i\transpose Y_i = I_d, \quad Y_i \in \R^{r \times d};
\end{equation}
i.e., they require that each $Y_i$ be an element of the Stiefel manifold $\Stiefel(d, r)$ in \eqref{Stiefel_manifold_definition}.  Consequently, Problem \ref{rank_restricted_semidefinite_relaxation_NLP_form_problem} can be equivalently formulated as an \emph{unconstrained} Riemannian optimization problem  on a product of Stiefel manifolds:

\begin{problem}[Rank-restricted semidefinite relaxation, Riemannian optimization form]
\label{rank_restricted_semidefinite_relaxation_Riemannian_optimization_form_problem}
 \begin{equation}
 \label{rank_restricted_semidefinite_relaxation_Riemannian_optimization_form}
 \SDPLRval = \min_{Y \in \Stiefel(d, r)^n} \tr(\nQ Y\transpose Y).
 \end{equation}
\end{problem}

\noindent This is the optimization problem that we will actually solve in practice.

\paragraph{Exploiting graph-theoretic structure:} The reduction of Problem \ref{dual_semidefinite_relaxation_for_SE3_synchronization_problem} to Problem \ref{rank_restricted_semidefinite_relaxation_Riemannian_optimization_form_problem} obviates the need to form or manipulate the large, dense matrix variable $Z$ directly.  However the data matrix $\nQ$ that parameterizes each of Problems \ref{Simplified_maximum_likelihood_estimation_for_SE3_synchronization}--\ref{rank_restricted_semidefinite_relaxation_Riemannian_optimization_form_problem} is also dense and of the same order as $Z$, and so presents a similar computational difficulty.  Accordingly, here we develop an analogous concise description of $\nQ$ in terms of sparse matrices (and their inverses) associated with the graph $\directed{G}$.

Equations \eqref{Q_quadratic_form_definition} and \eqref{nQtran_alternative_form} provide a decomposition of $\nQ$ in terms of the sparse matrices $\MeasRotConLap$, $\nT$, and $\tranPrecisions$, and the dense orthogonal projection matrix $\OrthoProjMatrix$.  However, since $\OrthoProjMatrix$ is also a matrix derived from a sparse graph, we might suspect that it too should admit some kind of sparse description.  And indeed, it turns out that $\OrthoProjMatrix$ admits a sparse decomposition as:
\begin{subequations}
\label{computing_orthogonal_projection_operator_eq}
\begin{equation}
\label{QR_decomposition_of_oriented_incidence_matrix}
\redIncMat(\directed{G}) \tranPrecisions^{\frac{1}{2}} = LQ_1
\end{equation}
\begin{equation}
 \label{decomposition_for_orthogonal_projection_operator_into_sparse_matrices}
\OrthoProjMatrix = I_m - \tranPrecisions^{\frac{1}{2}}\redIncMat(\directed{G})\transpose L\tinv L\inv \redIncMat(\directed{G})\tranPrecisions^{\frac{1}{2}}
\end{equation}
\end{subequations}
where equation \eqref{QR_decomposition_of_oriented_incidence_matrix} is a thin LQ decomposition\footnote{This is the transpose of a QR decomposition \cite[Sec.\ 5.2]{Golub1996Matrix}.} of the weighted reduced incidence matrix $\redIncMat(\directed{G}) \tranPrecisions^{\frac{1}{2}}$ of $\directed{G}$.  This result is derived in Appendix \ref{An_alternative_form_for_the_translational_data_matrix_subsection}.  Note that expression \eqref{decomposition_for_orthogonal_projection_operator_into_sparse_matrices} for $\OrthoProjMatrix$ requires only the sparse lower-triangular factor $L$ from \eqref{QR_decomposition_of_oriented_incidence_matrix}, which can be easily and efficiently obtained  (e.g.\ by applying successive Givens rotations \cite[Sec.\ 5.2.1]{Golub1996Matrix} directly to $\redIncMat(\directed{G}) \tranPrecisions^{\frac{1}{2}}$ itself).

Together, equations \eqref{Q_quadratic_form_definition}, \eqref{nQtran_alternative_form}, and \eqref{decomposition_for_orthogonal_projection_operator_into_sparse_matrices} provide a concise description of $\nQ$ in terms of sparse matrices, as desired. We will exploit this decomposition in Section \ref{Riemannian_optimization_method_subsection} to design a fast Riemannian optimization method for solving Problem \ref{rank_restricted_semidefinite_relaxation_Riemannian_optimization_form_problem}.

\subsubsection{The Riemannian Staircase}

At this point, it is again instructive to compare Problem \ref{rank_restricted_semidefinite_relaxation_Riemannian_optimization_form_problem} with the simplified maximum-likelihood estimation Problem \ref{Simplified_maximum_likelihood_estimation_for_SE3_synchronization} and its relaxation Problem \ref{dual_semidefinite_relaxation_for_SE3_synchronization_problem}.  Since the (special) orthogonal matrices satisfy condition \eqref{Stiefel_manifold_definition} with $k = n = d$, we have the set of inclusions:
\begin{equation}
\label{hierarchy_of_inclusions_for_rank_restricted_semidefinite_relaxation_feasible_sets}
 \SO(d) \subset \Orthogonal(d) = \Stiefel(d,d) \subset \Stiefel(d, d+1) \subset \dotsb \subset \Stiefel(d, r) \subset \dotsb
\end{equation}
and we can therefore view the set of rank-restricted Riemannian optimization problems \eqref{rank_restricted_semidefinite_relaxation_Riemannian_optimization_form} as comprising 
a \emph{hierarchy} of relaxations of the maximum-likelihood estimation \eqref{Simplified_maximum_likelihood_estimation_for_SE3_synchronization_optimization_problem} that are intermediate between Problem \ref{Orthogonal_relaxation_of_the_MLE_problem} and Problem \ref{dual_semidefinite_relaxation_for_SE3_synchronization_problem} for $d < r < dn$.  However, unlike Problem \ref{dual_semidefinite_relaxation_for_SE3_synchronization_problem}, the various instantiations of Problem \ref{rank_restricted_semidefinite_relaxation_Riemannian_optimization_form_problem} are \emph{nonconvex} due to the (re)introduction of the quadratic orthonormality constraints \eqref{Stiefel_manifold_definition}.  It may therefore not be clear whether anything has really been gained by relaxing Problem \ref{Simplified_maximum_likelihood_estimation_for_SE3_synchronization} to Problem \ref{rank_restricted_semidefinite_relaxation_Riemannian_optimization_form_problem}, since it appears that we may have simply replaced one difficult nonconvex optimization problem with another.  The following remarkable result (Corollary 8 of \citet{Boumal2106Nonconvex}) justifies this approach:

\begin{prop}[A sufficient condition for global optimality in Problem \ref{rank_restricted_semidefinite_relaxation_Riemannian_optimization_form_problem}]
\label{Global_optima_of_rank_restricted_semidefinite_relaxation_Riemannian_optimization_problem_prop}
 If $Y \in \Stiefel(d, r)^n$ is a \emph{(}row\emph{)} rank-deficient second-order critical point\footnote{That is, a point satisfying $\grad F(Y) = 0$ and $\Hess F(Y) \succeq 0$ (cf.\ \eqref{function_and_derivatives_for_rank_restricted_Riemannian_form_of_semidefinite_relaxation}--\eqref{Riemannian_Hessian_vector_product_expression}).} of Problem \ref{rank_restricted_semidefinite_relaxation_Riemannian_optimization_form_problem}, then $Y$ is a global minimizer of Problem \ref{rank_restricted_semidefinite_relaxation_Riemannian_optimization_form_problem} and $\Zopt = Y\transpose Y$ is a solution of the dual semidefinite relaxation Problem \ref{dual_semidefinite_relaxation_for_SE3_synchronization_problem}. 
 \end{prop}

Proposition \ref{Global_optima_of_rank_restricted_semidefinite_relaxation_Riemannian_optimization_problem_prop} immediately suggests an algorithm for recovering solutions $\Zopt$ of Problem \ref{dual_semidefinite_relaxation_for_SE3_synchronization_problem} from Problem \ref{rank_restricted_semidefinite_relaxation_Riemannian_optimization_form_problem}: simply apply a second-order Riemannian optimization method to search successively higher levels of the hierarchy of relaxations \eqref{rank_restricted_semidefinite_relaxation_Riemannian_optimization_form} until a \emph{rank-deficient} second-order critical point is obtained.\footnote{Note that since \emph{every} $Y \in \Stiefel(d,r)^n$ is row rank-deficient for $r > dn$, this procedure is guaranteed to recover an optimal solution after searching at most $dn + 1$ levels of the hierarchy \eqref{rank_restricted_semidefinite_relaxation_Riemannian_optimization_form}.}  This algorithm, the \emph{Riemannian Staircase} \cite{Boumal2015Riemannian,Boumal2106Nonconvex}, is summarized as Algorithm \ref{Riemannian_Staircase_algorithm}.  We emphasize that while Algorithm \ref{Riemannian_Staircase_algorithm} may require searching up to $O(dn)$ levels of the hierarchy \eqref{rank_restricted_semidefinite_relaxation_Riemannian_optimization_form} in the worst case, in practice this a gross overestimate; typically one or two ``stairs'' suffice.

\stepcounter{footnote}

\begin{algorithm}[t]
\caption{The Riemannian Staircase}
\label{Riemannian_Staircase_algorithm}
\begin{algorithmic}[1]
\Input An initial point $Y \in \Stiefel(d, r_0)^n$, $r_0 \ge d + 1$.
\Output A minimizer $\Yopt$ of Problem \ref{rank_restricted_semidefinite_relaxation_Riemannian_optimization_form_problem} corresponding to a solution $\Zopt = {\Yopt}\transpose \Yopt$ of Problem \ref{dual_semidefinite_relaxation_for_SE3_synchronization_problem}.
\Function{RiemannianStaircase}{$Y$}
\For{$r = r_0, \dotsc,  dn + 1$ }
\State\label{second_order_critical_points_search_for_Riemannian_Staircase} Starting at $Y$, apply a Riemannian optimization method\footnotemark[\value{footnote}] to identify a second-order 
\Statex[2] critical point $\Yopt \in \Stiefel(d, r)^n$ of Problem \ref{rank_restricted_semidefinite_relaxation_Riemannian_optimization_form_problem}.
\If{$\rank(\Yopt) < r$}
\State \Return $\Yopt$ 
\Else
\State Set $Y \leftarrow 
\begin{pmatrix}
 \Yopt \\
 0_{1 \times dn}
\end{pmatrix}$.  
\EndIf
\EndFor
\EndFunction
%
 \end{algorithmic}
\end{algorithm}
\footnotetext{For example, the second-order Riemannian trust-region method \cite[Algorithm 3]{Boumal2016Global}.}

\pagebreak

\subsubsection{A Riemannian optimization method for Problem \ref{rank_restricted_semidefinite_relaxation_Riemannian_optimization_form_problem}}
\label{Riemannian_optimization_method_subsection}
 
 Proposition \ref{Global_optima_of_rank_restricted_semidefinite_relaxation_Riemannian_optimization_problem_prop} and the Riemannian Staircase (Algorithm \ref{Riemannian_Staircase_algorithm}) provide a means of obtaining \emph{global} minimizers of Problem \ref{dual_semidefinite_relaxation_for_SE3_synchronization_problem} by \emph{locally} searching for second-order critical points of Problem \ref{rank_restricted_semidefinite_relaxation_Riemannian_optimization_form_problem}.  In this subsection, we design a Riemannian optimization method that will enable us to rapidly identify these critical points in practice.
 
Equations \eqref{Q_quadratic_form_definition}, \eqref{nQtran_alternative_form}, and \eqref{decomposition_for_orthogonal_projection_operator_into_sparse_matrices} provide an efficient means of computing products with $\nQ$ \emph{without} the need to form $\nQ$ explicitly by performing a sequence of sparse matrix multiplications and sparse triangular solves.  This operation is sufficient to evaluate the objective appearing in Problem \ref{rank_restricted_semidefinite_relaxation_Riemannian_optimization_form_problem}, as well as its gradient and Hessian-vector products when it is considered as a function on the ambient Euclidean space $\R^{r \times dn}$:


\begin{subequations}
 \label{function_and_derivatives_for_rank_restricted_Riemannian_form_of_semidefinite_relaxation}
\begin{equation}
F(Y) \triangleq \tr(\nQ Y\transpose Y) 
\end{equation}
\begin{equation}
\label{Euclidean_gradient}
\nabla F(Y) = 2 Y \nQ 
\end{equation}
\begin{equation}
\label{Euclidean_Hessian_vector_product}
\nabla^2 F(Y)[\dot{Y}] = 2\dot{Y} \nQ.
\end{equation}
\end{subequations}
Furthermore, there are simple relations between the ambient Euclidean gradient and Hessian-vector products in \eqref{Euclidean_gradient} and \eqref{Euclidean_Hessian_vector_product} and their corresponding Riemannian counterparts when $F(\cdot)$ is viewed as a function restricted to the embedded submanifold $\Stiefel(d, r)^n \subset \R^{r \times dn}$.  With reference to the orthogonal projection operator onto the tangent space of $\Stiefel(d,r)^n$ at $Y$ \cite[eq.\ (2.3)]{Edelman1998Geometry}:
\begin{equation}
\label{Stiefel_manifold_orthogonal_projection_operator}
\begin{split}
 \proj_Y &\colon T_Y\left(\R^{r \times dn} \right)\to T_Y\left(\Stiefel(d, r)^n \right) \\
 \proj_Y(X) &= X - Y \SymBlockDiag_d(Y\transpose X)
\end{split}
\end{equation}
the Riemannian gradient $\grad F(Y)$ is simply the orthogonal projection of the ambient Euclidean gradient $\nabla F(Y)$ (cf.\ \cite[eq.\ (3.37)]{Absil2009Optimization}): 
\begin{equation}
\label{Riemannian_gradient_expression}
 \grad F(Y) = \proj_Y \nabla F(Y).
\end{equation}
Similarly, the Riemannian Hessian-vector product $\Hess F(Y)[\dot{Y}]$ can be obtained as the orthogonal projection of the ambient directional derivative of the gradient vector field $\grad F(Y)$ in the direction of $\dot{Y}$ (cf.\ \cite[eq.\ (5.15)]{Absil2009Optimization}). A straightforward computation shows that this is given by:\footnote{We point out that equations \eqref{Stiefel_manifold_orthogonal_projection_operator}, \eqref{Riemannian_gradient_expression}, and \eqref{Riemannian_Hessian_vector_product_expression} correspond to equations (7), (8), and (9) in \cite{Boumal2015Riemannian}, with the caveat that \cite{Boumal2015Riemannian}'s definition of $Y$ is the \emph{transpose} of ours.  Our notation follows the more common convention (cf.\ e.g.  \cite{Edelman1998Geometry}) that elements of a Stiefel manifold are matrices with orthonormal \emph{columns}, rather than \emph{rows}.}
\begin{equation}
\label{Riemannian_Hessian_vector_product_expression}
 \begin{split}
\Hess F(Y)[\dot{Y}] &= \proj_Y \left( \DirectionalDerivative \left[\grad F(Y) \right][\dot{Y}] \right) \\
&= \proj_Y\left(\nabla^2 F(Y)[\dot{Y}] - \dot{Y} \SymBlockDiag_d\left( Y\transpose \nabla F(Y) \right) \right).
 \end{split}
\end{equation}

Equations \eqref{Q_quadratic_form_definition}, \eqref{nQtran_alternative_form}, and \eqref{computing_orthogonal_projection_operator_eq}--\eqref{Riemannian_Hessian_vector_product_expression} provide an efficient means of computing $F(Y)$, $\grad F(Y)$, and $\Hess F(Y)[\dot{Y}]$.  Consequently, we propose to employ a \emph{truncated-Newton trust-region} optimization method \cite{Dembo1983Truncated,Steihaug1983Conjugate,Nash2000Survey} to solve Problem \ref{rank_restricted_semidefinite_relaxation_Riemannian_optimization_form_problem}; this approach will enable  us to exploit the availability of an efficient routine for computing Hessian-vector products $\Hess F(Y)[\dot{Y}]$ to implement a second-order optimization method \emph{without} the need to explicitly form or factor the dense matrix $\Hess F(Y)$.  Moreover, truncated-Newton methods comprise the current state of the art for superlinear large-scale unconstrained nonlinear programming \cite[Sec.\ 7.1]{Nocedal2006Numerical}, and are therefore ideally suited for solving large-scale instances of \eqref{rank_restricted_semidefinite_relaxation_Riemannian_optimization_form}.  Accordingly, we will apply the truncated-Newton \emph{Riemannian Trust-Region} (RTR) method \cite{Absil2007Trust,Boumal2016Global} to efficiently compute high-precision\footnote{The requirement of high precision here is not superfluous: because Proposition \ref{Global_optima_of_rank_restricted_semidefinite_relaxation_Riemannian_optimization_problem_prop} requires the identification of \emph{rank-deficient} second-order critical points, whatever local search technique we apply to Problem \ref{rank_restricted_semidefinite_relaxation_Riemannian_optimization_form_problem} must be capable of numerically approximating a critical point precisely enough that its rank can be correctly determined.} estimates of second-order critical points of Problem \ref{rank_restricted_semidefinite_relaxation_Riemannian_optimization_form_problem}.

\begin{algorithm}[t]
\caption{Rounding procedure for solutions of Problem \ref{rank_restricted_semidefinite_relaxation_Riemannian_optimization_form_problem}}
\label{Rounding_algorithm}
\begin{algorithmic}[1]
\Input An optimal solution $\Yopt \in \Stiefel(d, r)^n$ of Problem \ref{rank_restricted_semidefinite_relaxation_Riemannian_optimization_form_problem}.
\Output A feasible point $\RotEst \in \SO(d)^n$.
\Function{RoundSolution}{$\Yopt$}
\State Compute a rank-$d$ truncated singular value decomposition $U_d \varXi_d V_d\transpose$ for $\Yopt$ \
\Statex[1] and assign $\RotEst \leftarrow \varXi_d V_d\transpose$.
\State Set $N_{+} \leftarrow \lvert \lbrace \RotEst_i \mid \det(\RotEst_i)  >0 \rbrace \rvert$.
\If{$N_{+} < \lceil \frac{n}{2} \rceil$}
\State $\RotEst \leftarrow \Diag(1, \dotsc, 1, -1) \RotEst$.
\EndIf
\For{$i = 1, \dotsc, n$}
\State Set $\RotEst_i \leftarrow$ \Call{NearestRotation}{$\RotEst_i$}.
\EndFor
\State \Return $\RotEst$
\EndFunction
 \end{algorithmic}
\end{algorithm}

\subsection{Rounding the solution}
\label{rounding_procedure_subsection}

In the previous subsection, we described an efficient algorithmic approach for computing minimizers $\Yopt \in \Stiefel(d,r)^n$ of Problem \ref{rank_restricted_semidefinite_relaxation_Riemannian_optimization_form_problem} that correspond to solutions $\Zopt = {\Yopt}\transpose \Yopt$ of Problem \ref{dual_semidefinite_relaxation_for_SE3_synchronization_problem}.  However, our ultimate goal is to extract an optimal solution $\Ropt \in \SO(d)^n$ of Problem \ref{Simplified_maximum_likelihood_estimation_for_SE3_synchronization} from $\Zopt$ whenever exactness holds, and a \emph{feasible approximate solution} $\RotEst \in \SO(d)^n$ otherwise.  In this subsection, we develop a rounding procedure satisfying these criteria.  To begin, let us consider the  case in which exactness obtains; here:
\begin{equation}
\label{equality_of_optimal_solutions_for_typical_case_in_rounding_procedure}
{\Yopt}\transpose \Yopt = \Zopt = {\Ropt}\transpose \Ropt
\end{equation}
for some optimal solution $\Ropt \in \SO(d)^n$ of Problem \ref{Simplified_maximum_likelihood_estimation_for_SE3_synchronization}.  Since $\rank(\Ropt) = d$, this implies that $\rank(\Yopt) = d$ as well.  Consequently, letting
\begin{equation}
\label{rank_d_truncated_singular_value_decomposition}
 \Yopt = U_d \varXi_d V_d\transpose
\end{equation}
denote a (rank-$d$) thin singular value decomposition \cite[Sec.\ 2.5.3]{Golub1996Matrix} of $\Yopt$, and defining
\begin{equation}
\label{Ybar_definition}
 \bar{Y} \triangleq \varXi_d V_d\transpose \in \R^{d \times dn},
\end{equation}
it follows from substituting \eqref{rank_d_truncated_singular_value_decomposition} into \eqref{equality_of_optimal_solutions_for_typical_case_in_rounding_procedure} that
\begin{equation}
\label{equality_of_Ybar_with}
\bar{Y}\transpose \bar{Y} = \Zopt = {\Ropt}\transpose \Ropt.
\end{equation}
Equation \eqref{equality_of_Ybar_with} implies that the $d \times d$ block-diagonal of $\bar{Y}\transpose \bar{Y}$ satisfies $\bar{Y}_i\transpose \bar{Y}_i = I_d$ for all $i \in [n]$, i.e.\ $\bar{Y} \in \Orthogonal(d)^n$.  Similarly, comparing the elements of the first block rows of $\bar{Y}\transpose \bar{Y}$ and ${\Ropt}\transpose \Ropt$ in \eqref{equality_of_Ybar_with} shows that $\bar{Y}_1 \transpose \bar{Y}_j = \Ropt_1 \Ropt_j$ for all $j \in [n]$.  Left-multiplying this set of equalities by $\bar{Y}_1$ and letting $A = \bar{Y}_1 \Ropt_1$ then gives:
\begin{equation}
\label{extracting_R_from_thin_SVD}
\bar{Y} = A \Ropt, \quad A \in \Orthogonal(d).
\end{equation}
Since any product of the form $A \Ropt$ with $A \in \SO(d)$ is \emph{also} an optimal solution of Problem \ref{Simplified_maximum_likelihood_estimation_for_SE3_synchronization} (by gauge symmetry), equation \eqref{extracting_R_from_thin_SVD} shows that $\bar{Y}$ as defined in \eqref{Ybar_definition} is optimal provided that $\bar{Y} \in \SO(d)$ specifically.  Furthermore, if this is not the case, we can always  make it so by left-multiplying $\bar{Y}$ by any orientation-reversing element of $\Orthogonal(d)$, for example $\Diag(1, \dotsc, 1, -1)$.  Thus, equations \eqref{rank_d_truncated_singular_value_decomposition}--\eqref{extracting_R_from_thin_SVD} give a straightforward means of recovering an optimal solution of Problem \ref{Simplified_maximum_likelihood_estimation_for_SE3_synchronization} from $\Yopt$ whenever exactness holds.

Moreover, this procedure can be straightforwardly generalized to the case that exactness fails, thereby producing a convenient rounding scheme. Specifically, we can consider the right-hand side of \eqref{Ybar_definition} as taken from a rank-$d$ \emph{truncated} singular value decomposition of $\Yopt$ (so that $\bar{Y}$ is an orthogonal transform of the best rank-$d$ approximation of $\Yopt$), multiply $\bar{Y}$ by an orientation-reversing element of $\Orthogonal(d)$ according to whether a majority of its block elements have positive or negative determinant, and then project each of $\bar{Y}$'s blocks to the nearest rotation matrix.\footnote{This is the \emph{special orthogonal Procrustes} problem, which admits a simple closed-form solution based upon the singular value decomposition   \cite{Hanson1981Analysis,Umeyama1991Least}.}  This generalized rounding scheme is formalized as Algorithm \ref{Rounding_algorithm}.

\subsection{The complete algorithm}

Combining the efficient optimization approach of Section \ref{optimization_approach_subsection} with the rounding procedure of Section \ref{rounding_procedure_subsection} produces \emph{\sync} (Algorithm \ref{SE_sync_algorithm}), our proposed algorithm for synchronization over the special Euclidean group.

\begin{algorithm}[t]
\caption{The \syncSpace  algorithm}
\label{SE_sync_algorithm}
\begin{algorithmic}[1]
\Input  An initial point $Y \in \Stiefel(d, r_0)^n$, $r_0 \ge d + 1$.
\Output A feasible estimate $\PoseEst\in \SE(d)^n$ for the maximum-likelihood estimation Problem \ref{SE3_synchronization_MLE_NLS_problem} and the lower bound $\SDPval$ for Problem \ref{SE3_synchronization_MLE_NLS_problem}'s optimal value.
\Function{\sync}{$Y$}
\State Set $\Yopt \leftarrow \Call{RiemannianStaircase}{Y}$.  \Comment{Solve Problems \ref{dual_semidefinite_relaxation_for_SE3_synchronization_problem} \& \ref{rank_restricted_semidefinite_relaxation_Riemannian_optimization_form_problem}}  
\State Set $\SDPval \leftarrow F(\nQ {\Yopt}\transpose \Yopt)$. \Comment{$\Zopt = {\Yopt}\transpose \Yopt$}
\State Set $\RotEst \leftarrow \Call{RoundSolution}{\Yopt}$.  
\State Recover the optimal translational estimates $\TranEst$ corresponding to $\RotEst$ via \eqref{minimizing_value_of_t_from_minimizing_value_of_R}.

\State Set $\PoseEst \leftarrow (\TranEst, \RotEst)$.
\State \Return $\left \lbrace  \PoseEst, \SDPval \right \rbrace$
\EndFunction
 \end{algorithmic}
\end{algorithm}

When applied to an instance of $\SE(d)$ synchronization, \syncSpace  returns a feasible point $\PoseEst \in \SE(d)^n$ for the maximum-likelihood estimation Problem \ref{SE3_synchronization_MLE_NLS_problem} and the lower bound $\SDPval \le \MLEval$ for Problem \ref{SE3_synchronization_MLE_NLS_problem}'s optimal value.  This lower bound provides an \emph{upper} bound on the suboptimality of \emph{any} feasible point $x = (\tran, \rot) \in \SE(d)^n$ as a solution of Problem \ref{SE3_synchronization_MLE_NLS_problem} according to:
\begin{equation}
 \label{suboptimality_bound_for_Problem_1}
 F(\nQ \rot \transpose \rot) - \SDPval \ge F(\nQ \rot\transpose \rot) - \MLEval.
\end{equation}
Furthermore, in the  case that Problem \ref{dual_semidefinite_relaxation_for_SE3_synchronization_problem} is exact, the  estimate $\PoseEst = (\TranEst, \RotEst) \in \SE(d)^n$ returned by Algorithm \ref{SE_sync_algorithm} \emph{attains} this lower bound:
\begin{equation}
 \label{MLE_attains_lower_bound_in_case_of_exactness}
 F(\nQ \RotEst\transpose \RotEst) = \SDPval.
\end{equation}
Consequently, verifying \emph{a posteriori} that \eqref{MLE_attains_lower_bound_in_case_of_exactness} holds provides a \emph{computational certificate} of $\PoseEst$'s correctness as a solution of the maximum-likelihood estimation Problem \ref{SE3_synchronization_MLE_NLS_problem}.  \syncSpace is thus a \emph{certifiably correct} algorithm for $\SE(d)$ synchronization, as claimed.

%% file: Experimental_Results.tex

\newcommand{\E}[1]{\times 10^{#1}}

\def \syncChol{\textsf{\sync-Chol}\xspace}
\def \syncQR{\textsf{\sync-QR}\xspace}
\def \GN{\textsf{GN}}

\def \rand{\textsf{rand}}
\def \chordal{\textsf{chord}}

\def \syncQRchord{\syncQR \textsf{+} \chordal}
\def \syncQRrand{\syncQR \textsf{+} \rand}

\def \GNchord{\GN\ \textsf{+} \chordal}

\def \cube{\textsf{cube}\xspace}

\def \sphere{\textsf{sphere}\xspace}
\def \torus{\textsf{torus}\xspace}
\def \grid{\textsf{grid}\xspace}
\def \garage{\textsf{garage}\xspace}
\def \cubicle{\textsf{cubicle\xspace}}
\def \rim{\textsf{rim}\xspace}

\def \pLC{p_{LC}}
\def \sT{\sigma_T}
\def \sR{\sigma_R}

\section{Experimental results}
\label{Experimental_results_section}

In this section we evaluate \sync's performance on a variety of special Euclidean synchronization problems drawn from the motivating application of 3D pose-graph simultaneous localization and mapping (SLAM).   We consider two versions of the algorithm that differ in their approach to evaluating products with the orthogonal projection matrix $\OrthoProjMatrix$: the first (\syncChol) employs the cached Cholesky factor $L$ in \eqref{QR_decomposition_of_oriented_incidence_matrix} to evaluate this product by working right to left through the sequence of sparse matrix multiplications and upper- and lower-triangular solves in \eqref{decomposition_for_orthogonal_projection_operator_into_sparse_matrices}, while the second (\syncQR) follows \eqref{Computing_projection_of_v_onto_Omega_half_At_via_LS}, using QR decomposition to solve the linear least-squares problem for $w^*$.   As a basis for comparison, we also evaluate the performance of the Gauss-Newton method (\GN), the \emph{de facto} standard approach for solving pose-graph SLAM problems in robotics \cite{Kuemmerle2011g20,Kaess2012iSAM2ijrr,Grisetti2010Tutorial}.

All of the following experiments are performed on a Dell Precision 5510 laptop with an Intel Xeon E3-1505M  2.80 GHz processor and 16 GB of RAM running Ubuntu 16.04.  Our experimental implementations of \syncSpace\footnote{Available online: \url{https://github.com/david-m-rosen/SE-Sync}.} and the Gauss-Newton method are written in MATLAB, and the former takes advantage of the truncated-Newton RTR method \cite{Absil2007Trust,Boumal2016Global} supplied by the Manopt toolbox \cite{Boumal2014Manopt} to implement the Riemannian optimization method of Section \ref{Riemannian_optimization_method_subsection}.  Each optimization algorithm is limited to a maximum of 500 (outer) iterations, and convergence is declared whenever the norm of the (Riemannian) gradient is less than $10^{-2}$ or the relative decrease in function value between two subsequent (accepted) iterations is less than $10^{-5}$; additionally, each outer iteration of RTR is limited to a maximum of 500 Hessian-vector product operations.  The Gauss-Newton method is initialized using the \emph{chordal initialization},  a state-of-the-art method for bootstrapping an initial solution in SLAM and bundle adjustment problems \cite{Carlone2015Initialization,Martinec2007Robust}, and we set $r_0 = 5$ in the Riemannian Staircase (Algorithm \ref{Riemannian_Staircase_algorithm}). Finally, since \syncSpace is based upon solving the (convex) semidefinite relaxation Problem \ref{dual_semidefinite_relaxation_for_SE3_synchronization_problem}, it does not require a high-quality initialization in order to reach a globally optimal solution; nevertheless, it can still benefit (in terms of reduced computation time) from being supplied with one.  Consequently, in the following experiments we employ two initialization procedures in conjunction with each version of \syncSpace: the first (\rand) simply samples a point uniformly randomly from $\Stiefel(d, r_0)^n$, while the second (\chordal) supplies the same chordal initialization that the Gauss-Newton method receives, in order to enable a fair comparison of the algorithms' computational speeds.

\subsection{\textsf{Cube} experiments}
\label{Cube_experiments_subsection}

In this first set of experiments, we are interested in investigating how the performance of \syncSpace is affected by factors such as measurement noise, measurement density, and problem size.  To that end, we conduct a set of simulation studies that enable us to interrogate each of these factors individually.  Concretely, we revisit the \cube experiments considered in our previous work \cite{Carlone2015Lagrangian}; this scenario simulates a robot traveling along a rectilinear path through a regular cubical lattice with a side length of $s$ poses (Fig.\ \ref{cube_example_fig}).  An odometric measurement is available between each pair of sequential poses, and measurements between nearby nonsequential poses are available with probability $\pLC$; the measurement values $\npose_{ij} = (\ntran_{ij}, \nrot_{ij})$ themselves are sampled according to \eqref{probabilistic_generative_model_for_noisy_observations}.  We fix default values for these parameters at $\kappa = 16.67$ (corresponding to an expected angular RMS error of 10 degrees for the rotational measurements $\nrot_{ij}$, cf.\  \eqref{standard_deviation_of_rotation_angle} in Appendix \ref{Isotropic_Langevin_distribution_appendix}), $\tau = 75$ (corresponding to an expected RMS error of .20 m for the translational measurements $\ntran_{ij}$), $\pLC = .1$, and $s = 10$ (corresponding to a default problem size of $1000$ poses), and consider the effect of varying each of them individually;  our complete dataset consists of 50 realizations of the \cube sampled from the generative model just described for \emph{each} joint setting of the parameters $\kappa$, $\tau$, $\pLC$, and $s$.  Results for these experiments are shown in Fig.\ \ref{Cube_experiments_figure}.

\begin{figure}
 \center
\includegraphics[width = .5\textwidth]{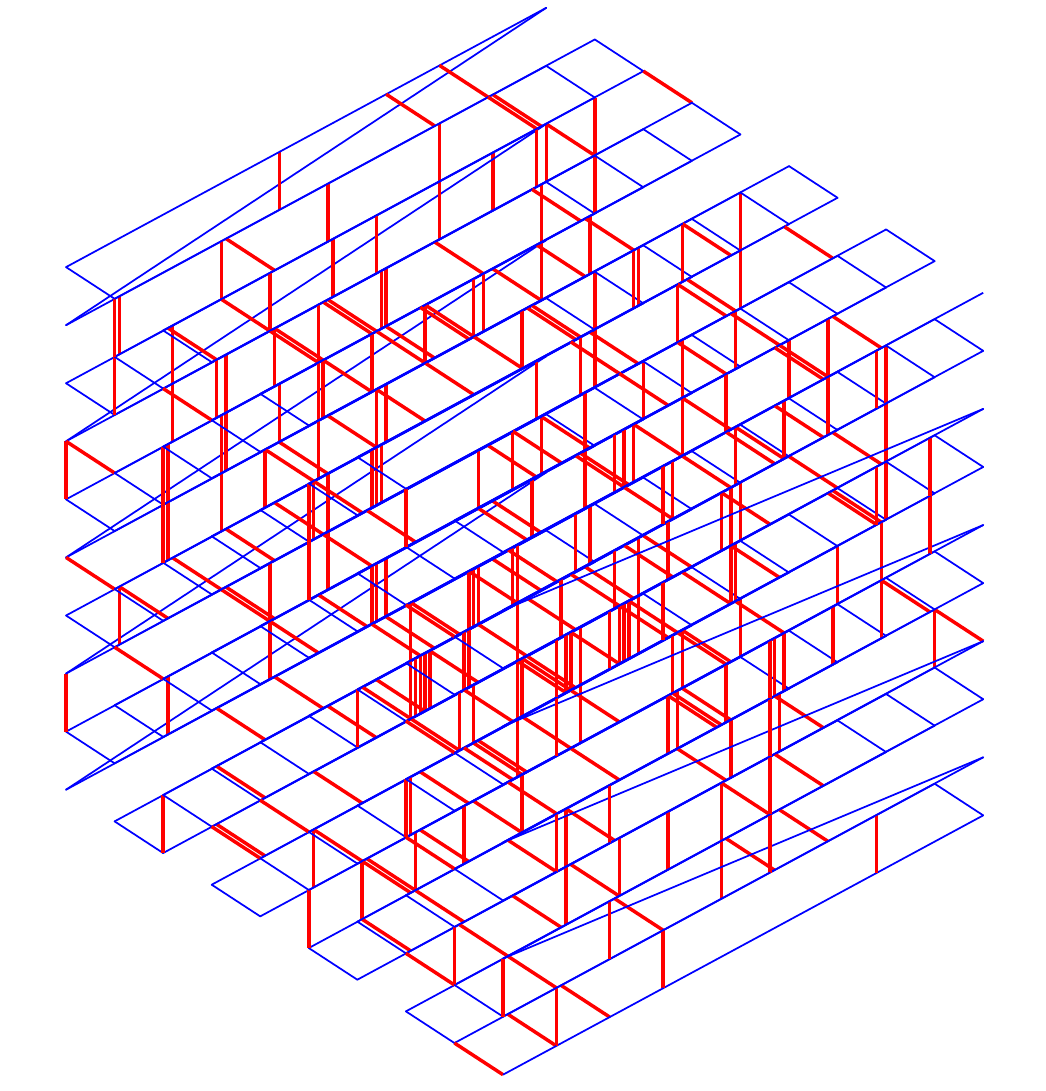}
\caption{The ground truth configuration of an instance of the \cube dataset with $s = 10$ and $\pLC = .1$.  The robot's trajectory (with associated odometric measurements) is drawn in blue, and loop closure observations in red.}
\label{cube_example_fig}
\end{figure}

\def \FigWidth{.31\textwidth}

\begin{figure}
\center
\subfigure[Varying $\kappa$ while holding $\tau = 75$, $p_{LC} = .1$, $s = 10$.]{
\includegraphics[width=\FigWidth]{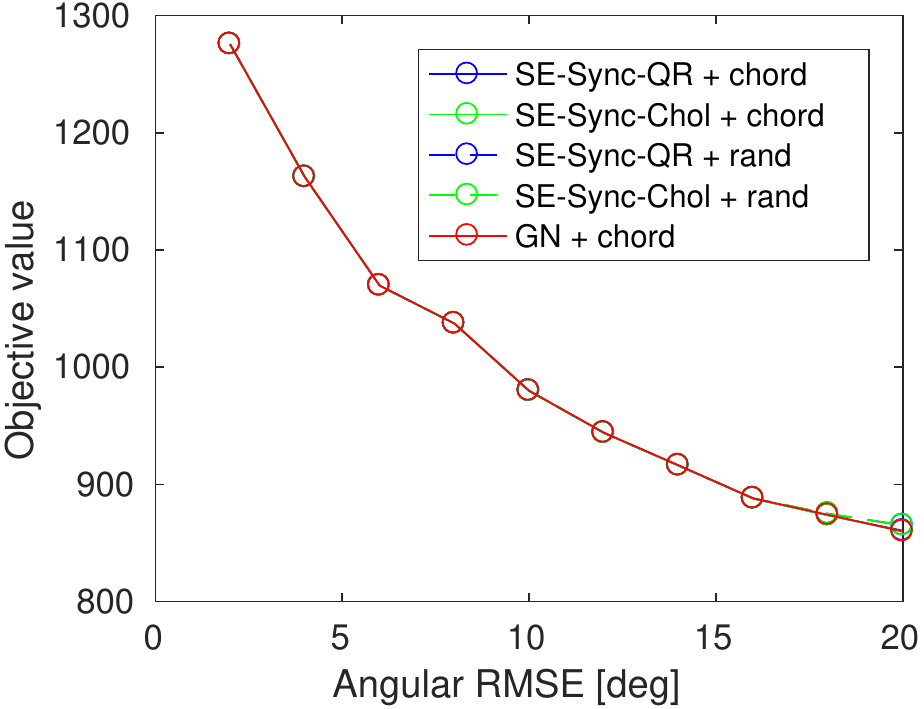} \quad 
\includegraphics[width=\FigWidth]{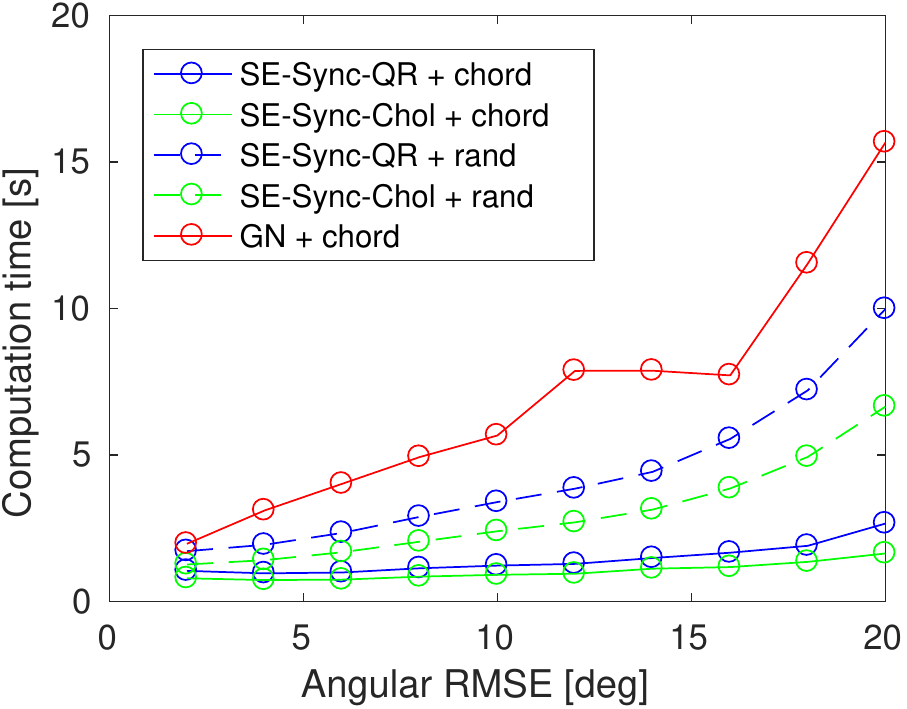} \quad
\includegraphics[width=\FigWidth]{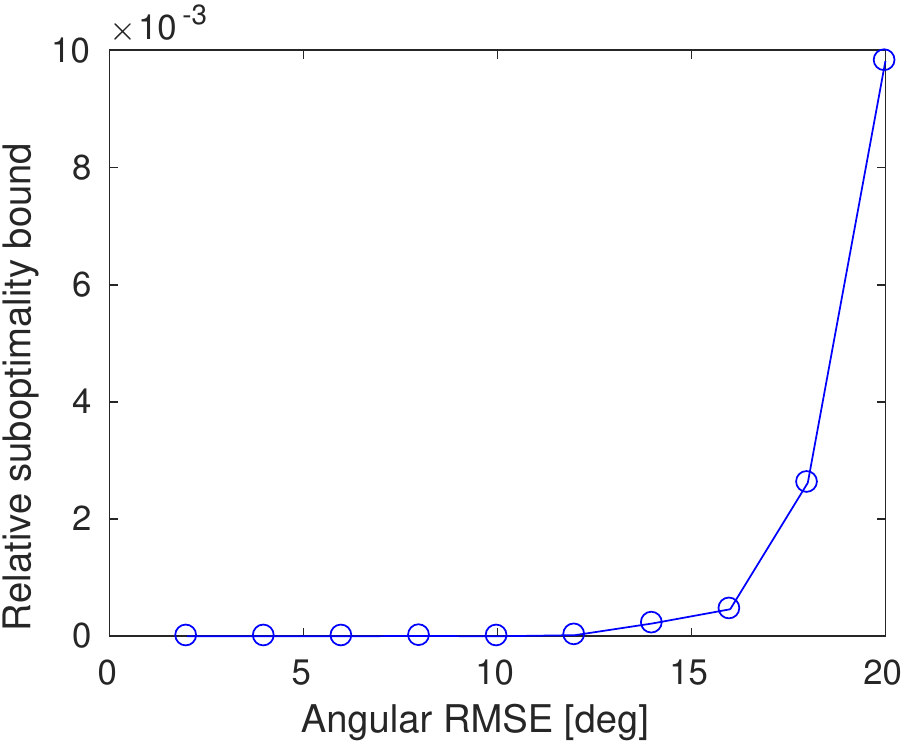}
\label{varying_rotational_noise_level}
} \\

\subfigure[Varying $\tau$ while holding $\kappa = 16.67$, $p_{LC} = .1$, $s = 10$.]{
\includegraphics[width=\FigWidth]{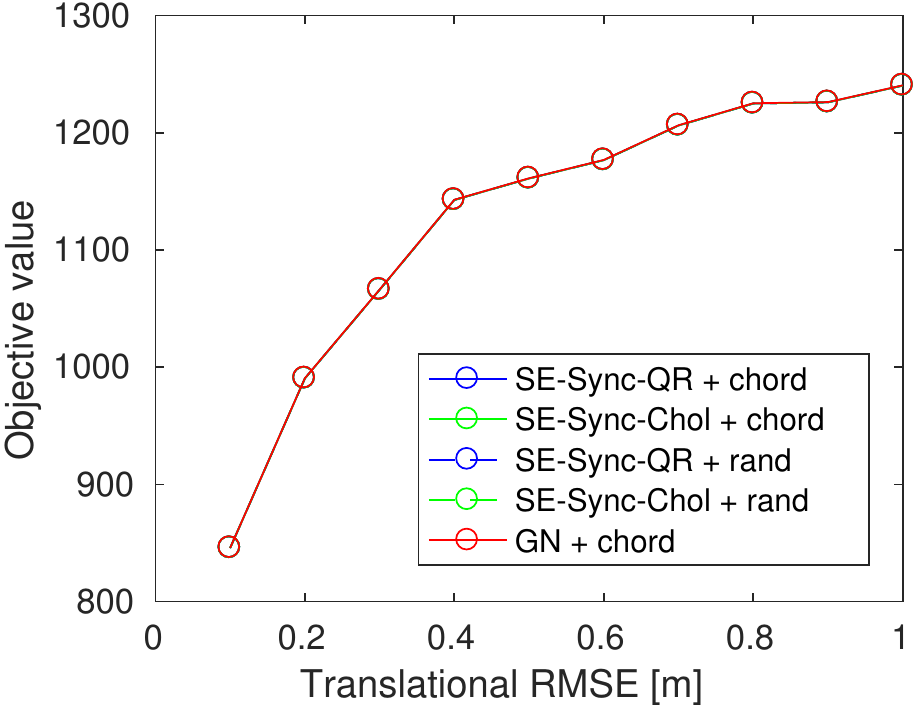} \quad 
\includegraphics[width=\FigWidth]{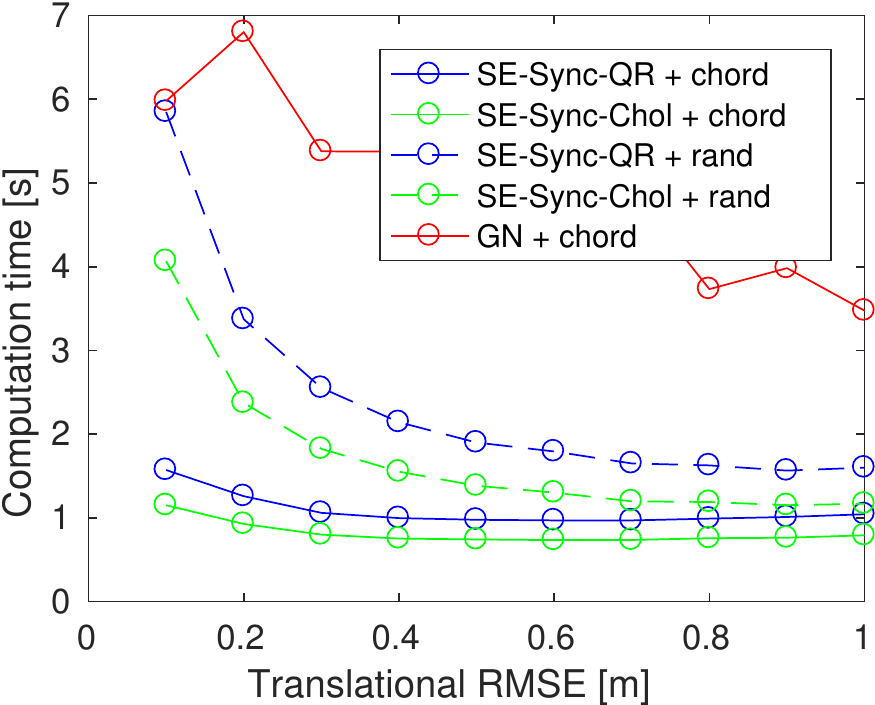} \quad
\includegraphics[width=\FigWidth]{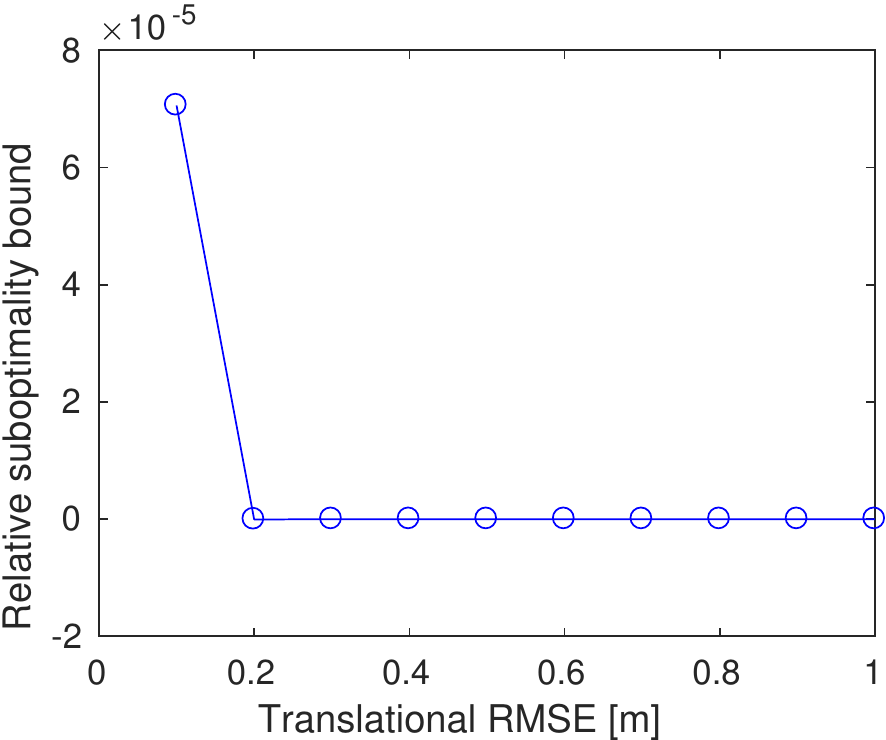}
\label{varying_translational_noise_level}
} \\

\subfigure[Varying $p_{LC}$ while holding  $\kappa = 16.67$, $\tau = 75$, $s = 10$.]{
\includegraphics[width=\FigWidth]{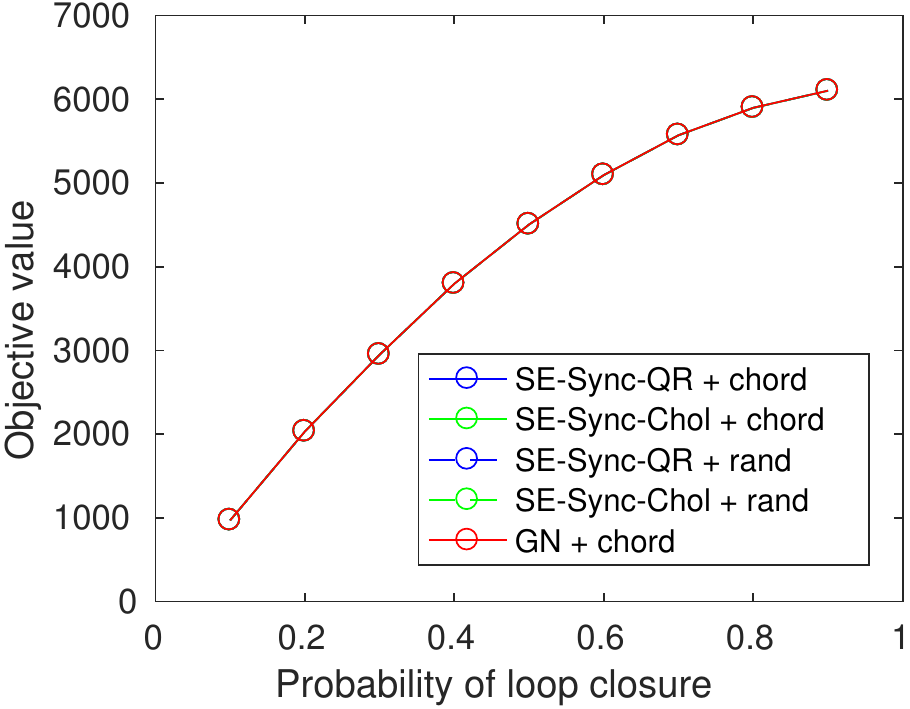} \quad 
\includegraphics[width=\FigWidth]{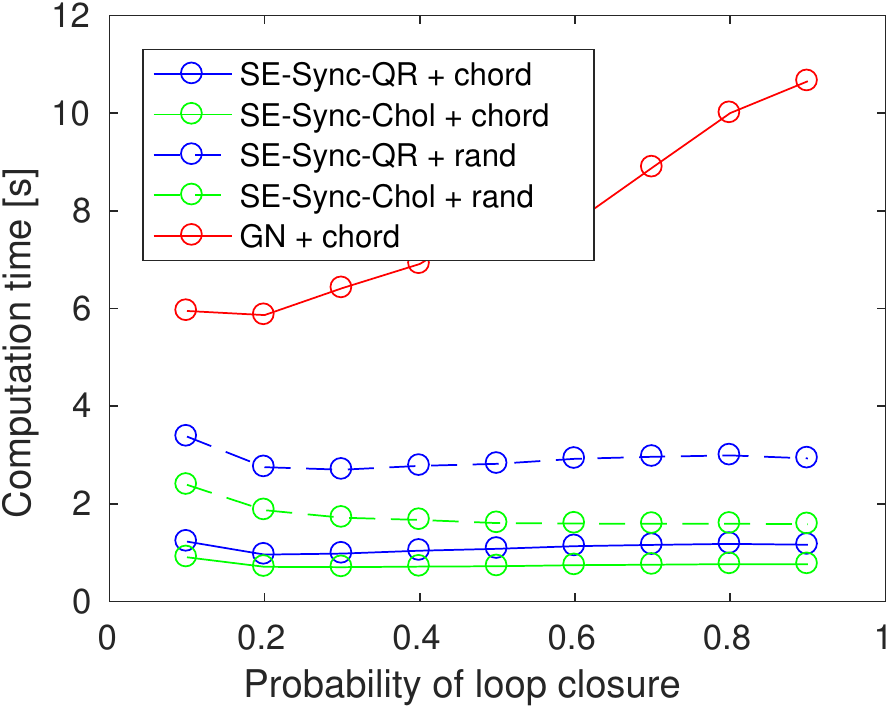} \quad
\includegraphics[width=\FigWidth]{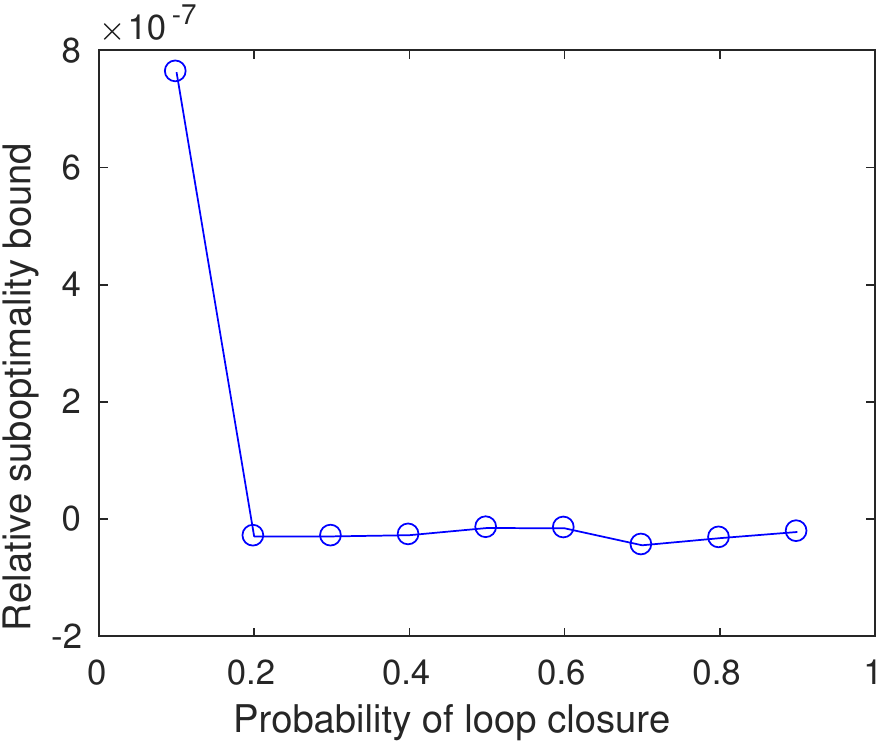}
\label{varying_loop_closure_prob}
} \\

\subfigure[Varying $s$ while holding $\kappa = 16.67$,  $\tau = 75$, $p_{LC} = .1$.]{
\includegraphics[width=\FigWidth]{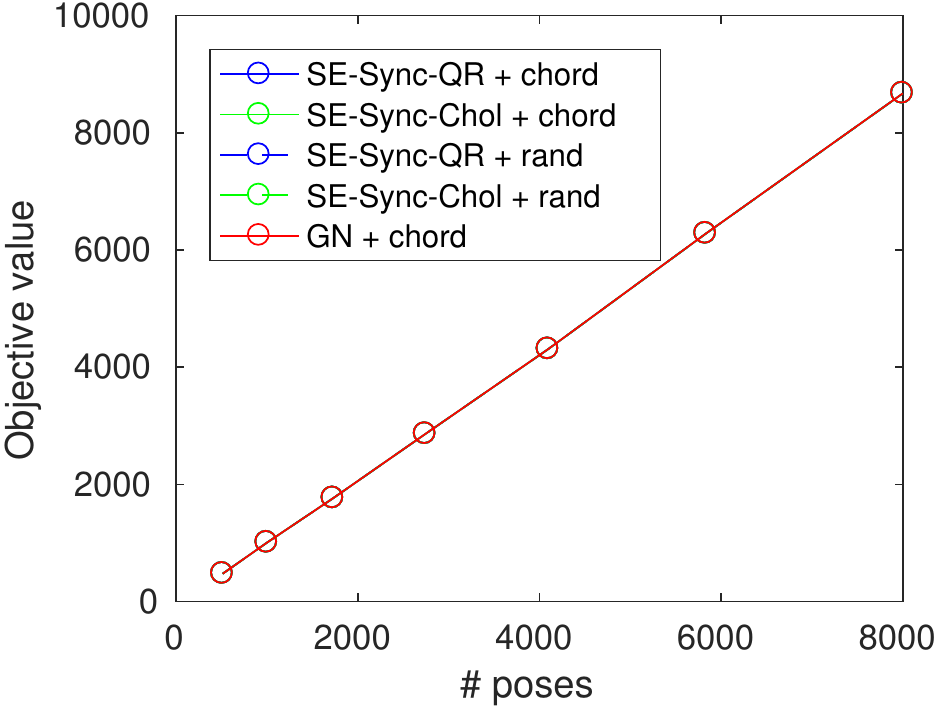} \quad 
\includegraphics[width=\FigWidth]{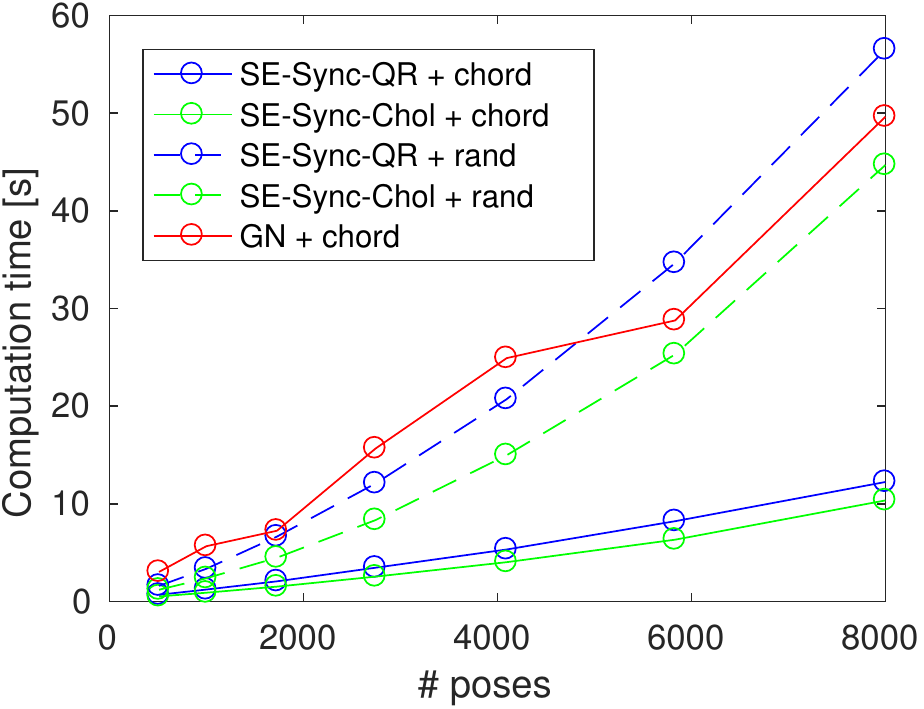} \quad
\includegraphics[width=\FigWidth]{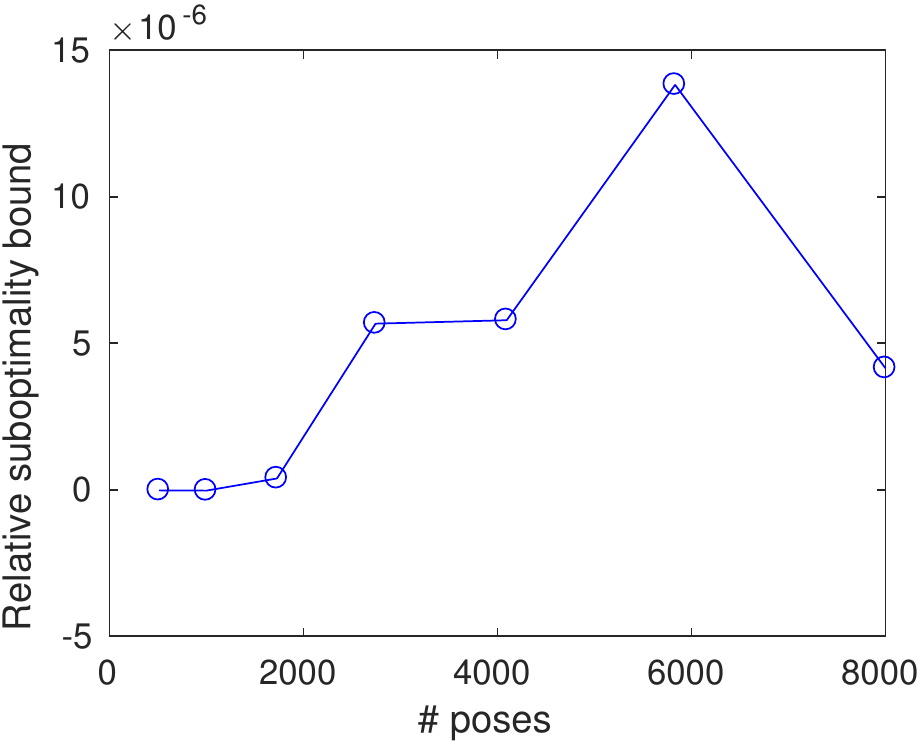}
\label{varying_number_nodes_fig}
}
\caption{Results for the \cube experiments.  These figures plot the mean of the objective values (left column) and elapsed computation times (center column) attained by the Gauss-Newton and \syncSpace algorithms, as well as the upper bound $(F(\nQ \RotEst\transpose \RotEst) - \SDPval) / \SDPval$ for the relative suboptimality of the solution recovered by \syncSpace (right column), for 50 realizations of the \cube dataset as functions of the measurement precisions $\kappa$ (first row) and $\tau$ (second row), the loop closure probability $\pLC$ (third row), and the problem size (fourth row).}

\label{Cube_experiments_figure}

\end{figure}

Consistent with our previous findings \cite{Carlone2015Lagrangian}, these results suggest that the exactness of the semidefinite relaxation \eqref{dual_semidefinite_relaxation_for_SE3_synchronization_optimization} depends primarily upon the level of noise corrupting the rotational observations $\nrot_{ij}$ in \eqref{probabilistic_generative_model_for_noisy_observations}.  Furthermore, we see from Fig.\ \ref{varying_rotational_noise_level} that in these experiments,  exactness obtains for rotational noise with a root-mean-square angular error up to about $15$ degrees; this is roughly an order of magnitude greater than the level of noise affecting sensors typically deployed in robotics and computer vision applications, which provides strong empirical evidence that \syncSpace is capable of recovering certifiably globally optimal solutions of pose-graph SLAM problems under ``reasonable'' operating conditions.\footnote{Interestingly, the Gauss-Newton method with chordal initialization (\GNchord) appears to obtain remarkably good solutions in the high-rotational-noise regime, although here we cannot certify their correctness, as our solution verification methods \cite{Carlone2015Lagrangian} depend upon the same semidefinite relaxation \eqref{dual_semidefinite_relaxation_for_SE3_synchronization_optimization} as does \syncSpace. }


In addition to its ability  to recover certifiably optimal solutions, these experiments also show that \syncSpace is considerably faster than the Gauss-Newton approach that underpins current state-of-the-art pose-graph SLAM algorithms \cite{Kuemmerle2011g20,Kaess2012iSAM2ijrr,Grisetti2010Tutorial,Rosen2014RISE}.  Indeed, examining the plots in the center column of Fig.\ \ref{Cube_experiments_figure} reveals that in almost all of our experimental scenarios, \syncSpace is able to both compute \emph{and} certify a globally optimal solution in less time than the Gauss-Newton method requires to simply \emph{obtain} an estimate when supplied with a high-quality chordal initialization.  Moreover, restricting attention only to those scenarios in which \syncSpace also employs the chordal initialization (in order to enable a fair comparison with the Gauss-Newton method), we see that both \syncChol and \syncQR are consistently many-fold faster than Gauss-Newton, and that this speed differential increases superlinearly with the problem size (Fig.\ \ref{varying_number_nodes_fig}).

Given that \syncSpace performs direct \emph{global} optimization, whereas Gauss-Newton is a purely \emph{local} search technique, it may be somewhat counterintuitive that the former is consistently so much faster than the latter.  However, we can attribute this improved performance to two critical design decisions that distinguish these techniques.  First, \syncSpace makes use of the \emph{exact} Hessian (cf.\ Section \ref{Riemannian_optimization_method_subsection}), whereas Gauss-Newton, by construction, implicitly uses an approximation whose quality degrades in the presence of either large measurement residuals \emph{or} strong nonlinearities in the underlying objective function (cf.\ e.g.\ \cite[Sec.\ III-B]{Rosen2014RISE}), both of which are typical features of SLAM problems.  This implies that the quadratic model function (cf.\ e.g.\ \cite[Chp.\ 2]{Nocedal2006Numerical}) that \syncSpace employs better captures the shape of the underlying objective than the one used by Gauss-Newton, so that the former is capable of computing higher-quality update steps.  Second, and more significantly, \syncSpace makes use of a truncated-Newton method (RTR), which avoids the need to explicitly form or factor the (Riemannian) Hessian $\Hess F(Y)$; instead, at each iteration this approach \emph{approximately} solves the Newton equations using a truncated conjugate gradient algorithm \cite[Chp.\ 10]{Golub1996Matrix},\footnote{Hence the nomenclature.} and only computes this approximate solution as accurately as is necessary to ensure adequate progress towards a critical point with each applied update.  The result is that RTR requires only a few sparse matrix-vector multiplications in order to obtain each update step; moreover, equations \eqref{Q_quadratic_form_definition}, \eqref{nQtran_alternative_form}, and \eqref{computing_orthogonal_projection_operator_eq} show that the constituent matrices involved in these products are \emph{constant}, and can therefore be precomputed and cached at the beginning of the algorithm.  In contrast, the Gauss-Newton method must recompute and refactor the Jacobian at \emph{each} iteration, which is considerably more expensive.


\subsection{SLAM benchmark datasets}
\label{SLAM_benchmarks_subsection}

The experiments in the previous section made extensive use of simulated \cube datasets to investigate the effects of  measurement noise, measurement density, and problem size on \syncSpace's performance.  In this next set of experiments, we evaluate \syncSpace on a suite of larger and more heterogeneous 3D pose-graph SLAM benchmarks that better represent the distribution of problems encountered in real-world SLAM applications.   The first three of these (the \sphere, \torus, and \grid datasets) are also synthetic (although generated using an observation model different from \eqref{probabilistic_generative_model_for_noisy_observations}), while the latter three (the \garage, \cubicle, and \rim datasets) are large-scale real-world examples (Fig.\ \ref{SLAM_datasets_fig}).  For the purpose of these experiments, we restrict attention to the version of \syncSpace employing QR factorization and the chordal initialization (\syncQRchord), and once again compare it with the Gauss-Newton method (\GNchord).  Results for these experiments are shown in Table \ref{3D_SLAM_benchmarks_table}.

\def \FigWidth2{.305\textwidth}
\begin{figure}
 \center 
 \subfigure[sphere]{\includegraphics[width = \FigWidth2]{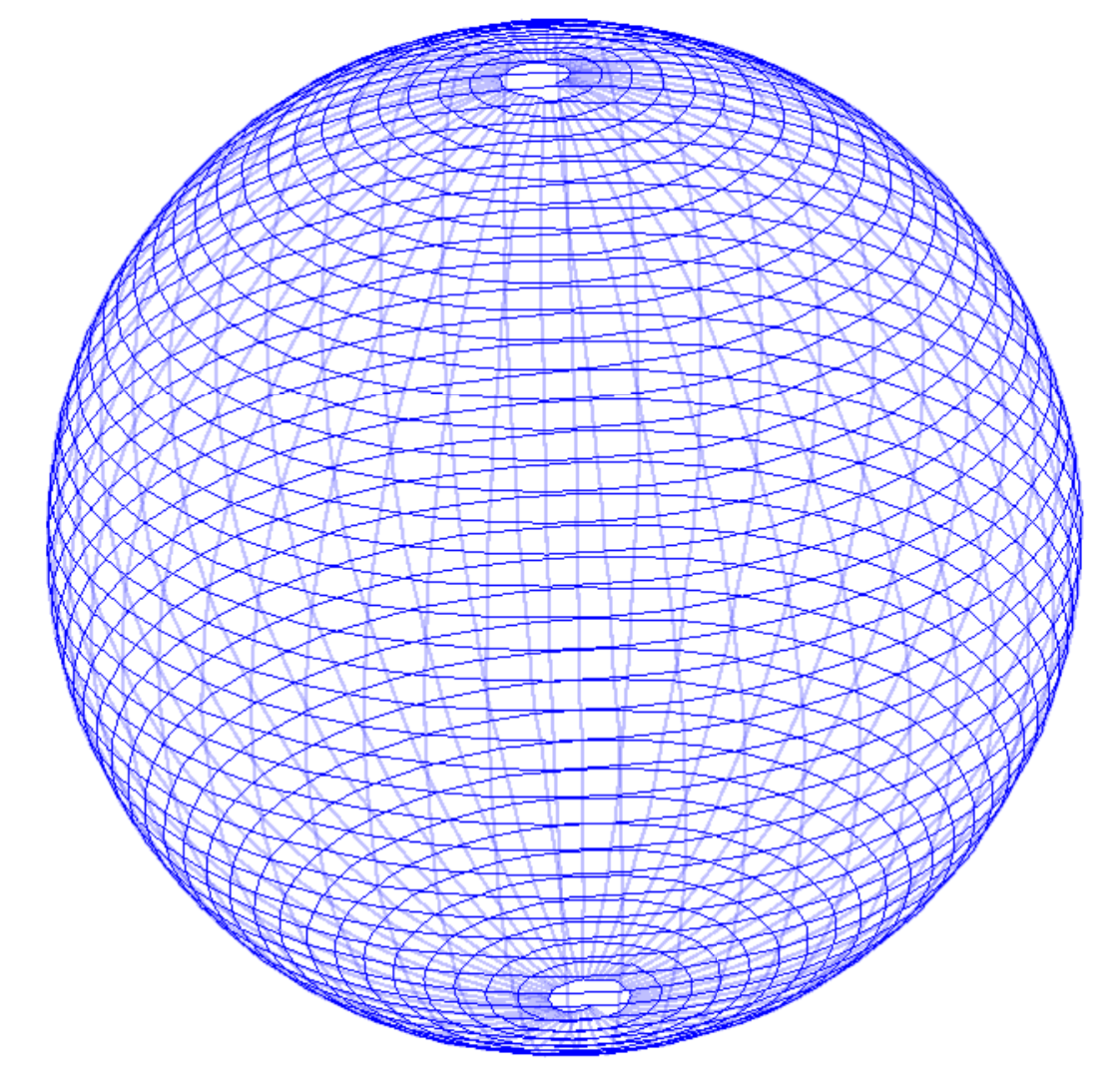}\label{sphere_fig}} \quad 
 \subfigure[torus]{\includegraphics[width = \FigWidth2]{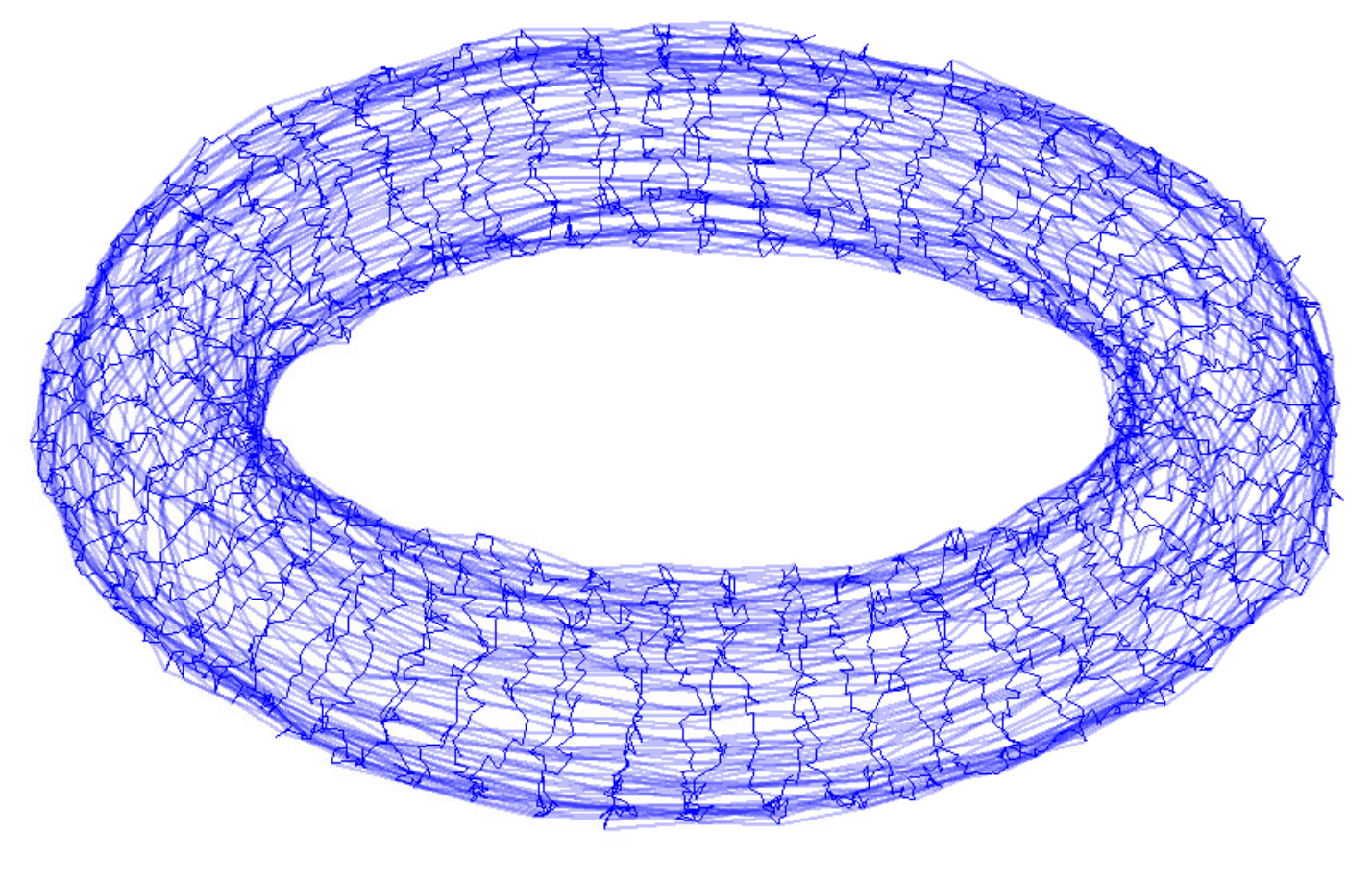}\label{torus_fig}} \quad 
 \subfigure[grid]{\includegraphics[width = \FigWidth2]{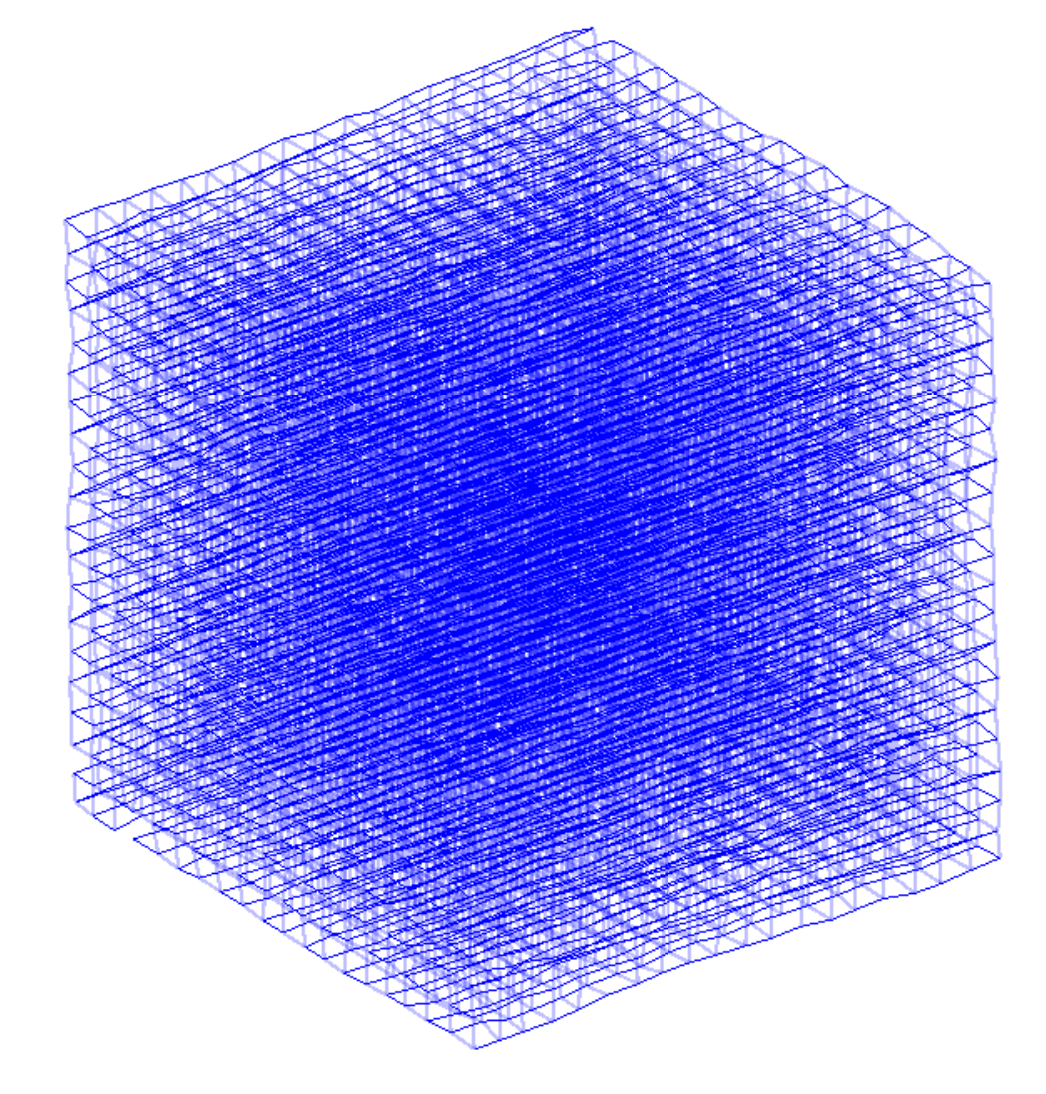}\label{grid_fig}} \\
   
  \subfigure[garage]{\includegraphics[width = \FigWidth2]{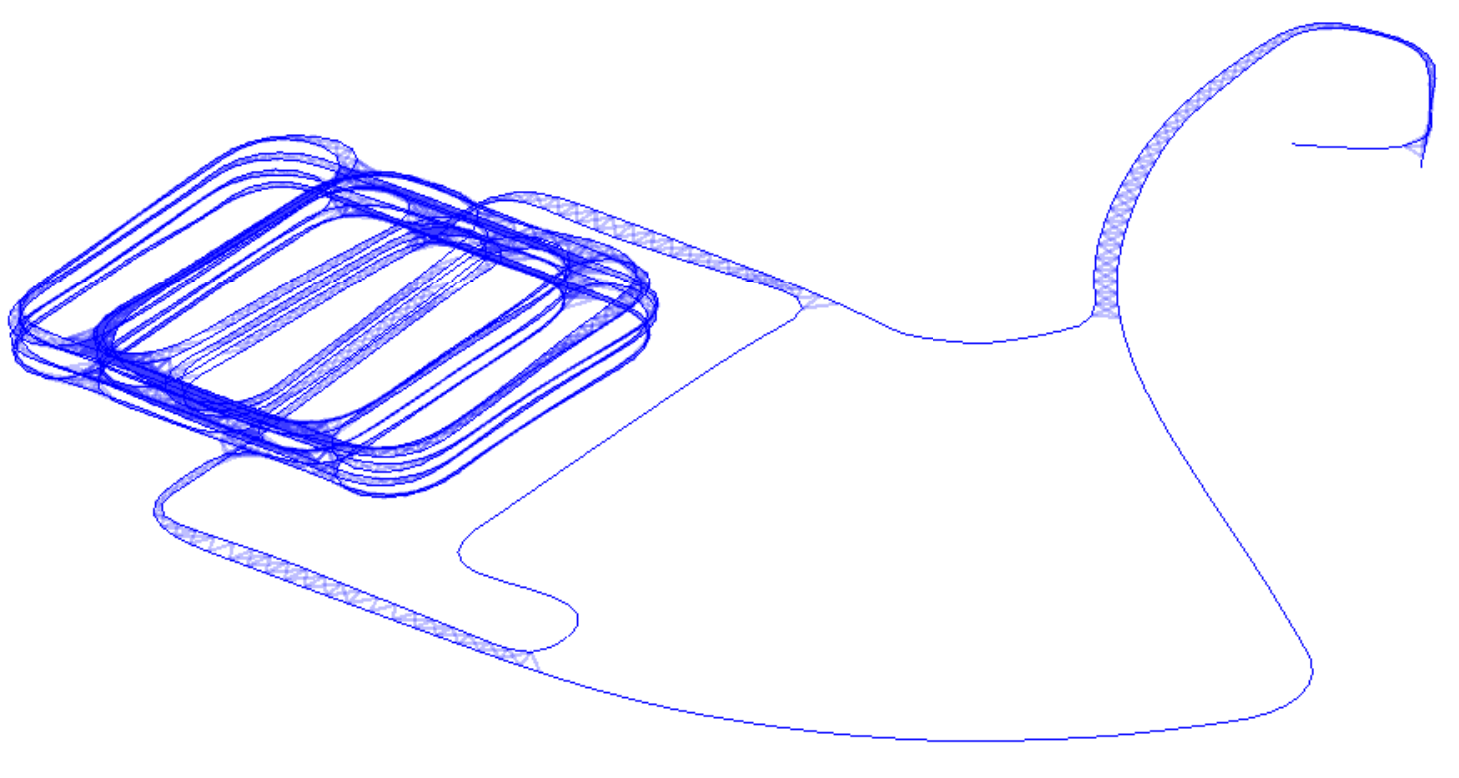}\label{garage_fig}} \quad 
  \subfigure[cubicle]{\includegraphics[width = \FigWidth2]{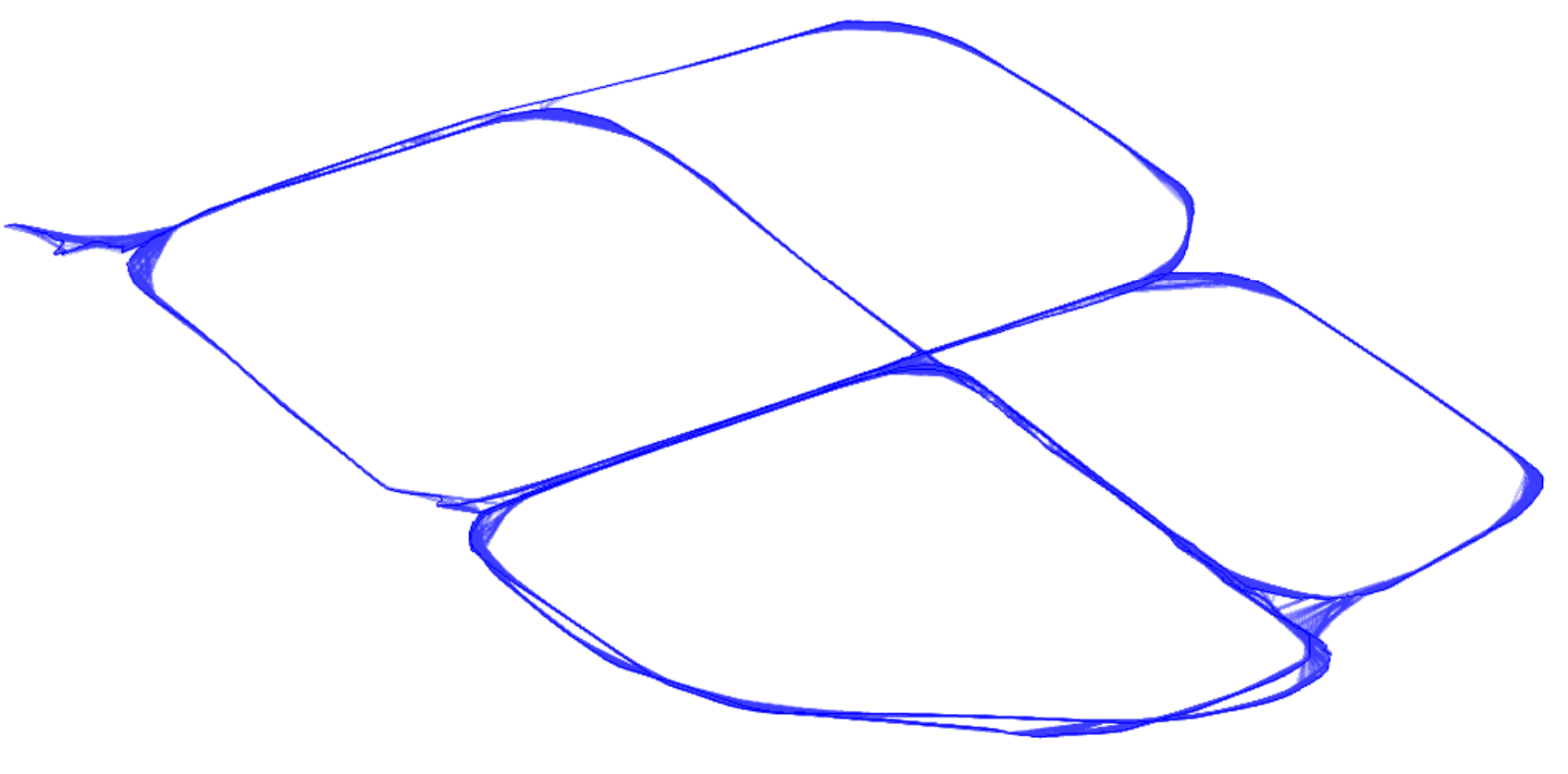}\label{cubicle_fig}} \quad
 \subfigure[rim]{\includegraphics[width = \FigWidth2]{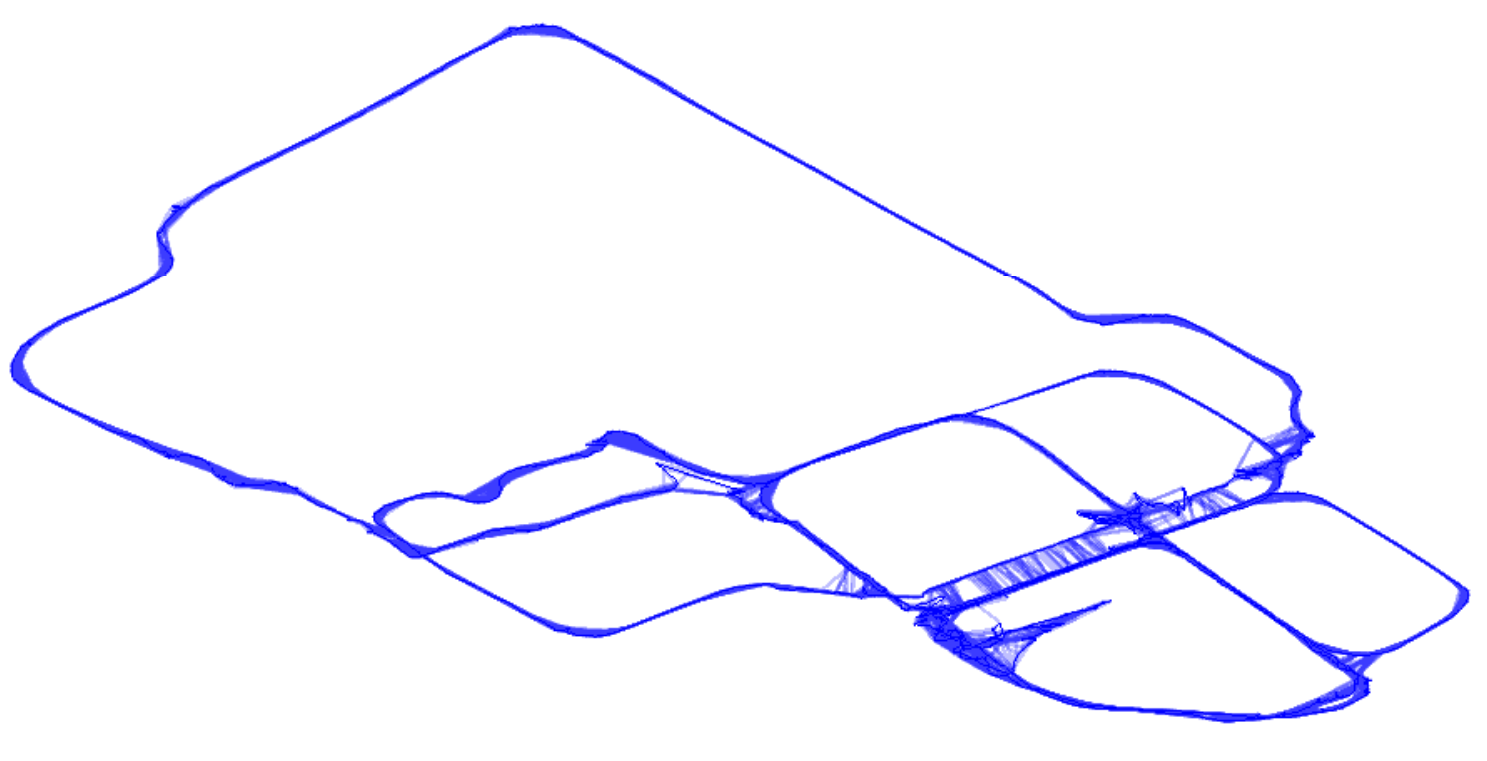}\label{rim_fig}}
 \caption{Globally optimal solutions for the SLAM benchmark datasets listed in Table \ref{3D_SLAM_benchmarks_table}.}
 \label{SLAM_datasets_fig}
\end{figure}

\begin{table}
\footnotesize
\centerline{
\begin{tabular}{|c || c | c || c|c || c|c | c ||}
\hline
 & & & \multicolumn{2}{|c||}{Gauss-Newton} & \multicolumn{3}{|c||}{\syncSpace}   \\ 
& \# Poses & \# Measurements &  Objective\ value & Time [s] & Objective\ value & Time [s] & Max.\ suboptimality  \\
\hline  
sphere & $2500$ & $4949$ & $1.687\E{3}$ & $14.98$ & $1.687\E{3}$ & $2.81$ & $1.410\E{-11}$ \\
\hline
torus & $5000$ & $9048$ & $2.423\E{4}$ & $31.94$ & $2.423\E{4}$ & $5.67$ & $7.276\E{-12}$ \\
\hline
grid  & $8000$ & $22236$ & $8.432\E{4}$ & $130.35$ & $8.432\E{4}$ & $22.37$ & $4.366\E{-11}$ \\
\hline
garage & $1661$ & $6275$ & $1.263\E{0}$ & $17.81$ & $1.263\E{0}$ & $5.33$ & $2.097\E{-11}$ \\
\hline 
cubicle & $5750$ & $16869$ & $7.171\E{2}$ & $136.86$ & $7.171\E{2}$ & $13.08$ & $1.603\E{-11}$ \\
\hline
rim & 10195 & 29743 & $5.461\E{3}$ & $575.42$ & $5.461\E{3}$ & $36.66$ & $5.639\E{-11}$ \\
\hline
\end{tabular}
}
\caption{Results for the SLAM benchmark datasets.} 
\label{3D_SLAM_benchmarks_table}
\end{table}

On each of these examples, both \syncSpace and the Gauss-Newton method converged to the same (globally optimal) solution; however, consistent with our findings in Section \ref{Cube_experiments_subsection}, \syncSpace did so considerably faster, outperforming Gauss-Newton in terms of computational speed by a factor of between $3.3$ (on \garage, the smallest dataset) and $15.7$ (on \rim, the largest dataset).  These results further support our claim that \syncSpace provides an effective  means of recovering globally optimal pose-graph SLAM solutions under real-world operating conditions, and does so considerably faster than current state-of-the-art Gauss-Newton-based alternatives.

%% file: Conclusion.tex
\section{Conclusion}

In this paper we presented \syncSpace, a certifiably correct algorithm for synchronization over the special Euclidean group.  Our algorithm is based upon the development of a semidefinite relaxation of the special Euclidean synchronization problem whose minimizer provides an \emph{exact}, \emph{globally optimal} solution so long as the magnitude of the noise corrupting the available measurements falls below a certain critical threshold, and employs a specialized, structure-exploiting Riemannian optimization method to solve large-scale instances of this semidefinite relaxation efficiently.  Experimental evaluation on a variety of simulated and real-world datasets drawn from the motivating application of pose-graph SLAM  shows that \syncSpace is capable of recovering globally optimal solutions when the available measurements are corrupted by noise up to an order of magnitude greater than that typically encountered in robotics and computer vision applications, and does so more than an order of magnitude faster than the Gauss-Newton-based approach that forms the basis of current state-of-the-art techniques.

In addition to enabling the computation of certifiably correct solutions under \emph{nominal} operating conditions, we believe that \syncSpace may also be extended to support provably robust and statistically efficient estimation in the case that some fraction of the available measurements $\npose_{ij}$ in  \eqref{probabilistic_generative_model_for_noisy_observations} are contaminated by outliers.   Our basis for this belief is the observation that Proposition \ref{A_sufficient_condition_for_exact_recovery_prop} together with the experimental results of Section \ref{Experimental_results_section} imply that, under typical operating conditions, the maximum-likelihood estimation Problem~\ref{SE3_synchronization_MLE_NLS_problem} is equivalent to a \emph{low-rank} convex program with a \emph{linear} observation model and a \emph{compact} feasible set (Problem \ref{dual_semidefinite_relaxation_for_SE3_synchronization_problem}); in contrast to general nonlinear estimation, this class of problems enjoys a beautiful geometric structure \cite{Chandrasekaran2012Convex} that has already been shown to enable remarkably robust recovery, even in the presence of gross contamination \cite{Candes2011Robust,Zhou2010Stable}.  We intend to investigate this possibility in future research.

Finally, while the specific relaxation \eqref{dual_semidefinite_relaxation_for_SE3_synchronization_optimization} underpinning \syncSpace was obtained by exploiting the well-known Lagrangian duality between quadratically-constrained quadratic programs and semidefinite programs \cite{Luo2010Semidefinite}, recent work in real algebraic geometry has revealed the remarkable fact that the much broader class of (\emph{rational}) \emph{polynomial optimization problems}\footnote{These are nonlinear programs in which the feasible set is a \emph{semialgebraic set} (the set of real solutions of a system of polynomial (in)equalities) and the objective is a (rational) polynomial function.} also admits a hierarchy of semidefinite relaxations that is likewise frequently exact (or can be made arbitrarily sharp) when applied to real-world problem instances  \cite{Lasserre2001Global,Parrilo2003Semidefinite,Lasserre2006Convergent,Laurent2009SOS,Bugarin2016Minimizing,Nie2006Minimizing,Nie2008Sparse,Waki2006SOS}.  Given the broad generality of this latter class of models, \syncSpace's demonstration that it is indeed possible to solve surprisingly large (but suitably structured) semidefinite programs with temporal and computational resources typical of those available in embedded systems suggests the further possibility of designing a broad class of practically-effective certifiably correct algorithms for robust machine perception based upon structured semidefinite programming relaxations.  It is our hope that this report will encourage further investigation of this exciting possibility within the machine perception community.

%% file: The_Isotropic_Langevin_distribution.tex
\section{The isotropic Langevin distribution}
\label{Isotropic_Langevin_distribution_appendix}

In this appendix we provide a brief overview of the isotropic Langevin distribution on $\SO(d)$, with a particular emphasis on the important special cases $d = 2,3$.

The \emph{isotropic Langevin distribution} on $\SO(d)$ with mode $M \in \SO(d)$ and concentration parameter $\kappa \ge 0$, denoted $\Langevin(M, \kappa)$, is the distribution determined by the following probability density function (with respect to the Haar measure on $\SO(d)$ \cite{Warner1983Manifolds}):
\begin{equation}
 \label{isotropic_Langevin_density_definition}
 p(X; M, \kappa) = \frac{1}{c_d(\kappa)} \exp\left(\kappa \tr(M\transpose X) \right),
\end{equation}
where $c_d(\kappa)$ is a normalization constant \cite{Chiuso2008Wide,Boumal2014Cramer}.  Note that the product $M\transpose X = M\inv X \in \SO(d)$ appearing in \eqref{isotropic_Langevin_density_definition} is the relative rotation sending $M$ to $X$.  

In general, given any $Z \in \SO(d)$, there exists some $U \in \Orthogonal(d)$ such that:

\begin{equation}
\label{canonical_form_for_relative_rotation_MTR}
U\transpose Z U = \begin{cases}
\begin{pmatrix}
R(\theta_1) \\
 & \ddots & \\
 & & R(\theta_k)
\end{pmatrix}, & d \bmod 2 = 0, \\
\begin{pmatrix}
R(\theta_1) \\
& \ddots & \\
& & R(\theta_k) \\
& & & 1
\end{pmatrix}, & d \bmod 2 = 1,
\end{cases}
\end{equation}
where $k \triangleq \lfloor d/2 \rfloor$, $\theta_i \in [-\pi, \pi)$ for all $i \in [k]$, and 
\begin{equation}
\label{2x2_rotation_matrix_definition}
 R(\theta) \triangleq 
 \begin{pmatrix}
\cos(\theta) & \sin(\theta) \\
-\sin(\theta) & \cos(\theta)
 \end{pmatrix} \in \SO(2);
\end{equation}
this \emph{canonical decomposition} corresponds to the fact that every rotation $Z \in \SO(d)$ acts on $\R^d$ as a set of elementary rotations \eqref{2x2_rotation_matrix_definition} of mutually-orthogonal $2$-dimensional subspaces. Since $\tr(\cdot)$ is a class function, it follows from \eqref{canonical_form_for_relative_rotation_MTR} and \eqref{2x2_rotation_matrix_definition} that the trace appearing in the isotropic Langevin density \eqref{isotropic_Langevin_density_definition} can be equivalently expressed as:
\begin{equation}
\label{trace_of_relative_rotations}
 \tr(M\transpose X) =  d \bmod 2 + 2\sum_{i = 1}^{k} \cos(\theta_i),
\end{equation}
where the $\theta_i$'s are the rotation angles for each of $M\transpose X$'s elementary rotations.  Note, however, that while the right-hand side of \eqref{trace_of_relative_rotations} depends upon the \emph{magnitudes} of these elementary rotations, it does \emph{not} depend upon the \emph{orientation} of their corresponding subspaces; this is the sense in which the Langevin density \eqref{isotropic_Langevin_density_definition} is ``isotropic''.

For the special cases $d = 2,3$, the normalization constant $c_d(\kappa)$ appearing in \eqref{isotropic_Langevin_density_definition} admits the following simple closed forms (cf.\ equations (4.6) and (4.7) of \cite{Boumal2014Cramer}):
\begin{subequations}
 \label{isotropic_Langevin_density_normalization_constants}
 \begin{align}
  c_2(\kappa) &= I_0(2\kappa), \label{isotropic_Langevin_density_normalization_constant_in_2d} \\
   c_3(\kappa) &= \exp(\kappa) \left(I_0(2\kappa) - I_1(2\kappa) \right), \label{isotropic_Langevin_density_normalization_constant_in_3d}
   \end{align}
\end{subequations}
where $I_n(z)$ denotes the \emph{modified Bessel function of the first kind} \cite[eq.\ (10.32.3)]{NIST2016DLMF}:
\begin{equation}
 \label{modified_Bessel_function_of_the_first_kind_definition}
 I_{n}(z) \triangleq \frac{1}{\pi} \int_0^{\pi} e^{z \cos(\theta)} \cos(n\theta) \: d\theta, \quad n \in \N.
\end{equation}
Furthermore, in these dimensions every rotation acts on a single $2$-dimensional subspace, and hence is described by a single rotation angle.  Letting $\theta \triangleq \angle(M\transpose X)$, it follows from \eqref{isotropic_Langevin_density_definition} and \eqref{trace_of_relative_rotations} that the distribution over $\theta$ induced by $\Langevin(M, \kappa)$ has a density satisfying:
\begin{equation}
\label{isotropic_Langevin_induced_density_over_rotation_angles_in_2_and_3d}
 p(\theta ; \kappa) \propto \exp(2\kappa \cos(\theta)), \quad \theta \in [-\pi, \pi).
\end{equation}

\begin{figure}
\center
\includegraphics[scale = .6]{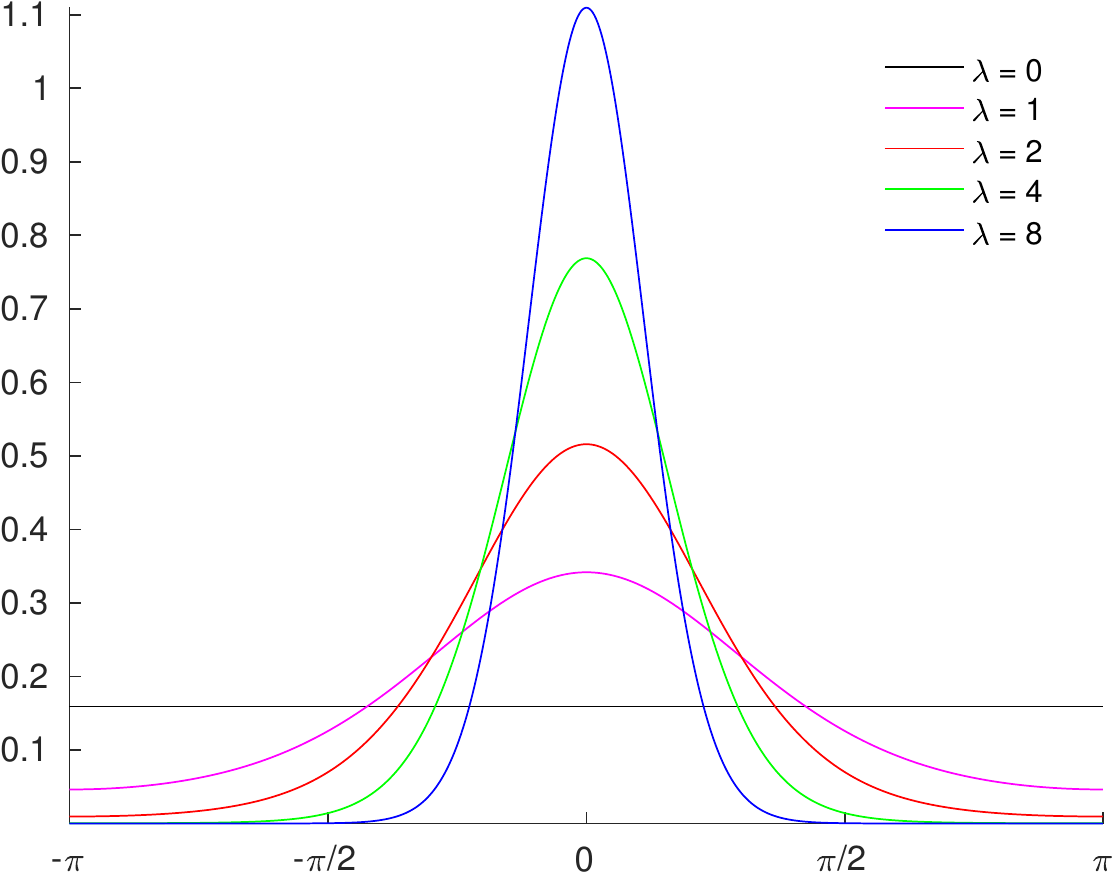}
\caption{The von Mises distribution.  This plot shows the densities \eqref{von_Mises_density} of several von Mises distributions on the circle (here identified with $[-\pi, \pi)$)  with common mode $\mu = 0$ and varying concentrations $\lambda$.  For $\lambda = 0$ the von Mises distribution reduces to the uniform distribution, and its density asymptotically converges to the Gaussian $\Gaussian(0, \lambda^{-1})$ as $\lambda \to \infty$.}
\label{von_Mises_densities_fig}
\end{figure}
\vspace{-.5\baselineskip}

Now recall that the \emph{von Mises distribution} on the circle (here identified with $[-\pi, \pi)$) with mode $\mu \in [-\pi, \pi)$ and concentration parameter $\lambda \ge 0$, denoted $\vonMises(\mu, \lambda)$, is the distribution determined by the following probability density function:
\begin{equation}
 \label{von_Mises_density}
 p(\theta ; \mu, \lambda) = \frac{\exp(\lambda \cos(\theta - \mu))}{2\pi I_0(\lambda)}.
\end{equation}
This distribution plays a role in circular statistics analogous to that of the Gaussian distribution on Euclidean space \cite{Fisher1953Dispersion}. For $\lambda = 0$ it reduces to the uniform distribution on $S^1$, and becomes increasingly concentrated at the mode $\mu$ as $\lambda \to \infty$ (cf.\ Fig.\ \ref{von_Mises_densities_fig}).  In fact, considering the asymptotic expansion \cite[10.40.1]{NIST2016DLMF}:
\begin{equation}
 I_0(\lambda) \sim \frac{\exp(\lambda)}{\sqrt{2\pi \lambda}}, \quad \lambda \to \infty,
\end{equation}
and the second-order Taylor series expansion of the cosine function:
\begin{equation}
\label{second_order_Taylor_expansion_for_cosine_of_von_Mises_distribution}
 \cos(\theta - \mu) \sim 1 - \frac{1}{2}(\theta - \mu)^2, \quad \theta \to \mu,
\end{equation}
it follows from \eqref{von_Mises_density}--\eqref{second_order_Taylor_expansion_for_cosine_of_von_Mises_distribution} that 
\begin{equation}
\label{von_Mises_Gaussian_approximation}
 p(\theta ; \mu, \lambda) \sim \frac{1} {\sqrt{2\pi \lambda\inv}} \exp \left(-\frac{(\theta - \mu)^2 }{2 \lambda\inv}\right), \quad \lambda \to \infty,
\end{equation}
which we recognize as the density for the Gaussian distribution $\Gaussian(\mu, \lambda\inv)$.

These observations lead to a particularly convenient generative description of the isotropic Langevin distribution in dimensions $2$ and $3$: namely, a realization of $X \sim \Langevin(M, \kappa)$ is obtained by perturbing the mode $M$ by a rotation through an angle $\theta \sim \vonMises(0, 2\kappa)$ about a uniformly-distributed axis (Algorithm \ref{isotropic_Langevin_sampling_algorithm}).  Furthermore, it  follows from \eqref{von_Mises_density} that the standard deviation of the angle $\theta$ of the relative rotation between $M$ and $X$ is given by:

\begin{algorithm}[t]
\caption{A sampler for the isotropic Langevin distribution on $\SO(d)$ in dimensions $2$ and $3$}
\label{isotropic_Langevin_sampling_algorithm}
\begin{algorithmic}[1]
\Input Mode $M \in \SO(d)$ with $d \in \lbrace 2, 3 \rbrace$, concentration parameter $\kappa \ge 0$.
\Output A realization of $X \sim \Langevin(M, \kappa)$.
\Function{SampleIsotropicLangevin}{$M, \kappa$}
\State Sample a rotation angle $\theta \sim \vonMises(0, 2\kappa)$. \Comment{See e.g.\ \cite{Best1979Efficient}}
\If{$d = 2$}
\State Set perturbation matrix $P \leftarrow R(\theta)$ as in \eqref{2x2_rotation_matrix_definition}.
\Else \Comment{$d = 3$}
\State Sample an axis of rotation $\hat{v} \sim \Uniform(S^2)$.
\State Set perturbation matrix $P \leftarrow \exp(\theta [\hat{v}]_{\times})$.
\EndIf 
\State \Return $M P$
\EndFunction
\end{algorithmic}
\end{algorithm}

\begin{equation}
\label{standard_deviation_of_rotation_angle}
 \SD[\theta] = \sqrt{\int_{-\pi}^{\pi} \theta^2 \frac{\exp(2\kappa \cos(\theta))}{2\pi I_0(2\kappa)} \: d\theta },
\end{equation}
which provides a convenient and intuitive measure of the dispersion of $\Langevin(M,\kappa)$.  The right-hand side of \eqref{standard_deviation_of_rotation_angle} can be efficiently evaluated to high precision via numerical quadrature for values of $\kappa$ less than $150$ (corresponding to an angular standard deviation of $3.31$ degrees).  For $\kappa > 150$, one can alternatively use the estimate of the angular standard deviation coming from the asyptotic Gaussian approximation \eqref{von_Mises_Gaussian_approximation} for the von Mises distribution $\vonMises(0, 2\kappa)$:
\begin{equation}
\label{Gaussian_approximation_for_angular_standard_deviation}
 \SD[\theta] \sim \frac{1}{\sqrt{2\kappa}}, \quad \kappa \to \infty.
\end{equation}
This approximation is accurate to within $1\%$ for $\kappa > 12.87$ (corresponding to an angular standard deviation of $11.41$ degrees) and to within $.1\%$ for $\kappa > 125.3$ (corresponding to an angular standard deviation of $3.62$ degrees).

%% file: Reformulating_the_MLE.tex
\section{Reformulating the estimation problem}
\label{Reformulating_the_MLE_section}

\subsection{Deriving Problem \ref{SE3_synchronization_MLE_QP_form_problem} from Problem \ref{SE3_synchronization_MLE_NLS_problem}}
\label{deriving_the_quadratic_form_of_the_MLE}

In this section we show how to derive Problem \ref{SE3_synchronization_MLE_QP_form_problem} from Problem \ref{SE3_synchronization_MLE_NLS_problem}.  Using the fact that $\vect(v) = v$ for any $v \in \R^{d \times 1}$ and the fact that $\vect(XY) = (Y^T \otimes I_k) \vect(X)$ for all $X \in \R^{m \times k}$ and $Y \in \R^{k \times l}$ \cite[Lemma 4.3.1]{Horn1991Topics}, we can write each summand of \eqref{SE3_synchronization_MLE_NLS_optimization} in a vectorized form as:
\begin{equation}
\label{vectorization_of_translational_cost_function}
\begin{split}
\tau_{ij} \left \lVert \tran_j - \tran_i - \rot_i \ntran_{ij} \right \Vert_2^2 &= \tau_{ij}\left \lVert \tran_j - \tran_i - \left(\ntran_{ij}\transpose \otimes I_d \right) \vect(\rot_i)\right \rVert_2^2\\
&= \left \lVert (\sqrt{\tau_{ij}} I_d) \tran_j - (\sqrt{\tau_{ij}} I_d) \tran_i - \sqrt{\tau_{ij}}\left(\ntran_{ij}\transpose \otimes I_d \right) \vect(\rot_i) \right \rVert_2^2
\end{split}
\end{equation}
and 
\begin{equation}
\label{vectorization_of_rotational_cost_function}
 \begin{split}
 \kappa_{ij}  \left \lVert \rot_j - \rot_i \nrot_{ij}  \right \rVert_F^2 &= \kappa_{ij} \left \lVert \vect(\rot_j) - \vect(\rot_i \nrot_{ij}) \right \rVert_2^2 \\
&= \left \lVert \sqrt{\kappa_{ij}}(I_d \otimes I_d) \vect(\rot_j) - \sqrt{\kappa_{ij}}\left(\nrot_{ij}\transpose \otimes I_d\right) \vect(\rot_i) \right \rVert_2^2.
 \end{split}
\end{equation}
Letting $\tran\in \R^{dn}$ and $\rot \in \R^{d \times dn}$ denote the concatenations of the $t_i$ and $R_i$ as defined in \eqref{block_matrix_state_definitions}, \eqref{vectorization_of_translational_cost_function} and \eqref{vectorization_of_rotational_cost_function} imply that \eqref{SE3_synchronization_MLE_NLS_optimization} can be rewritten in a vectorized form as:
\begin{equation}
\label{MLE_problem_in_vectorized_squared_L2_norm_format}
 \MLEval = \min_{\stackrel{\tran \in \R^{dn}}{\rot \in \SO(d)^n}} \left \lVert
B
 \begin{pmatrix}
\tran \\
\vect(R)
 \end{pmatrix}
\right \rVert_2^2
\end{equation}
where the coefficient matrix $B \in \R^{(d+d^2)m \times (d+d^2)n}$ (with $m = \lvert \dEdges \rvert$) has the block decomposition:
\begin{equation}
 B \triangleq 
 \begin{pmatrix}
B_1 & B_2 \\
0 & B_3 
 \end{pmatrix}
\end{equation}
and $B_1 \in \R^{dm \times dn}$, $B_2 \in \R^{dm \times d^2n}$, and $B_3 \in \R^{d^2m \times d^2n}$ are block-structured matrices whose block rows and columns are indexed by the elements of $\dEdges$ and $\Nodes$, respectively, and whose $(e,k)$-block elements are given by:
\begin{subequations}
\label{B_matrices_block_formulae}
\begin{equation}
\label{B1_block_formula}
 (B_1)_{ek} = 
 \begin{cases}
-\sqrt{\tau_{kj}}I_d, & e = (k,j) \in \dEdges, \\
\sqrt{\tau_{ik}} I_d, & e = (i,k) \in \dEdges,\\
0_{d \times d}, & \textnormal{otherwise},
 \end{cases}
 \end{equation}
 \begin{equation}
 \label{B2_block_formula}
 (B_2)_{ek} = 
 \begin{cases}
-\sqrt{\tau_{kj}} \left(\ntran_{kj}\transpose \otimes I_d\right), & e = (k,j) \in \dEdges, \\
0_{d \times d^2}, & \textnormal{otherwise},
 \end{cases}
 \end{equation}
 \begin{equation}
 \label{B3_block_formula}
(B_3)_{ek} = 
\begin{cases}
- \sqrt{\kappa_{kj}} \left(\nrot_{kj}\transpose \otimes I_d\right), & e = (k,j) \in \dEdges, \\
\sqrt{\kappa_{ik}} (I_d \otimes I_d), & e = (i,k) \in \dEdges, \\
0_{d \times d}, & \textnormal{otherwise}.
\end{cases}
\end{equation}
\end{subequations}
We can further expand the squared $\ell_2$-norm objective in \eqref{MLE_problem_in_vectorized_squared_L2_norm_format} to obtain:
\begin{equation}
 \MLEval = \min_{\stackrel{\tran \in \R^{dn}}{\rot \in \SO(d)^n}}
 \begin{pmatrix}
\tran \\
\vect(R)
 \end{pmatrix}\transpose B\transpose B 
 \begin{pmatrix}
\tran \\
\vect(R)
 \end{pmatrix}
\end{equation}
with 
\begin{equation}
 \label{block_decomposition_of_BTB}
B\transpose B= 
\begin{pmatrix}
B_1\transpose B_1 & B_1\transpose B_2 \\
B_2\transpose B_1 & B_2 \transpose B_2 + B_3\transpose B_3 
\end{pmatrix}.
\end{equation}
We now derive explicit forms for each of the four distinct products on the right-hand side of \eqref{block_decomposition_of_BTB}.

\textbf{Case $B_1\transpose B_1$:}  Consider the $(i,j)$-block of the product $B_1 \transpose B_1$:
\begin{equation}
\label{expansion_of_B1tB1}
(B_1\transpose B_1)_{ij} = \sum_{e \in \dEdges} (B_1)\transpose_{ei} (B_1)_{ej}.
\end{equation}
From \eqref{B1_block_formula} we have that $(B_1)_{ek} \ne 0$ only if $e = (i,k)$ or $e = (k,j)$.  Thus for $i \ne j$, $(B_1)_{ei}\transpose (B_1)_{ej} \ne 0$ only if $e = (i,j)$ or $e = (j,i)$, in which case \eqref{expansion_of_B1tB1} becomes:
\begin{equation}
\label{B1tB1_off_diagonal_formula}
(B_1\transpose B_1)_{ij} = \left(\pm \sqrt{\tau_e} I_d \right)\transpose \left( \mp \sqrt{\tau_e} I_d \right) = -\tau_e I_d.
\end{equation}
On the other hand, if $i = j$, then \eqref{expansion_of_B1tB1} reduces to
\begin{equation}
\label{B1tB1_diagonal_formula}
(B_1\transpose B_1)_{ii} = \sum_{e \in \incEdges(i) } \left( \pm \sqrt{\tau_e} I_d \right)\transpose\left( \pm \sqrt{\tau_e} I_d \right) = \sum_{e \in \incEdges(i) } \tau_e I_d.
\end{equation}
Comparing \eqref{B1tB1_off_diagonal_formula} and \eqref{B1tB1_diagonal_formula} with \eqref{Laplacian_of_translational_weight_graph} shows that $B_1\transpose B_1 = \LapTranW \otimes I_d$.

\textbf{Case $B_1\transpose B_2$:}  Consider the $(i,j)$-block of the product $B_1 \transpose B_2$:
\begin{equation}
\label{expansion_of_B1tB2}
 (B_1 \transpose B_2)_{ij} =  \sum_{e \in \dEdges} (B_1)\transpose_{ei} (B_2)_{ej}.
\end{equation}
From \eqref{B2_block_formula} we have that $(B_2)_{ej} \ne 0$  only if $e = (j,k)$, and \eqref{B1_block_formula} shows that $(B_1)_{ei}\transpose \ne 0$ only if $e = (i,k)$ or $e = (k,i)$.  Thus for $i \ne j$, $(B_1)\transpose_{ei} (B_2)_{ej} \ne 0$ only if $e = (j,i)$, in which case \eqref{expansion_of_B1tB2} becomes:
\begin{equation}
\label{B1t_B2_off_diagonal_elements}
 (B_1 \transpose B_2)_{ij} = \left( \sqrt{\tau_{ji}} I_d \right)\transpose \left( - \sqrt{\tau_{ji}} \left(\ntran_{ji}\transpose \otimes I_d\right) \right) = -\tau_{ji} \left(\ntran_{ji}\transpose \otimes I_d\right).
\end{equation}
On the other hand, for $i = j$, \eqref{expansion_of_B1tB2} simplifies to:
\begin{equation}
\label{B_1t_B2_diagonal_elements}
(B_1\transpose B_2)_{jj} = \sum_{(j,k) \in \dEdges} \left( -\sqrt{\tau_{jk}} I_d \right)\transpose \left( - \sqrt{\tau_{jk}} \left(\ntran_{jk}\transpose \otimes I_d \right) \right) = \sum_{(j,k) \in \dEdges}  \tau_{jk} \left(\ntran_{jk}\transpose \otimes I_d \right).
\end{equation}
Comparing \eqref{B1t_B2_off_diagonal_elements} and \eqref{B_1t_B2_diagonal_elements} with \eqref{cross_term_matrix_definition} shows that $B_1\transpose B_2 = \nCrossTerms \otimes I_d$.
 
\textbf{Case $B_2\transpose B_2$:}  Consider the $(i,j)$-block of the product $B_2\transpose B_2$: 
\begin{equation}
 (B_2\transpose B_2)_{ij} = \sum_{e \in \dEdges} (B_2)_{ei}\transpose (B_2)_{ej}.
\end{equation}
Since $(B_2)_{ei} = 0$ unless $e = (i,k)$ by \eqref{B2_block_formula}, $(B_2)\transpose_{ei}(B_2)_{ej} \ne 0$ only if $i = j$.  This implies that $(B_2\transpose B_2)_{ij} = 0$ for all $i \ne j$, and
\begin{equation}
\label{B2transpose_B2_block_diagonal}
\begin{split}
(B_2\transpose B_2)_{jj} &= \sum_{(j,k) \in \dEdges} \left(-\sqrt{\tau_{jk}} \left(\ntran_{jk}^T \otimes I_d \right) \right)\transpose \left(-\sqrt{\tau_{jk}} \left(\ntran_{jk}\transpose \otimes I_d\right) \right)  \\
&= \sum_{(j,k) \in \dEdges} \tau_{jk} \left(\ntran_{jk} \ntran_{jk}\transpose \otimes I_d\right).
\end{split}
\end{equation}
Comparing \eqref{B2transpose_B2_block_diagonal} with \eqref{Sigma_matrix_definition} then shows that $B_2\transpose B_2 = \nOuterProducts \otimes I_d$.

\textbf{Case $B_3\transpose B_3$:}   Finally, consider the $(i,j)$-block of the product $B_3 \transpose B_3$:
\begin{equation}
\label{B3tB3_expansion}
(B_3\transpose B_3)_{ij} = \sum_{e \in \dEdges} (B_3)_{ei}\transpose(B_3)_{ej}.
\end{equation}
Using \eqref{B3_block_formula}, for $i \ne j$ we find that
\begin{equation}
\begin{split}
 (B_3)_{ei}\transpose (B_3)_{ej} &= 
  \begin{cases}
\left( - \sqrt{\kappa_{ij}}\left(\nrot_{ij}^T \otimes I_d\right) \right)\transpose \left( \sqrt{\kappa_{ij}} (I_d \otimes I_d) \right), & e = (i,j) \in \dEdges, \\
 \left( \sqrt{\kappa_{ji}} I_d \otimes I_d \right)\transpose \left( -\sqrt{\kappa_{ji}} \left(\nrot_{ji}\transpose \otimes I_d \right) \right), & e = (j,i) \in \dEdges, \\
 0_{d^2\times d^2}, & \textnormal{otherwise}.
 \end{cases}
   \end{split}
\end{equation}
Consequently, for $i \ne j$ \eqref{B3tB3_expansion} reduces to
\begin{equation}
\label{B3_transpose_B3_off_diagonal}
 (B_3\transpose B_3)_{ij} = 
 \begin{cases}
-\kappa_{ij} (\nrot_{ij} \otimes I_d), & (i,j) \in \dEdges, \\
-\kappa_{ji} (\nrot_{ji}\transpose \otimes I_d), & (j,i) \in \dEdges, \\
0_{d^2 \times d^2}, & \textnormal{otherwise}.
 \end{cases}
\end{equation}
Similarly, for $i = j$ \eqref{B3_block_formula} shows that 
\begin{equation}
\begin{split}
 (B_3)_{ei}\transpose (B_3)_{ei} &= 
 \begin{cases}
 \left(-\sqrt{\kappa_{ik}} \left(\nrot_{ik} \otimes I_d\right) \right)  \left(-\sqrt{\kappa_{ik}} \left(\nrot_{ik}\transpose \otimes I_d\right) \right), & (i,k) \in \dEdges, \\
 \left(\sqrt{\kappa_{ki}} (I_d \otimes I_d) \right)  \left(\sqrt{\kappa_{ki}} (I_{d} \otimes I_d) \right), &(k,i) \in \dEdges, \\
 0_{d^2 \times d^2}, & \textnormal{otherwise}, 
 \end{cases} \\
 &= 
 \begin{cases}
 \kappa_{e} (I_d \otimes I_d), & e \in \incEdges(i) \subseteq \dEdges, \\
 0_{d^2 \times d^2}, & \textnormal{otherwise},
 \end{cases}
 \end{split}
 \end{equation}
and therefore
\begin{equation}
\label{B3_inner_products_on_diagonal}
 (B_3\transpose B_3)_{ii} = \sum_{ e \in \incEdges(i) } \kappa_e (I_d \otimes I_d), 
 \end{equation}
 Comparing \eqref{B3_transpose_B3_off_diagonal} and \eqref{B3_inner_products_on_diagonal} with \eqref{connection_Laplacian_definition} shows that $B_3\transpose B_3 = \MeasRotConLap \otimes I_d$.

Together, these results establish that $B\transpose B = \QPFormMatrix \otimes I_d$, where $\QPFormMatrix$ is the matrix defined in \eqref{M_matrix_definition}, and consequently that Problem \ref{SE3_synchronization_MLE_NLS_problem} is equivalent to Problem \ref{SE3_synchronization_MLE_QP_form_problem}.

\subsection{Deriving Problem \ref{Rotational_maximum_likelihood_estimation_for_SE3_synchronization} from Problem \ref{SE3_synchronization_MLE_QP_form_problem}}
\label{eliminating_translational_states_subsection}
In this subsection, we show how to analytically eliminate the translations $t$ appearing in Problem \ref{SE3_synchronization_MLE_QP_form_problem} to obtain the simplified form of Problem \ref{Rotational_maximum_likelihood_estimation_for_SE3_synchronization}. We make use of the following lemma (cf.\ \cite[Appendix A.5.5]{Boyd2004Convex} and \cite[Proposition 4.2]{Gallier2010Schur}):

\begin{lem}
\label{generalized_schur_complement_quadratic_minimization_lemma}
Given $A \in \Sym(p)$ and $b \in \R^p$, the function
\begin{equation}
  f(x) = x\transpose A x + 2b\transpose x
 \end{equation}
attains a minimum if and only if $A \succeq 0$ and $(I_d - AA^{\pinv})b = 0$, in which case
\begin{equation}
 \min_{x \in \R^p} f(x) =  - b\transpose A^{\pinv} b
\end{equation}
and
\begin{equation}
\label{optimal_value_of_x_for_quadratic_optimization_lemma}
 \argmin_{x \in \R^p} f(x) = \left \lbrace -A^{\pinv} b + U 
 \begin{pmatrix}
 0 \\
 z
 \end{pmatrix}
 \mid z \in \R^{p - r} \right \rbrace,
\end{equation}
where $A = U \Lambda U\transpose$ is an eigendecomposition of $A$ and $r = \rank(A)$.
\end{lem}

Now $\LapTranW \otimes I_d \succeq 0$ since $\LapTranW$ is a (necessarily positive semidefinite) graph Laplacian, and $I_{dn} - (\LapTranW \otimes I_d)(\LapTranW \otimes  I_d )\pinv$ is the orthogonal projection operator onto $\ker\left((\LapTranW \otimes I_d)\transpose\right) = \ker(\LapTranW \otimes I_d)$ \cite[eq.\ (5.13.12)]{Meyer2000Matrix}.  Using the relation $\vect(AYC) = (C\transpose \otimes A) \vect(Y)$ \cite[Lemma 4.3.1]{Horn1991Topics}, we find that $y \in \ker(\LapTranW \otimes I_d)$ if and only if $y = \vect(Y)$ for some $Y \in \R^{d \times n}$ satisfying $Y \LapTranW = 0$, or equivalently $\LapTranW Y\transpose = 0$.  Since $G$ is connected, then $\ker(\LapTranW) = \Span\lbrace \ones_n \rbrace$, and therefore we must have $Y\transpose = \ones_n c\transpose \in \R^{n \times d}$ for some $c \in \R^d$. Altogether, this establishes that:
\begin{equation}
\label{kernel_of_LWtau_otimes_I3}
 \ker( \LapTranW \otimes I_d ) = \left \lbrace \vect \left(
 c \ones_n\transpose \right) \mid c \in \R^d
 \right \rbrace.
\end{equation}
Now let $b = (\nCrossTerms \otimes I_d) \vect(\rot)$; we claim that $b \perp y$ for all $y = \vect(Y) \in \ker(\LapTranW \otimes I_d)$.  To see this, observe that $b = \vect(R \nCrossTerms\transpose)$ and therefore $\langle b, y \rangle_2 = \langle R \nCrossTerms\transpose, Y \rangle_F  = \tr(Y \nCrossTerms \rot \transpose)$.  However, $\ones_n\transpose \nCrossTerms = 0$ since the sum down each column of $\nCrossTerms$ is identically $0$ by \eqref{cross_term_matrix_definition}, and therefore $Y \nCrossTerms = 0$ by \eqref{kernel_of_LWtau_otimes_I3}.  This establishes that $b \perp \ker(\LapTranW \otimes I_d)$ for \emph{any} value of $\rot$.

Consequently, if we fix  $\rot \in \SO(d)^n$ and consider performing the optimization in \eqref{block_partitioned_quadratic_objective_in_MLE_problem} over the decision variables $\tran$ \emph{only}, we can apply Lemma \ref{generalized_schur_complement_quadratic_minimization_lemma} to compute the optimal value of the objective and a minimizing value $\topt$ of $t$ as functions of $\rot$.  This in turn enables us to analytically eliminate $\tran$ from \eqref{block_partitioned_quadratic_objective_in_MLE_problem}, producing the equivalent optimization problem:
\begin{equation}
\label{vectorized_form_of_reduced_rotation_only_MLE}
 \MLEval = \min_{\rot \in \SO(d)} \vect(\rot)\transpose \left(\left(\MeasRotConLap + \nOuterProducts - \nCrossTerms\transpose \LapTranW\pinv \nCrossTerms\right) \otimes I_d \right) \vect(\rot),
\end{equation}
with a corresponding optimal translational estimate $\topt$ given by:
\begin{equation}
\label{optimal_value_of_t_in_terms_of_R_vector_form}
 \topt = - \left(\LapTranW \pinv \nCrossTerms \otimes I_d \right) \vect(\Ropt).
\end{equation}
Rewriting \eqref{vectorized_form_of_reduced_rotation_only_MLE} and \eqref{optimal_value_of_t_in_terms_of_R_vector_form} in a more compact matricized form gives \eqref{Rotational_maximum_likelihood_estimation_for_SE3_synchronization_optimization_problem} and \eqref{minimizing_value_of_t_from_minimizing_value_of_R}, respectively.

\subsection{\texorpdfstring{Simplifying the translational data matrix $\nQtran$}{Simplifying the translational data matrix}}
\label{An_alternative_form_for_the_translational_data_matrix_subsection}

In this subsection we derive the simplified form of the translational data matrix $\nQtran$ given in \eqref{nQtran_alternative_form} from the one originally presented in \eqref{initial_Q_quadratic_form_definition}. To begin, recall from Appendix \ref{deriving_the_quadratic_form_of_the_MLE} that 
\begin{equation}
\label{expression_for_nQtran_in_terms_of_B_matrices}
 \begin{split}
\nQtran \otimes I_d &= \left(\nOuterProducts - \nCrossTerms \transpose \LapTranW\pinv \nCrossTerms\right) \otimes I_d \\
&= B_2\transpose B_2 - B_2\transpose B_1 \left(B_1\transpose B_1 \right)\pinv B_1\transpose B_2 \\
&= B_2\transpose\left(I_{dm} - B_1 \left(B_1\transpose B_1 \right)\pinv B_1\transpose \right) B_2,
 \end{split}
\end{equation}
where $B_1$ and $B_2$ are the matrices defined in \eqref{B1_block_formula} and \eqref{B2_block_formula}, respectively.  Using $\tranPrecisions$ and $\nT$ as defined in \eqref{translational_precision_matrix} and \eqref{nT_matrix_definition}, respectively, we may write $B_1$ and $B_2$ alternatively as:
\begin{equation}
\label{B1_and_B2_in_terms_of_incidence_matrice}
 B_1 = \tranPrecisions^{\frac{1}{2}} \incMat\transpose  \otimes I_d, \quad \quad B_2 = \tranPrecisions^{\frac{1}{2}}\nT \otimes I_d,
\end{equation}
where $\incMat \triangleq \incMat(\directed{G})$ is the incidence matrix of $\directed{G}$.  Substituting \eqref{B1_and_B2_in_terms_of_incidence_matrice} into \eqref{expression_for_nQtran_in_terms_of_B_matrices}, we derive:
\begin{equation}
\label{setup_for_application_of_orthogonal_projection_operator_for_reduced_incidence_matrix}
\begin{split}
\nQtran \otimes I_d &= B_2\transpose \left(I_{dm } - B_1 (B_1\transpose B_1)\pinv B_1\transpose \right) B_2 \\
&= B_2\transpose \left( I_{dm }  - \left(\tranPrecisions^{\frac{1}{2}}\incMat\transpose \otimes I_d \right) \left( \left(\tranPrecisions^{\frac{1}{2}}\incMat\transpose \otimes I_d \right)\transpose \left(\tranPrecisions^{\frac{1}{2}}\incMat\transpose \otimes I_d \right) \right)\pinv \left(\tranPrecisions^{\frac{1}{2}}\incMat\transpose \otimes I_d\right)\transpose \right) B_2 \\
&= B_2\transpose \left( I_{dm}  - \left(\tranPrecisions^{\frac{1}{2}}\incMat\transpose \otimes I_d\right) \left( \left(\incMat \tranPrecisions \incMat\transpose\right)\pinv \otimes I_d \right) \left(\tranPrecisions^{\frac{1}{2}}\incMat\transpose \otimes I_d\right)\transpose \right) B_2 \\
 &= B_2\transpose \left( I_{dm}  - \left( \tranPrecisions^{\frac{1}{2}}\incMat\transpose \left(\incMat\tranPrecisions \incMat\transpose \right)\pinv \incMat \tranPrecisions^{\frac{1}{2}} \right) \otimes I_d   \right) B_2 \\
 &= B_2\transpose \left[ \left( I_{m }  - \tranPrecisions^{\frac{1}{2}}\incMat\transpose \left(\incMat \tranPrecisions \incMat\transpose \right)\pinv \incMat \tranPrecisions^{\frac{1}{2}} \right)  \otimes I_d  \right] B_2 \\
 &= \left( \nT \transpose \tranPrecisions^{\frac{1}{2}} \left( I_{m }  - \tranPrecisions^{\frac{1}{2}}\incMat\transpose \left(\incMat\tranPrecisions \incMat\transpose \right)\pinv \incMat \tranPrecisions^{\frac{1}{2}} \right)  \tranPrecisions^{\frac{1}{2}}  \nT \right) \otimes I_d,
\end{split}
\end{equation}
or equivalently:
\begin{equation}
\label{setup_for_application_of_orthogonal_projection_operator_for_reduced_incidence_matrix_without_Kronecker_product}
 \nQtran = \nT \transpose \tranPrecisions^{\frac{1}{2}} \underbrace{\left( I_{m }  - \tranPrecisions^{\frac{1}{2}}\incMat\transpose \left(\incMat\tranPrecisions \incMat\transpose \right)\pinv \incMat \tranPrecisions^{\frac{1}{2}} \right)}_{\OrthoProjMatrix}    \tranPrecisions^{\frac{1}{2}}  \nT.
\end{equation}

Now let us develop the term labeled $\OrthoProjMatrix$ appearing in \eqref{setup_for_application_of_orthogonal_projection_operator_for_reduced_incidence_matrix_without_Kronecker_product}:
\begin{equation}
 \label{orthogonal_projection_matrix_definition}
 \begin{split}
 \OrthoProjMatrix &=  I_{m }  - \tranPrecisions^{\frac{1}{2}}\incMat\transpose \left(\incMat\tranPrecisions \incMat\transpose \right)\pinv \incMat \tranPrecisions^{\frac{1}{2}} \\
  &= I_m - \left(  \incMat \tranPrecisions^{\frac{1}{2}} \right)\transpose \left(  \left(  \incMat \tranPrecisions^{\frac{1}{2}} \right)  \left(  \incMat \tranPrecisions^{\frac{1}{2}}\right)\transpose \right)\pinv   \left(  \incMat \tranPrecisions^{\frac{1}{2}} \right) \\
  &= I_m - \left(  \incMat \tranPrecisions^{\frac{1}{2}} \right)\pinv \left(  \incMat \tranPrecisions^{\frac{1}{2}} \right)
 \end{split}
\end{equation}
where we have used the fact that $X\transpose (X X\transpose)\pinv = X\pinv$ for any matrix $X$ in passing from line 2 to line 3 above.  We may now recognize the final line of \eqref{orthogonal_projection_matrix_definition} as the matrix of the orthogonal projection operator $\OrthoProj \colon \R^m \to \ker( \incMat \tranPrecisions^{\frac{1}{2}})$ onto the kernel of the weighted incidence matrix $\incMat \tranPrecisions^{\frac{1}{2}}$ \cite[eq.\ (5.13.12)]{Meyer2000Matrix}.  Equation \eqref{nQtran_alternative_form} thus follows from \eqref{setup_for_application_of_orthogonal_projection_operator_for_reduced_incidence_matrix_without_Kronecker_product} and \eqref{orthogonal_projection_matrix_definition}.

Finally, although $\OrthoProjMatrix$ is generically dense, we now show that it admits a computationally convenient decomposition in terms of sparse matrices and their inverses.  By the Fundamental Theorem of Linear Algebra, $\ker(\incMat \tranPrecisions^{\frac{1}{2}})^{\perp} = \image(\tranPrecisions^{\frac{1}{2}} \incMat\transpose)$, and therefore every vector $v \in \R^m$ admits the orthogonal decomposition:
\begin{equation}
\label{orthogonal_decomposition_of_test_vector}
\begin{split}
v &= \OrthoProj(v) + c, \quad \quad \OrthoProj(v) \in \ker(\incMat \tranPrecisions^{\frac{1}{2}}), \; c \in \image(\tranPrecisions^{\frac{1}{2}} \incMat\transpose).
\end{split}
\end{equation}
Now $\rank(\incMat) = n-1$ since $\incMat$ is the incidence matrix of the weakly-connected directed graph $\directed{G}$; it follows that $\image(\tranPrecisions^{\frac{1}{2}} \incMat \transpose) = \image(\tranPrecisions^{\frac{1}{2}} \redIncMat\transpose)$, where $\redIncMat$ is the reduced incidence matrix of $\directed{G}$ formed by removing the final row of $\incMat$.  Furthermore, since $c$ is the complement of $\pi(v)$ in the orthogonal decomposition \eqref{orthogonal_decomposition_of_test_vector}, it is itself the orthogonal projection of $v$ onto $\image(\tranPrecisions^{\frac{1}{2}} \incMat\transpose) = \image(\tranPrecisions^{\frac{1}{2}} \redIncMat\transpose)$, and is therefore the value of the product realizing the minimum norm in:
\begin{equation}
 \label{closest_point_to_v_in_imageAT}
 \min_{w \in \R^{n-1}} \lVert v - \tranPrecisions^{\frac{1}{2}} \redIncMat\transpose w \rVert_2.
\end{equation}
Consequently, it follows from \eqref{orthogonal_decomposition_of_test_vector} and \eqref{closest_point_to_v_in_imageAT} that:
\begin{equation}
\label{Computing_projection_of_v_onto_Omega_half_At_via_LS}
\begin{split}
\pi(v) &= v - \tranPrecisions^{\frac{1}{2}} \redIncMat\transpose w^{*}, \\
w^* &= \argmin_{w \in \R^{n-1}}  \lVert v - \tranPrecisions^{\frac{1}{2}} \redIncMat\transpose w \rVert_2.
\end{split}
\end{equation}
Since $\redIncMat$ is full-rank, we can solve for the minimizer $w^*$ in \eqref{Computing_projection_of_v_onto_Omega_half_At_via_LS} via the normal equations, obtaining:
\begin{equation}
\label{closed_for_expression_for_optimal_w}
w^* = \left(\redIncMat \tranPrecisions \redIncMat\transpose \right)\inv \redIncMat \tranPrecisions^{\frac{1}{2}}v = L\tinv L\inv \redIncMat \tranPrecisions^{\frac{1}{2}} v,
\end{equation}
where $\redIncMat \tranPrecisions^{\frac{1}{2}} = LQ_1$ is a thin LQ decomposition of the weighted reduced incidence matrix $\redIncMat \tranPrecisions^{\frac{1}{2}}$.  Substituting \eqref{closed_for_expression_for_optimal_w} into \eqref{Computing_projection_of_v_onto_Omega_half_At_via_LS}, we obtain:
\begin{equation}
\label{derivation_of_explicit_sparse_matrix_form_of_orthogonal_projection_operator}
\begin{split}
 \pi(v) &= v - \tranPrecisions^{\frac{1}{2}} \redIncMat\transpose L\tinv L\inv \redIncMat \tranPrecisions^{\frac{1}{2}}v = \left(I_m - \tranPrecisions^{\frac{1}{2}} \redIncMat\transpose L\tinv L\inv \redIncMat \tranPrecisions^{\frac{1}{2}} \right) v.
 \end{split}
\end{equation}
Since $v$ was arbitrary, we conclude that the matrix in parentheses on the right-hand side of \eqref{derivation_of_explicit_sparse_matrix_form_of_orthogonal_projection_operator} is $\OrthoProjMatrix$, which gives \eqref{computing_orthogonal_projection_operator_eq}.

\subsection{Deriving Problem \ref{dual_semidefinite_relaxation_for_SE3_synchronization_problem} from Problem \ref{primal_semidefinite_relaxation_for_SE3_synchronization_problem}}
\label{Deriving_the_dual_semidefinite_relaxation_from_the_primal_subsection}
This derivation is a straightforward application of the  duality theory for semidefinite programs \cite{Vandenberghe1996Semidefinite}.  Letting $\lambda_{iuv} \in \R$ denote the $(u,v)$-element of the $i$th diagonal block of $\Lambda$ for $i = 1, \dotsc, n$ and $1 \le u \le v \le d$, we can rewrite Problem \ref{primal_semidefinite_relaxation_for_SE3_synchronization_problem} in the primal standard form of \cite[eq.\ (1)]{Vandenberghe1996Semidefinite} as:
 \begin{equation}
\label{primal_semidefinite_relaxation_in_standard_form}
\begin{split}
-\SDPval = \min_{\lambda \in \R^{\frac{d(d+1)}{2}n}} &c\transpose \lambda \\
\st \; &F(\lambda) \succeq 0,
\end{split}
 \end{equation}
with the problem data in \eqref{primal_semidefinite_relaxation_in_standard_form} given explicitly by:
\begin{equation}
\label{elements_of_primal_semidefinite_relaxation_in_standard_form}
 \begin{split}
F(\lambda) &\triangleq F_0 + \sum_{i = 1}^n \sum_{1 \le u \le v \le d} \lambda_{iuv}F_{iuv}, \\
F_0 &= Q, \\
F_{iuv} &= 
\begin{cases}
- \Diag(e_i) \otimes E_{uu}, & u = v, \\
- \Diag(e_i) \otimes (E_{uv} + E_{vu}), & u \ne v, 
\end{cases} \\
c_{iuv} &= 
\begin{cases}
-1, & u = v, \\
0, & u \ne v.
\end{cases}
 \end{split}
\end{equation}
The standard dual problem for \eqref{primal_semidefinite_relaxation_in_standard_form} is then (cf.\ equations (1) and (27) of \cite{Vandenberghe1996Semidefinite}):
\begin{equation}
\label{dual_semidefinite_relaxation_in_standard_form}
\begin{split}
-d_{\textnormal{SDP}}^* = \max_{Z \in \Sym(dn)} &-\tr(F_0 Z) \\
\st \; & \tr\left(F_{iuv} Z\right) = c_{iuv} \quad \forall i \in [n],\; 1 \le u \le v \le d, \\
& Z \succeq 0.
\end{split}
\end{equation}
Comparing \eqref{dual_semidefinite_relaxation_in_standard_form} with \eqref{elements_of_primal_semidefinite_relaxation_in_standard_form} reveals that the equality constraints are satisfied precisely when $Z$'s $(d \times d)$-block-diagonal is comprised of identity matrices, which gives the form of Problem \ref{dual_semidefinite_relaxation_for_SE3_synchronization_problem}.  Furthermore, since $\nQ - \Lambda \succ 0$ for any $\Lambda = -sI_{dn}$ with $s > \lVert \nQ \rVert_2$ is strictly feasible for Problem \ref{primal_semidefinite_relaxation_for_SE3_synchronization_problem} (hence also \eqref{primal_semidefinite_relaxation_in_standard_form}) and $I_{dn} \succ 0$ is strictly feasible for \eqref{dual_semidefinite_relaxation_in_standard_form}, Theorem 3.1 of \cite{Vandenberghe1996Semidefinite} implies that the optimal sets of \eqref{primal_semidefinite_relaxation_in_standard_form} and \eqref{dual_semidefinite_relaxation_in_standard_form} are nonempty, and that strong duality holds between them (so that the optimal value of Problem \ref{dual_semidefinite_relaxation_for_SE3_synchronization_problem} is $\SDPval$, as claimed).

%% file: A_Sufficient_Condition_for_Exactness.tex
\section{Proof of Proposition \ref{A_sufficient_condition_for_exact_recovery_prop}}
\label{A_sufficient_condition_for_exactness_appendix}

In this appendix we prove Proposition \ref{A_sufficient_condition_for_exact_recovery_prop}, following the general roadmap of the proof of a similar  result for the special case of \emph{angular synchronization} due to \citet{Bandeira2016Tightness}.  At a high level, our approach is based upon exploiting the Lagrangian duality between Problems \ref{Orthogonal_relaxation_of_the_MLE_problem} and \ref{dual_semidefinite_relaxation_for_SE3_synchronization_problem} to identify a matrix $\certMat$ (constructed from an optimal solution $\Ropt$ of Problem \ref{Orthogonal_relaxation_of_the_MLE_problem}) with the property that $\certMat \succeq 0$ and $\rank(\certMat) = dn - d$ imply that $\Zopt = {\Ropt}\transpose \Ropt$ is the unique optimal solution of Problem \ref{dual_semidefinite_relaxation_for_SE3_synchronization_problem}; we then show that these conditions can be assured by controlling the magnitude of the deviation $\dQ \triangleq \nQ - \tQ$ of the observed data matrix $\nQ$ from its exact latent value $\tQ$.  Specifically, our proof proceeds according to the following chain of reasoning:

\begin{enumerate}
 \item \label{KKT_bullet} We begin by deriving the first-order necessary optimality conditions for the extrinsic formulation of Problem \ref{Orthogonal_relaxation_of_the_MLE_problem}; these take the form $(\nQ - \Lopt){\Ropt}\transpose = 0$, where $\Ropt \in \Orthogonal(d)^n$ is a minimizer of Problem \ref{Orthogonal_relaxation_of_the_MLE_problem} and $\Lopt = \SymBlockDiag_d(\nQ {\Ropt}\transpose \Ropt)$ is a symmetric block-diagonal matrix of Lagrange multipliers corresponding to the orthogonality constraints  in \eqref{orthogonally_constrained_relaxed_program}.
 
 \item \label{sufficient_condition_for_unique_solution_bullet} Exploiting the Lagrangian duality between Problems \ref{Orthogonal_relaxation_of_the_MLE_problem} and \ref{dual_semidefinite_relaxation_for_SE3_synchronization_problem} together with nondegeneracy results for semidefinite programs \cite{Alizadeh1997Nondegeneracy}, we identify a matrix $\certMat \triangleq \nQ - \Lopt$ with the property that $\certMat \succeq 0$ and $\rank(\certMat) = dn - d$ imply that $\Zopt = {\Ropt}\transpose \Ropt$ is the unique optimal solution of Problem \ref{dual_semidefinite_relaxation_for_SE3_synchronization_problem} (Theorem \ref{Sufficient_conditions_for_exact_recovery_for_orthogonally_relaxed_MLE_prop}).  
 
 \item \label{noiseless_case_bullet} We next observe that for the idealized case in which the measurements $\npose_{ij}$ of the relative transforms are \emph{noiseless}, the true latent rotations $\trot$ comprise a minimizer of Problem \ref{Orthogonal_relaxation_of_the_MLE_problem} with corresponding certificate $\true{\certMat} = \tQ  \succeq \TrueRotConLap$.  We then show (by means of similarity) that the spectrum of $\TrueRotConLap$ consists of $d$ copies of the spectrum of the rotational weight graph $\RotW$; in particular, $\TrueRotConLap$ has $d$ eigenvalues equal to $0$, and the remaining $d(n-1)$ are lower-bounded by the algebraic connectivity $\lambda_2(\LapRotW) > 0$ of $\RotW$.  Consequently $\true{\certMat} \succeq 0$, and $\rank(\true{\certMat}) = dn - d$ since $\trot \in \ker \tQ = \ker \true{\certMat}$.  It follows that Problem \ref{dual_semidefinite_relaxation_for_SE3_synchronization_problem} is \emph{always} exact in the absence of measurement noise (Theorem \ref{exactness_of_relaxation_in_the_zero_noise_case_thm}).
 
 \item \label{estimation_error_bullet} In the presence of noise, the minimizer $\Ropt \in \Orthogonal(d)^n$ of Problem \ref{Orthogonal_relaxation_of_the_MLE_problem} will generally not coincide with  the true latent rotations $\trot \in \SO(d)^n$.  Nevertheless, we can still derive an upper bound for the error in the estimate $\Ropt$ in terms of the magnitude of the deviation $\dQ \triangleq \nQ - \tQ$ of the data matrix $\nQ$ from the the true latent value $\tQ$ (Theorem \ref{An_upper_bound_on_the_estimation_error_in_Problem_5_Theorem}).
 
 \item \label{controlling_noise_implies_positive_semidefiniteness_bullet} The first-order necessary optimality condition for Problem \ref{Orthogonal_relaxation_of_the_MLE_problem} (point \ref{KKT_bullet} above) can alternatively be read as $\certMat {\Ropt}\transpose = 0$, which shows that $d$ eigenvalues of $\certMat $ are always $0$.  Since in general the eigenvalues of a matrix $X$ are continuous functions of $X$, it follows from points \ref{noiseless_case_bullet} and \ref{estimation_error_bullet} above and the definition of $\certMat$ that the other $d(n-1)$ eigenvalues of $C$ can be controlled into remaining nonnegative by controlling the norm of $\dQ$.	In light of point \ref{sufficient_condition_for_unique_solution_bullet}, this establishes the existence of a constant $\beta_1 > 0$ such that, if $\lVert \nQ - \tQ \rVert_2 < \beta_1$, $\Zopt = {\Ropt}\transpose \Ropt$ is the unique solution of Problem \ref{dual_semidefinite_relaxation_for_SE3_synchronization_problem}.
 
 \item \label{controlling_noise_implies_correct_orientation_bullet} Finally, we observe that since $\Orthogonal(d)$ is the disjoint union of two  components separated by a distance $\sqrt{2}$ in the Frobenius norm, and the true latent rotations $\trot_i$ all lie in the $+1$ component, it follows from point \ref{estimation_error_bullet} that there exists a constant $\beta_2 > 0$ such that, if $\lVert \nQ - \tQ \rVert_2 < \beta_2$, a minimizer $\Ropt \in \Orthogonal(d)^n$ of Problem \ref{Orthogonal_relaxation_of_the_MLE_problem} must in fact lie in $\SO(d)^n$, and is therefore also a minimizer of Problem \ref{Simplified_maximum_likelihood_estimation_for_SE3_synchronization}.
 
 \item \label{establishing_proposition_bullet} Taking $\beta \triangleq \min \lbrace \beta_1, \beta_2 \rbrace$, Proposition \ref{A_sufficient_condition_for_exact_recovery_prop} follows from points \ref{controlling_noise_implies_positive_semidefiniteness_bullet} and \ref{controlling_noise_implies_correct_orientation_bullet} and Theorem \ref{certifying_exactness_theorem}.
\end{enumerate}

The remainder of this appendix is devoted to rigorously establishing each of claims \ref{KKT_bullet}--\ref{establishing_proposition_bullet} above.

\subsection{Gauge symmetry and invariant metrics for Problems \ref{Simplified_maximum_likelihood_estimation_for_SE3_synchronization} and \ref{Orthogonal_relaxation_of_the_MLE_problem}}

A critical element of the proof of Proposition \ref{A_sufficient_condition_for_exact_recovery_prop} is the derivation of an upper bound for the estimation error of a minimizer $\Ropt$ of Problem \ref{Orthogonal_relaxation_of_the_MLE_problem}  as a function of $\dQ$ (point \ref{estimation_error_bullet}).  However, we have observed previously that Problems \ref{SE3_synchronization_MLE_NLS_problem}--\ref{Orthogonal_relaxation_of_the_MLE_problem} always admit \emph{multiple} (in fact, infinitely many) solutions for dimensions $d \ge 2$ due to gauge symmetry.\footnote{Recall that $\SO(1) = \lbrace +1 \rbrace$, so for $d = 1$ Problems \ref{Rotational_maximum_likelihood_estimation_for_SE3_synchronization} and \ref{Simplified_maximum_likelihood_estimation_for_SE3_synchronization} admit only the (trivial) solution $\Ropt = (1, \dotsc, 1)$.  Similarly, $\Orthogonal(1) = \lbrace \pm 1 \rbrace$, so for $d = 1$ Problem \ref{Orthogonal_relaxation_of_the_MLE_problem} admits pairs of solutions related by multiplication by $-1$.}  In consequence, it may not be immediately clear how we should quantify the estimation error of a \emph{particular} point estimate $\Ropt$ obtained as a minimizer of Problem \ref{Simplified_maximum_likelihood_estimation_for_SE3_synchronization} or Problem \ref{Orthogonal_relaxation_of_the_MLE_problem}, since $\Ropt$ is an arbitrary representative of an infinite set of \emph{distinct} but \emph{equivalent} minimizers that are related by gauge transforms.  To address this complication, in this subsection we study the gauge symmetries of Problems \ref{Simplified_maximum_likelihood_estimation_for_SE3_synchronization} and \ref{Orthogonal_relaxation_of_the_MLE_problem}, and then develop a pair of \emph{gauge-invariant} metrics suitable for quantifying estimation error in these problems in a consistent, ``symmetry aware'' manner.

Recall from Section \ref{Problem_formulation_subsection} (cf.\ footnote \ref{MLEs_are_nonunique_footnote} on page \pageref{MLEs_are_nonunique_footnote}) that solutions of Problem \ref{SE3_synchronization_MLE_NLS_problem} are determined only up to a global gauge symmetry (corresponding to the diagonal left-action of $\SE(d)$ on $\SE(d)^n$).  Similarly, it is straightforward to verify that if $\Ropt$ is any minimizer of \eqref{Simplified_maximum_likelihood_estimation_for_SE3_synchronization_optimization_problem}  (resp.\ \eqref{orthogonally_relaxed_primal_problem}), then $G \bullet \Ropt$  also minimizes \eqref{Simplified_maximum_likelihood_estimation_for_SE3_synchronization_optimization_problem} (resp.\ \eqref{orthogonally_relaxed_primal_problem}) for any choice of $G \in \SO(d)$ (resp. $G \in \Orthogonal(d)$), where $\bullet$ is the diagonal left-action of $\Orthogonal(d)$ on $\Orthogonal(d)^n$:
\begin{equation}
 \label{diagonal_left_action_of_SO3_on_SO3n}
 \begin{split}
 G \bullet \rot \triangleq (G \rot_1, \dotsc, G \rot_n).
 \end{split}
\end{equation}
It follows that the sets of minimizers of Problems \ref{Simplified_maximum_likelihood_estimation_for_SE3_synchronization} and \ref{Orthogonal_relaxation_of_the_MLE_problem} are partitioned into  \emph{orbits} of the form:

\pagebreak

\begin{subequations}
\label{orbit_definitions}
\begin{equation}
 \label{Sorbit_definition}
 \Sorbit(\rot) \triangleq \lbrace G \bullet \rot \mid G \in \SO(d) \rbrace \subset \SO(d)^n,
\end{equation}
\begin{equation}
 \label{Oorbit_definition}
\Oorbit(\rot) \triangleq \lbrace G \bullet \rot \mid G \in \Orthogonal(d) \rbrace \subset \Orthogonal(d)^n,
\end{equation}
\end{subequations}
each of which comprise a set of point estimates that are \emph{equivalent} to $R$ for the estimation problems \eqref{Simplified_maximum_likelihood_estimation_for_SE3_synchronization_optimization_problem} and \eqref{orthogonally_relaxed_primal_problem}, respectively.  Consequently, when quantifying the dissimilarity between a pair of feasible points $X, Y$ for Problem \ref{Simplified_maximum_likelihood_estimation_for_SE3_synchronization} or \ref{Orthogonal_relaxation_of_the_MLE_problem}, we are interested in measuring the distance between the \emph{orbits} determined by these two points, \emph{not} the distance between the \emph{specific representatives} $X$ and $Y$ themselves (which may in fact represent \emph{equivalent} solutions, but differ in their assignments of coordinates by a global gauge symmetry).  We therefore introduce the following \emph{orbit distances}:
\begin{subequations}
\label{orbit_distance_definitions}
\begin{equation}
 \label{Sorbit_distance_definition}
\Sorbdist(X, Y) \triangleq \min_{G \in \SO(d)} \lVert  X - G \bullet Y \rVert_F, \quad X,Y \in \SO(d)^n,
\end{equation}
\begin{equation}
\label{Oorbit_distance_definition}
\Oorbdist(X, Y) \triangleq \min_{G \in \Orthogonal(d)} \lVert  X - G \bullet Y \rVert_F, \quad X,Y \in \Orthogonal(d)^n;
\end{equation}
\end{subequations}
these functions report the Frobenius norm distance between the two closest representatives of the orbits \eqref{orbit_definitions} in $\SO(d)^n$ and $\Orthogonal(d)^n$ determined by $X$ and $Y$, respectively.  Using the orthogonal invariance of the Frobenius norm, it is straightforward to verify that these functions satisfy:
\begin{subequations}
\label{orbit_distance_invariance}
\begin{equation}
 \label{Sorbit_distance_invariance}
\Sorbdist(X, Y) = \Sorbdist(G_1 \bullet X, G_2 \bullet Y), \quad \quad  X, Y \in \SO(d)^n, \; G_1, G_2 \in \SO(d),
\end{equation}
\begin{equation}
\label{Oorbit_distance_invariance}
\Oorbdist(X, Y) = \Oorbdist(G_1 \bullet X, G_2 \bullet Y), \quad \quad X, Y\in \Orthogonal(d)^n, \; G_1, G_2 \in \Orthogonal(d),
\end{equation}
\end{subequations}
i.e., they define notions of dissimilarity between feasible points of Problems \ref{Simplified_maximum_likelihood_estimation_for_SE3_synchronization} and \ref{Orthogonal_relaxation_of_the_MLE_problem}, respectively, that are \emph{invariant} with respect to the action of the gauge symmetries for these problems, and so provide a consistent, gauge-invariant means of quantifying the estimation error.\footnote{We remark that while the formulation of the distance functions presented in \eqref{orbit_distance_definitions} may at first appear somewhat \emph{ad hoc}, one can justify this choice rigorously using the language of Riemannian geometry \cite{Boothby2003Riemannian,Kobayashi1996Foundations}.  Since the Frobenius norm distance is orthogonally invariant, the diagonal left-actions \eqref{diagonal_left_action_of_SO3_on_SO3n} of $\SO(d)$ on $\SO(d)^n$ and $\Orthogonal(d)$ on $\Orthogonal(d)^n$ are \emph{isometries}, and are trivially free and proper; consequently, the quotient spaces $\manifold_{\Sorbit} \triangleq \SO(d)^n / \SO(d)$ and $\manifold_{\Oorbit} \triangleq \Orthogonal(d)^n / \Orthogonal(d)$ obtained by \emph{identifying} the elements of the orbits \eqref{orbit_definitions} are manifolds, and the projections $\pi_{\Sorbit} \colon \SO(d)^n \to \manifold_{\Sorbit}$ and $\pi_{\Oorbit} \colon \Orthogonal(d)^n \to \manifold_{\Oorbit}$ are submersions.  Furthermore, it is straightforward to verify that the restrictions of the derivative maps $d(\pi_{\Sorbit})_{R} \colon T_{R}(\SO(d)^n) \to T_{[R]}(\manifold_{\Sorbit})$ and $d(\pi_{\Oorbit})_{R} \colon T_{R}(\Orthogonal(d)^n) \to T_{[R]}(\manifold_{\Oorbit})$ to the horizontal spaces $H_{R}(\SO(d)^n) \triangleq \ker(d(\pi_{\Sorbit})_R)^{\perp}$ and $H_{R}(\Orthogonal(d)^n) \triangleq \ker(d(\pi_{\Oorbit})_R)^{\perp}$ are linear isometries onto $T_{[R]}(\manifold_{\Sorbit})$ and $T_{[R]}(\manifold_{\Oorbit})$, respectively, and therefore induce well-defined Riemannian metrics on the quotient spaces $\manifold_{\Sorbit}$ and $\manifold_{\Oorbit}$ from the Riemannian metrics on the total spaces $\SO(d)^n$ and $\Orthogonal(d)^n$, with corresponding distance functions $d_{\manifold_{\Sorbit}}(\cdot, \cdot)$, $d_{\manifold_{\Oorbit}}(\cdot, \cdot)$, respectively.  The functions $\Sorbdist(\cdot, \cdot)$ and $\Oorbdist(\cdot, \cdot)$ defined in \eqref{orbit_distance_definitions} are then simply the functions that report the distances between the images of $X$ and $Y$ after projecting them to these quotient spaces: $\Sorbdist(X,Y) = d_{\manifold_{\Sorbit}}(\pi_{\Sorbit}(X), \pi_{\Sorbit}(Y))$ and $\Oorbdist(X,Y) = d_{\manifold_{\Oorbit}}(\pi_{\Oorbit}(X), \pi_{\Oorbit}(Y))$.  Consequently, these are in fact the canonical distance functions  for comparing points in $\SO(d)^n$ and $\Orthogonal(d)^n$ while accounting for the  gauge symmetry \eqref{diagonal_left_action_of_SO3_on_SO3n}.}

The following result enables us to compute these distances easily in practice:

\begin{thm}[Computing the orbit distance]
\label{computing_the_orbit_distance_theorem}
Given $X, Y \in \Orthogonal(d)^n$, let
\begin{equation}
\label{singular_value_decomposition_for_computation_of_orbit_distance}
 X Y\transpose = U \Sigma V\transpose
\end{equation}
be a singular value decomposition of $XY\transpose$ with $\Sigma = \Diag(\sigma_1, \dotsc, \sigma_d)$ and  $\sigma_1 \ge \dotsb \ge \sigma_d \ge 0$.  Then the orbit distance $\Oorbdist(X,Y)$ between $\Oorbit(X)$ and $\Oorbit(Y)$ in $\Orthogonal(d)^n$ is:
\begin{equation}
\label{closed_form_Oorbdist_computation}
\Oorbdist(X,Y) = \sqrt{2dn - 2\lVert XY\transpose \rVert_{*}},
\end{equation}
and a minimizer $G_{\Oorbit}^* \in \Orthogonal(d)$ realizing this optimal value in \eqref{Oorbit_distance_definition} is given by:
\begin{equation}
\label{optimal_registration_for_orthogonal_Procrustes_problem}
 G_{\Oorbit}^* = U V\transpose.
\end{equation}
If additionally $X,Y \in \SO(d)$, then the orbit distance $\Sorbdist(X,Y)$ between $\Sorbit(X)$ and $\Sorbit(Y)$ in $\SO(d)^n$ is given  by:
\begin{equation}
\label{closed_form_Sorbdist_computation}
 \Sorbdist(X,Y) = \sqrt{2dn - 2\tr(\varXi\Sigma)},
\end{equation}
where $\varXi$ is the matrix:
\begin{equation}
\label{varXi_matrix_definition}
 \varXi = \Diag\left(1, \dotsc, 1, \det(UV\transpose)\right) \in \R^{d \times d},
\end{equation}
and a minimizer $G_{\Sorbit}^* \in \SO(d)$ realizing this optimal value in \eqref{Sorbit_distance_definition} is given by:
\begin{equation}
\label{optimal_registration_for_special_orthogonal_Procrustes_problem}
G_{\Sorbit}^* = U \varXi V\transpose.
\end{equation}
\end{thm}

\begin{proof}
Observe that
\begin{equation}
\label{squared_Frobenius_norm_error_for_orbit_distance_derivation}
  \begin{split}
\left \lVert X - G \bullet Y \right \rVert_F^2 &= \sum_{i = 1}^n \left \lVert  X_i - GY_i \right \rVert_F^2 \\
&= 2dn -2 \sum_{i = 1}^n \left\langle X_i, GY_i \right\rangle_F \\
&= 2dn - 2 \left \langle G, \sum_{i = 1}^n X_i Y_i\transpose \right \rangle_F \\
&= 2dn - 2 \left\langle G, X Y\transpose \right\rangle_F.
\end{split}
\end{equation}
Consequently, a minimizer $G_{\Oorbit}^* \in \Orthogonal(d)$ attaining the optimal value in \eqref{Oorbit_distance_definition} is determined by:
\begin{equation}
\label{maximizing_G_value_in_Oorbit_distance}
 G_{\Oorbit}^* \in \argmax_{G \in \Orthogonal(d)} \left\langle G, X Y\transpose \right\rangle_F.
\end{equation}
But now we may recognize \eqref{maximizing_G_value_in_Oorbit_distance} as an instance of the orthogonal Procrustes problem, with maximizer $G_{\Oorbit}^*$ given by \eqref{optimal_registration_for_orthogonal_Procrustes_problem} \cite{Hanson1981Analysis,Umeyama1991Least}.  Substituting  \eqref{singular_value_decomposition_for_computation_of_orbit_distance} and \eqref{optimal_registration_for_orthogonal_Procrustes_problem} into \eqref{squared_Frobenius_norm_error_for_orbit_distance_derivation} and simplifying the resulting expression using the orthogonal invariance of the Frobenius inner product then shows that the orbit distance $d_{\Oorbit}(X,Y)$ is:
\begin{equation}
 d_{\Oorbit}(X,Y) = \sqrt{2dn - 2\tr(\Sigma)} = \sqrt{2dn - 2\lVert XY\transpose \rVert_{*}},
\end{equation}
which is \eqref{closed_form_Oorbdist_computation}.  If additionally $X,Y \in \SO(d)$, then \eqref{squared_Frobenius_norm_error_for_orbit_distance_derivation} implies that a minimizer $G_{\Sorbit}^* \in \SO(d)$ attaining the optimal value in \eqref{Sorbit_distance_definition} is determined by:
\begin{equation}
 \label{maximizing_G_value_in_Sorbit_distance}
 G_{\Sorbit}^* \in \argmax_{G \in \SO(d)} \left\langle G, X Y\transpose \right\rangle_F.
\end{equation}
Equation \eqref{maximizing_G_value_in_Sorbit_distance} is an instance of the special orthogonal Procrustes problem, with a maximizer $G_{\Sorbit}^* \in \SO(d)$ given by \eqref{optimal_registration_for_special_orthogonal_Procrustes_problem} \cite{Hanson1981Analysis,Umeyama1991Least}.  Substituting \eqref{singular_value_decomposition_for_computation_of_orbit_distance} and \eqref{optimal_registration_for_special_orthogonal_Procrustes_problem} into \eqref{squared_Frobenius_norm_error_for_orbit_distance_derivation} and once again simplifying the resulting expression using the orthogonal invariance of the Frobenius inner product then shows that the orbit distance $d_{\Sorbit}(X,Y)$ is given by \eqref{closed_form_Sorbdist_computation}.
\end{proof}

\subsection{A sufficient condition for exact recovery in Problem \ref{Orthogonal_relaxation_of_the_MLE_problem}}

In this subsection we address points \ref{KKT_bullet} and \ref{sufficient_condition_for_unique_solution_bullet} in our roadmap, deriving sufficient conditions to ensure the recovery of a minimizer $\Ropt \in \Orthogonal(d)^n$ of Problem \ref{Orthogonal_relaxation_of_the_MLE_problem} by means of solving the dual semidefinite relaxation Problem \ref{dual_semidefinite_relaxation_for_SE3_synchronization_problem}.  Our approach is based upon exploiting the Lagrangian duality between Problem \ref{Orthogonal_relaxation_of_the_MLE_problem} and Problems \ref{primal_semidefinite_relaxation_for_SE3_synchronization_problem} and \ref{dual_semidefinite_relaxation_for_SE3_synchronization_problem} to construct a matrix $\certMat$ whose positive semidefiniteness serves as a \emph{certificate of optimality} for $\Zopt = {\Ropt}\transpose \Ropt$ as a solution of Problem \ref{dual_semidefinite_relaxation_for_SE3_synchronization_problem}.  

We begin by deriving the first-order necessary optimality conditions for \eqref{orthogonally_constrained_relaxed_program}:

\begin{lem}[First-order necessary optimality conditions for Problem \ref{Orthogonal_relaxation_of_the_MLE_problem}]
\label{KKT_conditions_for_primal_semidefinite_relaxation_lemma}
If $\Ropt \in \Orthogonal(d)^n$ is a minimizer of Problem \ref{Orthogonal_relaxation_of_the_MLE_problem}, then there exists a  matrix $\Lopt \in \SBD(d,n)$ such that
 \begin{equation}
 \label{primal_semidefinite_relaxation_first_order_necessary_condition_eq}
 (\nQ- \Lopt){\Ropt}\transpose = 0.
 \end{equation}
Furthermore, $\Lopt$ can be computed in closed form according to:
\begin{equation}
\label{closed_form_solution_for_Lambda_star}
\Lopt = \SymBlockDiag_d\left(\nQ {\Ropt}\transpose \Ropt \right).
\end{equation}
\end{lem}

\begin{proof}
If we consider \eqref{orthogonally_relaxed_primal_problem} as an \emph{unconstrained} minimization of the objective $F(R) \triangleq \tr(\nQ \rot\transpose \rot)$ over the Riemannian manifold $\Orthogonal(d)^n$, then the first-order necessary optimality condition is simply that the Riemannian gradient $\grad F$ must vanish at the minimizer $\Ropt$:
\begin{equation}
\label{intrinsic_first_order_necessary_condition}
 \grad F(\Ropt) = 0.
\end{equation}
Furthermore, if we consider $\Orthogonal(d)^n$ as an embedded submanifold of $\R^{d \times dn}$, then this embedding induces a simple relation between the Riemannian gradient $\grad F$ of $F$ viewed as a function restricted to $\Orthogonal(d)^n \subset \R^{d \times dn}$ and $\nabla F$, the gradient of $F$ considered as a function on the ambient Euclidean space $\R^{d \times dn}$.  Specifically, we have:
\begin{equation}
\label{Euclidean_and_Riemannian_gradient_relation_for_proof_of_certification_theorem}
\grad F(R) = \proj_R \nabla F(R),
\end{equation}
where $\proj_R \colon T_R(\R^{d \times dn}) \to T_R(\Orthogonal(d)^n)$ is the orthogonal projection operator onto the tangent space of $\Orthogonal(d)^n$ at $R$  \cite[eq.\ (3.37)]{Absil2009Optimization}; this latter is given  explicitly by:
\begin{equation}
\label{rotation_manifold_orthogonal_projection_operator}
\begin{split}
 \proj_R &\colon T_R\left(\R^{d \times dn} \right)\to T_R\left(\Orthogonal(d)^n \right) \\
 \proj_R(X) &= X - R \SymBlockDiag_d(R\transpose X).
\end{split}
\end{equation}
Straightforward differentiation shows that the Euclidean gradient is $\nabla F(R) = 2R\nQ$, and consequently \eqref{intrinsic_first_order_necessary_condition}--\eqref{rotation_manifold_orthogonal_projection_operator} imply that:
\begin{equation}
 \label{derivation_of_Lagrange_multiplier_in_closed_form_eq}
 0 = \proj_{\Ropt} \nabla F(\Ropt) = 2\Ropt\nQ - 2\Ropt \SymBlockDiag_d\left({\Ropt}\transpose \Ropt \nQ\right).
\end{equation}
Dividing both sides of \eqref{derivation_of_Lagrange_multiplier_in_closed_form_eq} by $2$ and taking the transpose produces \eqref{primal_semidefinite_relaxation_first_order_necessary_condition_eq}, using the definition of $\Lopt$ given in \eqref{closed_form_solution_for_Lambda_star}.
\end{proof}

Despite its simplicity, it turns out that Lemma \ref{KKT_conditions_for_primal_semidefinite_relaxation_lemma} is actually already enough to enable the derivation of sufficient conditions to ensure the exactness of the semidefinite relaxation Problem \ref{dual_semidefinite_relaxation_for_SE3_synchronization_problem} with respect to Problem \ref{Orthogonal_relaxation_of_the_MLE_problem}.  Comparing \eqref{primal_semidefinite_relaxation_first_order_necessary_condition_eq} with the \emph{extrinsic} formulation \eqref{orthogonally_constrained_relaxed_program} of Problem \ref{Orthogonal_relaxation_of_the_MLE_problem}, we may recognize $\Lopt$ as nothing more than a matrix of Lagrange multipliers corresponding to the orthogonality constraints $\rot_i\transpose \rot_i = I_d$.  Consequently, in the case that exactness holds between Problems \ref{Orthogonal_relaxation_of_the_MLE_problem} and \ref{dual_semidefinite_relaxation_for_SE3_synchronization_problem} (i.e., that $\Zopt = {\Ropt}\transpose \Ropt$ is a minimizer of Problem \ref{dual_semidefinite_relaxation_for_SE3_synchronization_problem}), strong duality obtains \emph{a fortiori} between Problems \ref{Orthogonal_relaxation_of_the_MLE_problem} and \ref{primal_semidefinite_relaxation_for_SE3_synchronization_problem}, and therefore $\Lopt$ also comprises an optimal solution for the Lagrangian relaxation \eqref{primal_semidefinite_relaxation_for_SE3_synchronization_optimization} (cf.\ e.g.\ \cite[Sec.\ 5.5.3]{Boyd2004Convex}).  It follows that $\nQ - \Lopt \succeq 0$ (since $\Lopt$ is \emph{a fortiori} feasible for Problem \ref{primal_semidefinite_relaxation_for_SE3_synchronization_problem} if it is optimal), and $(\nQ - \Lopt) \Zopt = 0$ (from the definition of $\Zopt$ and \eqref{primal_semidefinite_relaxation_first_order_necessary_condition_eq}).  But now observe that these last two conditions are precisely the first-order necessary and sufficient conditions for $\Zopt$ to be an optimal solution of Problem \ref{dual_semidefinite_relaxation_for_SE3_synchronization_problem} (cf.\ equation (33) of \cite{Vandenberghe1996Semidefinite}); furthermore, they provide a \emph{closed-form expression} for a KKT certificate for $\Zopt$ (namely, $\nQ - \Lopt$) in terms of a minimizer $\Ropt$ of Problem \ref{Orthogonal_relaxation_of_the_MLE_problem} using \eqref{closed_form_solution_for_Lambda_star}.  The utility of this expression is that, although it was originally derived under the \emph{assumption} of exactness, it can also be exploited to derive a \emph{sufficient condition} for same, as shown in the next theorem.

\begin{thm}[A sufficient condition for exact recovery in Problem \ref{Orthogonal_relaxation_of_the_MLE_problem}]
\label{Sufficient_conditions_for_exact_recovery_for_orthogonally_relaxed_MLE_prop}
Let $\Ropt \in \Orthogonal(d)^n$ be a minimizer of Problem \ref{Orthogonal_relaxation_of_the_MLE_problem} with corresponding Lagrange multiplier matrix $\Lopt \in \SBD(d,n)$ as in Lemma \ref{KKT_conditions_for_primal_semidefinite_relaxation_lemma}, and define
\begin{equation}
\label{candidate_certificate_matrix_definition}
 \certMat  \triangleq \nQ - \Lopt.
\end{equation}
If $\certMat \succeq 0$, then $\Zopt = {\Ropt}\transpose \Ropt$ is a minimizer of Problem \ref{dual_semidefinite_relaxation_for_SE3_synchronization_problem}.  If additionally $\rank(\certMat) = dn - d$, then $\Zopt$ is the \emph{unique} minimizer of Problem \ref{dual_semidefinite_relaxation_for_SE3_synchronization_problem}.
\end{thm}

\begin{proof}
Since $C = \nQ - \Lopt \succeq 0$ by hypothesis, and equation \eqref{primal_semidefinite_relaxation_first_order_necessary_condition_eq} of Lemma \ref{KKT_conditions_for_primal_semidefinite_relaxation_lemma} implies that
\begin{equation}
\label{complementary_slackness_condition_for_certification_of_RtR_as_solution_of_Problem_6}
  (\nQ - \Lopt)\Zopt = (\nQ - \Lopt)\left({\Ropt}\transpose \Ropt \right) = 0,
\end{equation}
$\Lopt$ and $\Zopt$ satisfy the necessary and sufficient conditions characterizing primal-dual pairs of optimal solutions for the strictly feasible primal-dual pair of semidefinite programs \eqref{primal_semidefinite_relaxation_for_SE3_synchronization_optimization} and \eqref{dual_semidefinite_relaxation_for_SE3_synchronization_optimization} (cf.\ Appendix \ref{Deriving_the_dual_semidefinite_relaxation_from_the_primal_subsection} and  Theorem 3.1 in \cite{Vandenberghe1996Semidefinite}).  In other words, $\certMat  \succeq 0$ \emph{certifies} the optimality of $\Zopt$ as a solution of Problem \ref{dual_semidefinite_relaxation_for_SE3_synchronization_problem}.

Now suppose further that $\rank(\certMat) = dn - d$; we establish that $\Zopt$ is the \emph{unique} solution of Problem \ref{dual_semidefinite_relaxation_for_SE3_synchronization_problem} using nondegeneracy results from \cite{Alizadeh1997Nondegeneracy}.  Specifically, we observe that the equivalent formulations \eqref{dual_semidefinite_relaxation_in_standard_form} and \eqref{primal_semidefinite_relaxation_in_standard_form} of Problems \ref{dual_semidefinite_relaxation_for_SE3_synchronization_problem} and \ref{primal_semidefinite_relaxation_for_SE3_synchronization_problem} derived in Appendix \ref{Deriving_the_dual_semidefinite_relaxation_from_the_primal_subsection} match the forms of the primal and dual semidefinite programs given in equations (2) and (3) of \cite{Alizadeh1997Nondegeneracy}, respectively.  Consequently, Theorem 10 of \cite{Alizadeh1997Nondegeneracy} guarantees that Problem \ref{dual_semidefinite_relaxation_for_SE3_synchronization_problem} has a unique solution provided that we can exhibit a dual nondegenerate solution of Problem \ref{primal_semidefinite_relaxation_for_SE3_synchronization_problem}.  Since we have already identified  $\Lopt$ as a solution of Problem \ref{primal_semidefinite_relaxation_for_SE3_synchronization_problem} (via \eqref{complementary_slackness_condition_for_certification_of_RtR_as_solution_of_Problem_6}), it suffices to show that $\Lopt$ is dual nondegenerate.  To that end, let
\begin{equation}
\label{eigendecomposition_of_primal_certificate_of_optimality_for_proof_of_uniqueness_of_dual_relaxation_solution}
\nQ - \Lopt = 
\begin{pmatrix}
U & V
\end{pmatrix}
\Diag(0, 0, 0, \sigma_{d+1}, \dotsc, \sigma_{dn}) 
 \begin{pmatrix}
U & V 
\end{pmatrix}\transpose
\end{equation}
be an eigendecomposition of $\nQ - \Lopt$ as in equation (16) of \cite{Alizadeh1997Nondegeneracy} (with $\sigma_k > 0$ for $k \ge d + 1$, $U \in \R^{dn \times d}$, and $V \in \R^{dn \times (dn - d)}$).   Theorem 9 of \cite{Alizadeh1997Nondegeneracy} states that $\Lopt$ is dual nondegenerate if and only if
\begin{equation}
\label{sufficient_condition_for_uniqueness_of_dual_certificate_for_Lagrangian_dual_relaxation}
 \left \lbrace U\transpose \varXi U \mid \varXi \in \SBD(d,n)  \right \rbrace = \Sym(d).
\end{equation}
Now the matrix $U$ appearing in \eqref{eigendecomposition_of_primal_certificate_of_optimality_for_proof_of_uniqueness_of_dual_relaxation_solution} can be characterized as a matrix whose columns form an orthonormal basis for the $d$-dimensional null space of $\nQ - \Lopt$.  But equation \eqref{primal_semidefinite_relaxation_first_order_necessary_condition_eq} shows that the columns of ${\Ropt}\transpose$ span this same subspace, and are pairwise orthogonal since ${\Ropt}\transpose$ is composed of orthogonal blocks.  Consequently, without loss of generality we may take $U = \frac{1}{\sqrt{n}} {\Ropt} \transpose$ in \eqref{eigendecomposition_of_primal_certificate_of_optimality_for_proof_of_uniqueness_of_dual_relaxation_solution}.  Now we can write the left-hand side of \eqref{sufficient_condition_for_uniqueness_of_dual_certificate_for_Lagrangian_dual_relaxation} more explicitly as:
\begin{equation}
\label{explicit_form_for_span_of_symmetric_product_for_proof_of_uniqueness_of_solution_for_dual_semidefinite_relaxation}
 \left \lbrace U\transpose \varXi U \mid \varXi \in \SBD(d,n)  \right \rbrace =  \left \lbrace  \frac{1}{n} \sum_{i = 1}^n \Ropt_i \varXi_i {\Ropt_i}\transpose \mid \varXi \in \SBD(d,n) \right \rbrace,
\end{equation}
and it is immediate from \eqref{explicit_form_for_span_of_symmetric_product_for_proof_of_uniqueness_of_solution_for_dual_semidefinite_relaxation} that given any $S \in \Sym(d)$, we have $S = U\transpose \varXi U$ for $\varXi = \Diag(n {\Ropt_1}\transpose S {\Ropt_1}, 0, \dotsc, 0)$.  This shows that \eqref{sufficient_condition_for_uniqueness_of_dual_certificate_for_Lagrangian_dual_relaxation} holds, so $\Lopt$ is a dual nondegenerate solution of Problem \ref{primal_semidefinite_relaxation_for_SE3_synchronization_problem}, and we conclude that $\Zopt$ is indeed the \emph{unique} minimizer of Problem \ref{dual_semidefinite_relaxation_for_SE3_synchronization_problem}, as claimed.
\end{proof}

In short, Theorem \ref{Sufficient_conditions_for_exact_recovery_for_orthogonally_relaxed_MLE_prop} enables us to reduce the question of Problem \ref{dual_semidefinite_relaxation_for_SE3_synchronization_problem}'s exactness to the  problem of verifying the positive semidefiniteness of a certain matrix $\certMat$ that can be constructed from an optimal solution of Problem \ref{Orthogonal_relaxation_of_the_MLE_problem}.  The remainder of this appendix is devoted to establishing conditions that are sufficient to guarantee that this latter (much simpler) criterion is satisfied.

\subsection{The noiseless case}

As our first application of Theorem \ref{Sufficient_conditions_for_exact_recovery_for_orthogonally_relaxed_MLE_prop}, in this subsection we prove that the semidefinite relaxation Problem \ref{dual_semidefinite_relaxation_for_SE3_synchronization_problem} is always exact in the (highly idealized) case that the measurements $\npose_{ij}$ in \eqref{probabilistic_generative_model_for_noisy_observations} are noiseless.  In addition to providing a baseline sanity check for the feasibility of our overall strategy, our analysis of this idealized case will also turn out to admit a straightforward generalization that suffices to prove Proposition \ref{A_sufficient_condition_for_exact_recovery_prop}.

To that end, let $\TrueRotConLap$ and $\tQtran$ denote the rotational and translational data matrices of the form appearing in Problem \ref{Simplified_maximum_likelihood_estimation_for_SE3_synchronization} constructed using the true (latent) relative transforms $\tpose_{ij}$ appearing in \eqref{probabilistic_generative_model_for_noisy_observations}.  The following pair of lemmas characterize several important properties of these matrices:

\begin{lem}[Exact rotational connection Laplacians]
\label{exact_rotational_connection_Laplacian_lemma}
Let $\TrueRotConLap$ be the rotational connection Laplacian constructed using the true \emph{(}latent\emph{)} relative rotations $\trot_{ij} \triangleq \trot_i^{-1} \trot_j$ in \eqref{probabilistic_generative_model_for_noisy_observations}, $\RotW$ the corresponding rotational weight graph, and $\trot \in \SO(d)^n$ the matrix of true \emph{(}latent\emph{)} rotational states.  Then the following hold:
 \begin{enumerate}
  \item [$(i)$]  $\LapRotW \otimes I_d = S \TrueRotConLap S^{-1}$ for $S = \Diag(\trot_1, \dotsc, \trot_n) \in \R^{dn \times dn}$;
    \item [$(ii)$]  $\lambda_{d+1}(\TrueRotConLap) = \lambda_2(\LapRotW)$;
  \item [$(iii)$]  $\ker(\TrueRotConLap) = \lbrace \trot\transpose v \mid v \in \R^d \rbrace$.
 \end{enumerate}
\end{lem}

\begin{proof}
A direct computation using the definitions of $S$ in claim (i) and of $\TrueRotConLap$ in \eqref{connection_Laplacian_definition} shows that the $(i,j)$-block of the product $S \TrueRotConLap S^{-1}$ is given by:
\begin{equation}
 \left(S \TrueRotConLap S^{-1}\right)_{ij} = 
 \begin{cases}
\rotdeg_i I_d, & i =j, \\
-\kappa_{ij} I_d, & \edge \in \Edges, \\
0_{d \times d}, & \edge \notin \Edges,
 \end{cases}
\end{equation}
which we recognize as the $(d \times d)$-block description of $\LapRotW \otimes I_d$; this proves claim (i).  For claim (ii), we observe that $\LapRotW \otimes I_d$ and $\TrueRotConLap$ have the same spectrum (since claim (i) shows that they are similar), and the spectrum of $\LapRotW \otimes I_d$ consists of $d$ copies of the spectrum of $\LapRotW$ (this follows from the fact that the spectra of   $A \in \R^{d_1 \times d_1}$, $B \in \R^{d_2 \times d_2}$ and $A \otimes B$ are related by $\lambda(A \otimes B) = \lbrace \lambda_i(A) \lambda_j(B) \mid i \in [d_1], \: j \in [d_2] \rbrace$, see e.g.\ \cite[Theorem 4.2.12]{Horn1991Topics}).  For claim (iii), another direct computation using definition \eqref{connection_Laplacian_definition} shows that $\TrueRotConLap \trot\transpose = 0$, and therefore that $\image(\trot\transpose) = \lbrace \trot\transpose v \mid v \in \R^d \rbrace \subseteq \ker(\TrueRotConLap)$; furthermore, $\dim(\image(\trot \transpose)) = d$ since $\rank(\trot\transpose) = d$ (as it has $d$ orthonormal columns).  On the other hand, claim (ii) shows that $\lambda_{d+1}(\TrueRotConLap) = \lambda_2(\LapRotW) > 0$ since $G$ is connected, and therefore $\dim(\ker(\TrueRotConLap)) \le d$; consequently, $\image(R\transpose)$ is \emph{all} of $\ker(\TrueRotConLap)$.
\end{proof}

\begin{lem}[Orthogonal projections of exact measurements]
\label{exact_translational_projector_lemma}
Let $\tT \in \R^{m \times dn}$ denote the data matrix of the form \eqref{nT_matrix_definition} constructed using the true \emph{(}latent\emph{)} values of the translational measurements $\ttran_{ij}$ in \eqref{probabilistic_generative_model_for_noisy_observations} and $\trot \in \SO(d)^n$ the matrix of true \emph{(}latent\emph{)} rotational states.  Then $\tranPrecisions^{\frac{1}{2}} \tT \trot\transpose \in \ker \OrthoProjMatrix$.
\end{lem}

\begin{proof}
It follows from \eqref{SE3_group_operations} and the definitions of $\ttran_{ij}$ in \eqref{probabilistic_generative_model_for_noisy_observations} and $\tT$ in \eqref{nT_matrix_definition} that the product
$ \tT \trot\transpose \in \R^{m \times d}$ is a $(1 \times d)$-block structured matrix with rows indexed by the edges $\dedge \in \dEdges$ and whose $\dedge$-th row is given by:
\begin{equation}
\label{row_elements_of_tranPrecisions_tT_trot_transpose}
 \left(  \tT \trot \transpose \right)_{\dedge} = -\ttran_{ij}\transpose \trot_{i}\transpose = -\left(\trot_{i} \ttran_{ij}\right)\transpose = \ttran_i\transpose - \ttran_j\transpose .
\end{equation}
Now observe that the quantities  $\ttran_i\transpose - \ttran_j\transpose$ associated with each edge $(i,j) \in \dEdges$ in the product $\tT \trot\transpose$ are realizable as differences of values $\ttran_i, \ttran_j$ assigned to the endpoints of $(i,j)$, i.e., the columns of $\tT \trot \transpose$ are realizable as \emph{potential differences} associated with the \emph{potential function} assigning $\ttran_i\transpose$ to vertex $i \in \Nodes$ for all $i$   \cite{Biggs1997Algebraic}.  Formally, we have from \eqref{row_elements_of_tranPrecisions_tT_trot_transpose} that
\begin{equation}
\tT \trot\transpose = \incMat(\directed{G}) \transpose
\begin{pmatrix}
 -\ttran_1\transpose \\
 \vdots \\
 -\ttran_n\transpose
\end{pmatrix},
\end{equation}
so that the columns of $\tT\trot\transpose$ lie in $\image(\incMat(\directed{G})\transpose)$.  It follows that $\tranPrecisions^{\frac{1}{2}} \tT \trot\transpose \in \image(\tranPrecisions^{\frac{1}{2}}\incMat(\directed{G})\transpose)$.  But $\image(\tranPrecisions^{\frac{1}{2}}\incMat(\directed{G})\transpose) \perp 
\ker(\incMat(\directed{G}) \tranPrecisions^{\frac{1}{2}})$ by the Fundamental Theorem of Linear Algebra, and therefore $\tranPrecisions^{\frac{1}{2}} \tT \trot\transpose$ lies in the kernel of the orthogonal projector $\OrthoProjMatrix$, as claimed.
\end{proof}

With the aid of Lemmas \ref{exact_rotational_connection_Laplacian_lemma} and \ref{exact_translational_projector_lemma}, it is now straightforward to show that Problem \ref{dual_semidefinite_relaxation_for_SE3_synchronization_problem} is always exact in the noiseless case:

\begin{thm}[Problem \ref{dual_semidefinite_relaxation_for_SE3_synchronization_problem} is exact in the noiseless case]
\label{exactness_of_relaxation_in_the_zero_noise_case_thm}
Let $\tQ$ be the data matrix of the form \eqref{Q_quadratic_form_definition} constructed using the true \emph{(}latent\emph{)} relative transforms $\tpose_{ij}$ in \eqref{probabilistic_generative_model_for_noisy_observations}.  Then $\Zopt = {\trot}\transpose \trot$ is the unique solution of the instance of Problem \ref{dual_semidefinite_relaxation_for_SE3_synchronization_problem} parameterized by $\tQ$.
\end{thm}

\begin{proof}
Since $\TrueRotConLap \succeq 0$ by Lemma \ref{exact_rotational_connection_Laplacian_lemma}(i) and $\tQtran \succeq 0$ (immediate from the definition \eqref{nQtran_alternative_form}), $\tQ = \TrueRotConLap + \tQtran \succeq 0$ as well, and therefore the optimal value of Problem \ref{Orthogonal_relaxation_of_the_MLE_problem} satisfies $\OMLEval \ge 0$.  Furthermore, $\trot\transpose \in \ker(\TrueRotConLap), \ker(\tQtran)$ by Lemmas \ref{exact_rotational_connection_Laplacian_lemma}(iii) and \ref{exact_translational_projector_lemma}, respectively, so $\trot\transpose \in \ker(\tQ)$ as well.  This implies that  $\tr(\tQ \trot\transpose \trot) = 0$, and we conclude that $\trot$ is an optimal solution of the noiseless version of Problem \ref{Orthogonal_relaxation_of_the_MLE_problem} (since it is a feasible point that attains the lower bound of $0$ for Problem \ref{Orthogonal_relaxation_of_the_MLE_problem}'s optimal value $\OMLEval$).  This also implies  $\rank(\tQ) = dn - d$, since $\tQ \succeq \TrueRotConLap$ (and $\TrueRotConLap$ has $dn - d$ positive eigenvalues by Lemma \ref{exact_rotational_connection_Laplacian_lemma}(i)) and $\image(\trot \transpose) \subseteq \ker(\tQ)$ with $\dim ( \image(\trot\transpose) ) = d$.  Finally, a straightforward computation using equations \eqref{closed_form_solution_for_Lambda_star} and \eqref{candidate_certificate_matrix_definition} shows that the candidate certificate matrix corresponding to the optimal solution $\trot$ is $\true{\certMat} = \tQ$. The claim then follows from an application of Theorem  \ref{Sufficient_conditions_for_exact_recovery_for_orthogonally_relaxed_MLE_prop}.
\end{proof}

In addition to providing a useful sanity check on the feasibility of our overall strategy by showing that it will at least succeed under ideal conditions, the proof of Theorem \ref{exactness_of_relaxation_in_the_zero_noise_case_thm} also points the way towards a proof of the more general Proposition \ref{A_sufficient_condition_for_exact_recovery_prop}, as we now describe.  Observe that in the noiseless case, the certificate matrix $\certMat = \tQ = \TrueRotConLap + \tQtran$ corresponding to the optimal solution $\trot$ has a spectrum consisting of $d$ copies of $0$ and $dn - d$ strictly positive eigenvalues that are lower-bounded by $\lambda_2(\LapRotW) > 0$.  Now in the more general (noisy) case, both the data matrix $\nQ$ \emph{and} the minimizer $\Ropt$ of Problem \ref{Orthogonal_relaxation_of_the_MLE_problem} will vary as a function of the noise added to the measurements $\npose_{ij}$ in \eqref{probabilistic_generative_model_for_noisy_observations}, and in consequence so will the matrix $\certMat$.  However the first-order condition \eqref{primal_semidefinite_relaxation_first_order_necessary_condition_eq} appearing in Lemma \ref{KKT_conditions_for_primal_semidefinite_relaxation_lemma} can alternatively be read as $\certMat {\Ropt}\transpose = 0$, which guarantees that $\certMat$ always has at least $d$ eigenvalues fixed to $0$; furthermore, in general the eigenvalues of a matrix $X$ are continuous functions of $X$, and equations \eqref{closed_form_solution_for_Lambda_star} and \eqref{candidate_certificate_matrix_definition} show that $\certMat$ is a continuous function of $\nQ$ and $\Ropt$.  Consequently, if we can bound the magnitude of the estimation error $\Oorbdist(\trot, \Ropt)$ for a minimizer $\Ropt$ of Problem \ref{Orthogonal_relaxation_of_the_MLE_problem} as a function of the magnitude of the noise $\dQ = \nQ - \tQ$ corrupting the data matrix $\nQ$, then by controlling $\dQ$ we can in turn ensure (via continuity) that the eigenvalues of the matrix $\certMat$ constructed at the minimizer $\Ropt$ remain nonnegative, and hence Problem \ref{dual_semidefinite_relaxation_for_SE3_synchronization_problem} will remain exact.

\subsection{An upper bound for the estimation error in Problem \ref{Orthogonal_relaxation_of_the_MLE_problem}} 
\label{Upper_bound_for_estimation_error_subsection}

In this subsection we derive an upper bound on the estimation error $\Oorbdist(\trot, \Ropt)$ of a minimizer $\Ropt$ of Problem \ref{Orthogonal_relaxation_of_the_MLE_problem} as a function of the noise $\dQ \triangleq \nQ - \tQ$ corrupting the data matrix $\nQ$.  To simplify the derivation,  in the sequel we will assume (without loss of generality) that $\Ropt$ is an element of its orbit \eqref{Oorbit_definition} attaining the orbit distance $\Oorbdist(\trot, \Ropt)$ defined in \eqref{Oorbit_distance_definition}.

To begin, the optimality of $\Ropt$ implies that:
\begin{equation}
 \label{traces_of_RotMLEtranspose_RotMLE_for_global_error_bound}
 \begin{split}
\tr(\nQ \trot\transpose \trot ) &= \tr\left(\dQ \trot\transpose \trot \right) + \tr\left(\tQ \trot\transpose \trot \right) \\
&\ge \tr\left(\dQ {\Ropt}\transpose \Ropt \right) + \tr\left(\tQ {\Ropt}\transpose \Ropt \right) = \tr\left(\nQ {\Ropt}\transpose \Ropt \right).
\end{split}
\end{equation}
Now $\tr(\tQ\trot\transpose \trot) = 0$ since we showed in the previous subsection that $\image(\trot\transpose) = \ker(\tQ)$, and the identity $\tr(\dQ \trot\transpose \trot) = \vect(\trot)\transpose (\dQ \otimes I_d) \vect(\trot)$ together with the submultiplicativity of the spectral norm shows that
\begin{equation}
\label{upper_bound_on_trace_of_product_with_error_matrix}
\lvert \tr(\dQ \trot\transpose \trot) \rvert \le \left \lVert \dQ \otimes I_d \right \rVert_2 \left \lVert  \vect(\trot) \right \rVert_2^2 = \left \lVert \dQ \right \rVert_2 \left \lVert  \trot \right \rVert_F^2 = dn \lVert \dQ \rVert_2
\end{equation}
(and similarly for $\lvert \tr(\dQ {\Ropt}\transpose \Ropt) \rvert$); consequently, \eqref{traces_of_RotMLEtranspose_RotMLE_for_global_error_bound} in turn implies:
\begin{equation}
\label{upper_bound_on_tQ_Ropt_squared}
 2dn \lVert \dQ \rVert_2 \ge \tr\left(\tQ {\Ropt}\transpose \Ropt \right).
\end{equation}

We will now lower-bound the right-hand side of \eqref{upper_bound_on_tQ_Ropt_squared} as a function of the estimation error $\Oorbdist(\trot, \Ropt)$, thereby enabling us to upper-bound this error by controlling $\lVert \dQ \rVert_2$.  To do so, we make use of the following:

\begin{lem}
\label{orthogonal_projection_onto_imageRt_lemma}
 Fix $R \in \Orthogonal(d)^n \subset \R^{d \times dn}$, and let $M = \lbrace W R \mid W \in \R^{d \times d} \rbrace \subset \R^{d \times dn}$ denote the subspace of matrices whose rows are contained in $\image(R\transpose)$.  Then

\begin{equation}
\label{orthogonal_projection_operator_onto_image_of_Ztranspose}
 \begin{aligned}
&\proj_{V} \colon \R^{dn} \to \image(R\transpose) \\
&\proj_{V}(x) = \frac{1}{n} R\transpose Rx
 \end{aligned}
\end{equation}
is the orthogonal projection operator onto $\image(R\transpose)$ with respect to the usual $\ell_2$ inner product on $\R^{dn}$, and the mapping
\begin{equation}
 \begin{aligned}
&\proj_M \colon \R^{d \times dn} \to M \\
&\proj_M(X) = \frac{1}{n} XR\transpose R
 \end{aligned}
\end{equation}
that applies $\proj_{V}$ to each row of $X$ is the orthogonal projection operator onto $M$ with respect to the Frobenius inner product on $\R^{d \times dn}$. 
 \end{lem}

 \begin{proof}
  If $x \in \image(R\transpose)$, then $x = R\transpose v$ for some $v \in \R^{d}$, and 
  \begin{equation}
   \proj_{V}(x) = \frac{1}{n} R\transpose R (R\transpose v) = R\transpose v = x,
  \end{equation}
since $R R\transpose = n I_d$ as $R \in \Orthogonal(d)^n$ by hypothesis; this shows that $\proj_V$ is a projection onto $\image(R\transpose)$.  To show that $\proj_V$ is \emph{orthogonal} projection with respect to the $\ell_2$ inner product on $\R^{dn}$, it suffices to show that $\image(\proj_V) \perp \ker(\proj_V)$.  To that end, let $x, y \in \R^{dn}$, and observe that
\begin{equation}
\begin{split}
 \langle \proj_V(x), y - \proj_V(y) \rangle &= \left \langle \frac{1}{n} R\transpose R x, y - \frac{1}{n} R\transpose R y \right \rangle_2 \\
 &=\frac{1}{n} \left \langle R\transpose R x, y  \right \rangle_2 - \frac{1}{n^2} \left \langle R\transpose R x, R\transpose R y \right \rangle_2 \\
 &= \frac{1}{n} x\transpose R\transpose R y - \frac{1}{n^2} x\transpose R\transpose RR\transpose R y \\
 &= 0.
 \end{split}
\end{equation}

Next, let $X \in \R^{d \times dn}$ and observe that $\proj_M(X) = \frac{1}{n} XR\transpose R$ is the matrix obtained by applying the projection $\proj_V$ to each row of $X$; this immediately implies that $\proj_M$ is itself a projection onto $M$.  Furthermore, given $X, Y \in \R^{d \times dn}$, we observe that
\begin{equation}
 \begin{split}
\left \langle \proj_M(X), Y - \proj_M(Y) \right \rangle_F &= \left \langle \proj_M(X)\transpose, Y\transpose - \proj_M(Y)\transpose \right \rangle_F \\
&= \left \langle \vect\left(\proj_M(X)\transpose \right), \vect\left(Y\transpose - \proj_M(Y)\transpose \right) \right \rangle_2 \\
&= 0,
 \end{split}
\end{equation}
since we have already established that $\proj_M$ acts row-wise by $\proj_V$, which is orthogonal projection with respect to the $\ell_2$ inner product.
\end{proof}

Since $\ker(\tQ) = \image(\trot\transpose)$ and $\dim(\image(\trot\transpose) ) = d$, it follows from Lemma \ref{orthogonal_projection_onto_imageRt_lemma} that
\begin{equation}
\label{lower_bound_on_tQ_Ropt_squared_in_terms_of_normal_component_to_kernel}
\tr\left(\tQ {\Ropt}\transpose \Ropt \right) \ge \lambda_{d+1}(\tQ) \lVert P \rVert_F^2,
\end{equation}
where
\begin{equation}
\label{orthogonal_decomposition_of_Ropt}
\begin{split}
\Ropt &= K + P, \\
K &= \proj_M(\Ropt) = \frac{1}{n}  \Ropt \trot\transpose \trot, \\
P &= \Ropt - \proj_M(\Ropt) = \Ropt -  \frac{1}{n} \Ropt \trot\transpose \trot
\end{split}
\end{equation}
is an orthogonal decomposition of $\Ropt$ with respect to the Frobenius inner product on $\R^{d \times dn}$, and the rows of $P$ are contained in $\image(R\transpose)^{\perp} = \ker(\tQ)^{\perp}$.  Using \eqref{orthogonal_decomposition_of_Ropt}, we compute:
\begin{equation}
 \lVert K \rVert_F^2 = \frac{1}{n^2} \tr \left(\trot\transpose \trot {\Ropt}\transpose \Ropt \trot\transpose \trot \right) = \frac{1}{n} \tr \left( \trot {\Ropt}\transpose \Ropt \trot\transpose \right) = \frac{1}{n} \left \lVert \trot {\Ropt}\transpose \right \rVert_F^2
\end{equation}
where we have used the cyclic property of the trace and the fact that $\trot\trot\transpose = nI_d$.  Since \eqref{orthogonal_decomposition_of_Ropt} is an orthogonal decomposition, it follows that
\begin{equation}
\label{normal_component_of_Ropt}
 \lVert P \rVert_F^2 = \left\lVert \Ropt \right \rVert_F^2 - \lVert K \rVert_F^2 = dn - \frac{1}{n} \left \lVert \trot {\Ropt}\transpose \right \rVert_F^2.
\end{equation}
We may therefore lower-bound $\lVert P \rVert_F^2$ by upper-bounding $ \lVert \trot {\Ropt}\transpose \rVert_F^2$ as functions of $\Oorbdist(\trot, \Ropt)$.  To that end, recall that $\Ropt$ is by hypothesis a representative of its orbit \eqref{Oorbit_definition} that attains the orbit distance \eqref{Oorbit_distance_definition}; Theorem \ref{computing_the_orbit_distance_theorem} then implies that
\begin{equation}
\label{estimation_error_of_minimizer_in_terms_of_singular_values}
\Oorbdist(\trot, \Ropt)^2 = \lVert \trot - \Ropt \rVert_F^2 = 2dn - 2\sum_{i = 1}^d \sigma_i,
\end{equation}
where
\begin{equation}
\label{singular_value_decomposition_for_upper_bounding_P_squared_norm}
\trot {\Ropt}\transpose =  U \Diag(\sigma_1, \dotsc, \sigma_d) V\transpose
\end{equation}
is a singular value decomposition of $\trot {\Ropt}\transpose$.  It follows from \eqref{singular_value_decomposition_for_upper_bounding_P_squared_norm} and the orthogonal invariance of the Frobenius inner product that 
\begin{equation}
\label{norm_of_trot_Ropt_squared_in_terms_of_singular_values}
 \left \lVert \trot {\Ropt}\transpose \right \rVert_F^2 = \left \lVert \Diag(\sigma_1, \dotsc, \sigma_d) \right \rVert_F^2 = \sum_{i = 1}^d \sigma_i^2,
\end{equation}
and therefore \eqref{estimation_error_of_minimizer_in_terms_of_singular_values} and \eqref{norm_of_trot_Ropt_squared_in_terms_of_singular_values} imply that we may obtain an upper bound $\epsilon^2$ for $\lVert \trot {\Ropt}\transpose  \rVert_F^2$ in terms of $\delta^2 = \Oorbdist(\trot, \Ropt)^2$ as the optimal value of:
\begin{equation}
\label{maximization_problem_for_squared_Frobenius_norm_of_RotMLEtrot_for_Problem4}
\begin{split}
 \epsilon^2 &= \max_{\sigma_i \ge 0}  \sum_{i = 1}^d \sigma_i^2 \\
 \st &2dn - 2\sum_{i = 1}^d \sigma_i = \delta^2.
 \end{split}
\end{equation}
The first-order necessary optimality condition for \eqref{maximization_problem_for_squared_Frobenius_norm_of_RotMLEtrot_for_Problem4} is
\begin{equation}
 2\sigma_i = -2 \lambda
\end{equation}
for all $i \in [d]$, where $\lambda \in \R$ is a Lagrange multiplier, and therefore $\sigma_1 = \dotsb = \sigma_d = \sigma$ for some $\sigma \in \R$.  Solving the constraint in \eqref{maximization_problem_for_squared_Frobenius_norm_of_RotMLEtrot_for_Problem4} for $\sigma$
shows that
\begin{equation}
\label{optimal_value_of_sigma_for_upper_bound_on_squared_Frobenius_norm_of_RotMLEtrot_for_Problem4}
 \sigma = n - \frac{\delta^2}{2d},
\end{equation}
and therefore the optimal value of the objective in \eqref{maximization_problem_for_squared_Frobenius_norm_of_RotMLEtrot_for_Problem4} is 
\begin{equation}
\label{optimal_value_of_squared_Frobenius_norm_of_trot_RotEst_transpose_as_a_function_of_delta2_for_Problem_4}
 \epsilon^2 = d\left( n - \frac{\delta^2}{2d} \right)^2.
\end{equation}
Recalling the original definitions of $\epsilon^2$ and $\delta^2$, we conclude from \eqref{optimal_value_of_squared_Frobenius_norm_of_trot_RotEst_transpose_as_a_function_of_delta2_for_Problem_4} and \eqref{normal_component_of_Ropt} that 
\begin{equation}
\label{lower_bound_on_normal_component_of_RotEst_for_Problem_4}
 \lVert P \rVert_F^2 \ge dn - \frac{d}{n} \left( n - \frac{\Oorbdist(\trot, \Ropt)^2}{2d} \right)^2  = \Oorbdist(\trot, \Ropt)^2 - \frac{\Oorbdist(\trot, \Ropt)^4}{4dn}.
\end{equation}
Applying the inequality $\Oorbdist(\trot, \Ropt)^2 \le 2dn$ (which follows immediately from the nonnegativity of the nuclear norm in \eqref{closed_form_Oorbdist_computation}), we may in turn lower-bound the right-hand side of \eqref{lower_bound_on_normal_component_of_RotEst_for_Problem_4} as:
\begin{equation}
\label{reduction_to_quadratic_function_of_estimation_error}
 \Oorbdist(\trot, \Ropt)^2 - \frac{\Oorbdist(\trot, \Ropt)^4}{4dn} = \left(1 - \frac{\Oorbdist(\trot, \Ropt)^2}{4dn}\right) \Oorbdist(\trot, \Ropt)^2 \ge \frac{1}{2} \Oorbdist(\trot, \Ropt)^2.
\end{equation}
Finally, combining inequalities \eqref{upper_bound_on_tQ_Ropt_squared}, \eqref{lower_bound_on_tQ_Ropt_squared_in_terms_of_normal_component_to_kernel}, \eqref{lower_bound_on_normal_component_of_RotEst_for_Problem_4}, and \eqref{reduction_to_quadratic_function_of_estimation_error}, we obtain the following:

\begin{thm}[An upper bound for the estimation error in Problem \ref{Orthogonal_relaxation_of_the_MLE_problem}]
\label{An_upper_bound_on_the_estimation_error_in_Problem_5_Theorem}
Let $\tQ$ be the data matrix of the form \eqref{Q_quadratic_form_definition} constructed using the true \emph{(}latent\emph{)} relative transforms $\tpose_{ij} = (\ttran_{ij}, \trot_{ij})$ in \eqref{probabilistic_generative_model_for_noisy_observations}, $\trot \in \SO(d)^n$ the matrix composed of the true \emph{(}latent\emph{)} rotational states, and $\Ropt \in \Orthogonal(d)^n$ an estimate of $\trot$ obtained as a minimizer of Problem \ref{Orthogonal_relaxation_of_the_MLE_problem}. Then the estimation error $\Oorbdist(\trot, \Ropt)$ admits the following upper bound:
\begin{equation}
 \sqrt{\frac{4dn \lVert \nQ - \tQ \rVert_2}{\lambda_{d+1}(\tQ)}} \ge \Oorbdist(\trot, \Ropt).
\end{equation}
\end{thm}

\subsection{Finishing the proof}

Finally, we complete the proof of Proposition \ref{A_sufficient_condition_for_exact_recovery_prop} with the aid of Theorems \ref{Sufficient_conditions_for_exact_recovery_for_orthogonally_relaxed_MLE_prop} and \ref{An_upper_bound_on_the_estimation_error_in_Problem_5_Theorem}.

\begin{proof}[Proof of Proposition \ref{A_sufficient_condition_for_exact_recovery_prop}]
Let $\trot \in \SO(d)^n$ be the matrix of true (latent) rotations, $\Ropt \in \Orthogonal(d)^n$ an estimate of $\trot$ obtained as a minimizer of Problem \ref{Orthogonal_relaxation_of_the_MLE_problem}, and assume without loss of generality that $\Ropt$ is an element of its orbit \eqref{Oorbit_definition} attaining the orbit distance $\Oorbdist(\trot, \Ropt)$ defined in \eqref{Oorbit_distance_definition}.  Set $\dQ \triangleq \nQ - \tQ$ and $\dR \triangleq \Ropt - \trot$, and consider the following decomposition of the certificate matrix $\certMat$ defined in  \eqref{candidate_certificate_matrix_definition} and \eqref{closed_form_solution_for_Lambda_star}:
\begin{equation}
   \label{decomposition_of_certificate_matrix}
 \begin{split}
\certMat &= \nQ - \SymBlockDiag_d\left(\nQ {\Ropt}\transpose \Ropt \right) \\
&= \left(\tQ + \dQ \right) - \SymBlockDiag_d\left( \left(\tQ + \dQ \right) (\trot + \dR)\transpose (\trot + \dR) \right) \\
&= \tQ + \dQ - \SymBlockDiag_d\left(
\begin{split}
&\tQ \trot \transpose \trot + \dQ \trot\transpose\trot + \tQ \trot\transpose\dR + \tQ\dR\transpose \trot \\
&\quad + \dQ\trot\transpose \dR + \dQ\dR\transpose \trot + \tQ\dR\transpose \dR + \dQ \dR\transpose \dR
 \end{split}\right) \\
 &= \tQ + \underbrace{\dQ - \SymBlockDiag_d\left(
\begin{split}
& \dQ \trot\transpose\trot + \tQ\dR\transpose \trot + \dQ\trot\transpose \dR \\
&\quad + \dQ\dR\transpose \trot + \tQ\dR\transpose \dR + \dQ \dR\transpose \dR
 \end{split}\right)}_{\Delta \certMat },
 \end{split}
\end{equation}
where we have used the fact that $\image(\trot\transpose) = \ker(\tQ)$ in passing from lines 2 to 3 above (cf. Lemmas \ref{exact_rotational_connection_Laplacian_lemma} and \ref{exact_translational_projector_lemma}).  Observe that the term labeled $\Delta \certMat$ in \eqref{decomposition_of_certificate_matrix} depends continuously upon $\dQ$ and $\dR$, with $\Delta \certMat = 0$ for $(\dQ, \dR) = (0,0)$; furthermore, $\dQ \to 0$ implies $\dR \to 0$  by Theorem \ref{An_upper_bound_on_the_estimation_error_in_Problem_5_Theorem}.  It therefore follows from continuity that there exists some $\beta_1 > 0$ such that $\lVert \Delta \certMat \rVert_2 < \lambda_{d+1}(\tQ)$ for all $\lVert \dQ \rVert_2 < \beta_1$.  Moreover, if $\lVert \Delta \certMat \rVert_2 < \lambda_{d+1}(\tQ)$, it follows from \eqref{decomposition_of_certificate_matrix} that 
\begin{equation}
 \lambda_i(\certMat) \ge \lambda_i(\tQ) - \lVert \Delta \certMat \rVert_2 > \lambda_i(\tQ) - \lambda_{d+1}(\tQ),
\end{equation}
and therefore $\lambda_i(\certMat) > 0$ for $i \ge d+1$; i.e., $\certMat$ has at least $dn - d$ strictly positive eigenvalues.  Furthermore, Lemma \ref{KKT_conditions_for_primal_semidefinite_relaxation_lemma} shows that $\certMat {\Ropt}\transpose = 0$, which implies that $\ker(\certMat) \supseteq \image({\Ropt}\transpose)$; since $\dim(\image({\Ropt}\transpose)) = d$, this in turn implies that $\certMat$ has at least $d$ eigenvalues equal to $0$.  Since this exhausts $\certMat$'s $dn$ eigenvalues, we conclude that $\certMat \succeq 0$ and $\rank(\certMat) = dn - d$, and consequently Theorem \ref{Sufficient_conditions_for_exact_recovery_for_orthogonally_relaxed_MLE_prop} guarantees that $\Zopt = {\Ropt}\transpose \Ropt$ is the unique minimizer of Problem \ref{dual_semidefinite_relaxation_for_SE3_synchronization_problem}.

Now suppose further that $\lVert \dQ \rVert_2 < \beta_2$ with $\beta_2 \triangleq \lambda_{d+1}(\tQ) / 2dn$.  Then Theorem \ref{An_upper_bound_on_the_estimation_error_in_Problem_5_Theorem} implies $\Oorbdist(\trot, \Ropt) = \lVert \trot - \Ropt \rVert_F < \sqrt{2}$, and therefore in particular that $\lVert \trot_i - \Ropt_i \rVert_F < \sqrt{2}$ for all $i \in [n]$.  But the $+1$ and $-1$ components of $\Orthogonal(d)$ are separated by a distance $\sqrt{2}$ under the Frobenius norm, so $\trot_i \in \SO(d)$ and $\lVert \trot_i - \Ropt_i \rVert_F < \sqrt{2}$ for all $i \in [n]$ together imply that $\Ropt \in \SO(d)^n$, and therefore that $\Ropt$ is in fact an optimal solution of Problem \ref{Simplified_maximum_likelihood_estimation_for_SE3_synchronization} as well.

Proposition \ref{A_sufficient_condition_for_exact_recovery_prop} then follows from the preceding paragraphs by taking $\beta \triangleq \min \lbrace \beta_1, \beta_2 \rbrace > 0$.
\end{proof}